\documentclass[twoside,11pt]{article}

\usepackage{blindtext}

\usepackage[preprint]{jmlr2e}

\usepackage{amsmath,amssymb}
\usepackage{mathtools}
\usepackage{bm,bbm}
\usepackage{physics}
\usepackage{xcolor}
\usepackage[ruled,vlined,linesnumbered]{algorithm2e}
\usepackage{booktabs}
\usepackage{enumitem}
\usepackage{graphicx}
\usepackage{subcaption}

\newcommand\numberthis{\addtocounter{equation}{1}\tag{\theequation}}
\newcommand{\argmin}{\mathop{\arg\min}}
\newcommand{\argmax}{\mathop{\arg\max}}
\newcommand{\m}{\middle|}
\newcommand{\R}{\mathcal{R}}
\newcommand{\nec}{\mathcal{A}}
\newcommand{\B}{\mathcal{B}}

\newcommand{\sP}{\mathbb{P}}
\newcommand{\E}{\mathbb{E}}
\newcommand{\dis}{\mathcal{D}_\alpha}
\newcommand{\Fdis}{\mathcal{F}_\alpha}
\newcommand{\Pdis}{\mathcal{P}_\alpha}

\newcommand{\ind}{\mathbbm{1}}
\newcommand{\pn}[1]{^{(#1)}}
\newcommand{\ul}{\underline{\lambda}}
\newcommand{\wm}{\widetilde{m}}

\newcommand{\ma}{m_\alpha}
\newcommand{\da}{d_\alpha}

\newcommand{\nn}{\nonumber\\}
\newcommand{\n}{\nonumber}

\newcommand{\lams}{\lambda^*}
\newcommand{\acon}{C_1(\dis)}
\newcommand{\wlow}{w^*(\dis)}

\newcommand{\GR}{{\sf{FTPL~GR}}}
\newcommand{\CGR}{{\sf{FTPL~CGR}}}
\newcommand{\HYBRID}{{\sf{HYBRID}}}
\newcommand{\LBINFV}{{\sf{LBINFV-LS}}}

\newcommand{\qed}{\hfill\BlackBox\\[2mm]}
\def\dqed{\relax\tag*{\BlackBox}} 
\newenvironment{proof2}{\par\noindent{\bf Proof\ }}{\vspace{0mm}} 

\usepackage{lastpage}
\jmlrheading{23}{2026}{1-\pageref{LastPage}}{1/21; Revised 5/22}{9/22}{21-0000}{Botao Chen, Jongyeong Lee, Chansoo Kim, and Junya Honda}

\ShortHeadings{A Further Efficient Algorithm with BOBW for \texorpdfstring{$m$}{m}-Set Semi-Bandits}{Chen, Lee, Kim, and Honda}
\firstpageno{1}

\begin{document}

\title{A Further Efficient Algorithm with Best-of-Both-Worlds Guarantees for \texorpdfstring{$m$}{m}-Set Semi-Bandit Problem}

\author{\name Botao Chen \email chen.botao.63r@st.kyoto-u.ac.jp \\
       \addr Department of Systems Science\\
       Graduate School of Informatics, Kyoto University\\
       Kyoto, Japan
       \AND
       \name Jongyeong Lee \email jongyeong@kist.re.kr \\
       \addr Computational Science Research Center\\
       Korea Institute of Science and Technology\\
       Seoul, Korea
       \AND
       \name Chansoo Kim \email eau@ust.ac.kr \\
       \addr Computational Science Research Center\\
       Korea Institute of Science and Technology \& University of Science and Technology\\
       Seoul, Korea
       \AND
       \name Junya Honda \email honda@i.kyoto-u.ac.jp \\
       \addr Department of Systems Science\\
       Graduate School of Informatics, Kyoto University\\
       Kyoto, Japan\\
       and Center for Advanced Intelligence Project, RIKEN\\
       Tokyo, Japan}

\editor{My editor}

\maketitle

\begin{abstract}
This paper studies the optimality and complexity of Follow-the-Perturbed-Leader (FTPL) policy in $m$-set semi-bandit problems. 
FTPL has been studied extensively as a promising candidate of an efficient algorithm with favorable regret for adversarial combinatorial semi-bandits. Nevertheless, the optimality of FTPL has still been unknown unlike Follow-the-Regularized-Leader (FTRL) whose optimality has been proved for various tasks of online learning. 
In this paper, we extend the analysis of FTPL with geometric resampling (GR) to $m$-set semi-bandits, which is a special case of combinatorial semi-bandits, showing that FTPL with Fr\'{e}chet and Pareto distributions with certain parameters achieves the best possible regret of $O(\sqrt{mdT})$ in adversarial setting. 
We also show that FTPL with Fr\'{e}chet and Pareto distributions with a certain parameter achieves a logarithmic regret for stochastic setting, meaning the Best-of-Both-Worlds optimality of FTPL for $m$-set semi-bandit problems. 
Furthermore, we extend the conditional geometric resampling to $m$-set semi-bandits for efficient loss estimation in FTPL, reducing the computational complexity from $O(d^2)$ of the original geometric resampling to $O\qty(md(\log(d/m)+1))$ without sacrificing the regret performance.
\end{abstract}


\section{Introduction}
The combinatorial semi-bandit problem is a sequential decision-making problem under uncertainty, which is a generalization of the classical multi-armed bandit (MAB) problem \citep{cesa2012combinatorial}, and is instrumental in many practical applications, such as recommender systems \citep{wang2017efficient}, online advertising \citep{nuara2022online}, crowdsourcing \citep{ul2016efficient}, adaptive routing \citep{gai2012combinatorial} and network optimization \citep{kveton2014matroid}. 
In this problem, the learner selects an action $a_t$ from an action set $\nec\subset\{0,1\}^d$, where $d\in\mathbb{N}$ is the dimension of the action set. 
At round $t\in[T]=\{1,2,\dots,T\}$, the loss vector $\ell_t=(\ell_{t,1},\ell_{t,2},\dots,\ell_{t,d})$ is determined by the environment, and the learner selects an action $a_t$ from $\nec$.
The learner then incurs a loss $\left\langle \ell_t,a_t\right\rangle$ and can only observe the loss $\ell_{t,i}$ for which $a_{t,i}=1$.
The learner's goal is to minimize the cumulative loss over all rounds.
The performance of the learner is often measured by the \textit{pseudo-regret} defined as 
$\mathcal{R}(T)=\E[\sum_{t=1}^{T}\left\langle \ell_t,a_t-a^*\right\rangle]$ for $a^*=\argmin_{a\in\nec}\E[\sum_{t=1}^{T}\left\langle \ell_t,a\right\rangle]$,
which describes the gap between the expected cumulative loss of the learner and of the single optimal action $a^*$
fixed in hindsight.
In this paper, we focus on one of the most fundamental classes of combinatorial semi-bandit problems, referred to as $m$-set semi-bandit problem, where the action set is defined as $\nec=\{a\in\{0,1\}^d:\left\lVert a\right\rVert_1=m\}$ with $m\in[d]$ representing the size of each action in $\nec$.

Since the introduction by \citet{chen2013combinatorial}, combinatorial semi-bandit problems have been widely studied, primarily focusing on two settings on the formulation of the environment for generating the loss vectors, namely the stochastic setting and the adversarial setting.
In the stochastic setting, the sequence of loss vectors $\{\ell_t\}_{t=1}^T$ is assumed to be independent and identically distributed (i.i.d.) from an unknown but fixed distribution $\mathcal{D}$ over $[0,1]^d$ with mean 
$\mu=\E_{\ell\sim\mathcal{D}}[\ell]$.
CombUCB \citep{kveton2015tight} and Combinatorial Thompson Sampling \citep{wang2018thompson} can achieve a gap-dependent regret bounds of $O\qty(dm\log T/\Delta)$ for general action sets and $O\qty((d-m)\log T/\Delta)$ for matroid semi-bandits, where $\Delta=\min_{a\in\nec\setminus\{a^*\}}\{\mu^{\top}(a-a^*)\}$ represents the minimum suboptimality gap.

In the adversarial setting, the loss vector $\ell_t\in[0,1]^d$ is determined by an adversary in an arbitrary manner, which is not assumed to follow any specific distribution \citep{kveton2015tight,neu2015first,wang2018thompson}. 
For this setting, the regret bound of $O(\sqrt{mdT})$ can be achieved by some policies, such as OSMD \citep{audibert2014regret} and FTRL with hybrid-regularizer \citep{zimmert2019beating}, which matches the lower bound of $\Omega(\sqrt{mdT})$ \citep{audibert2014regret}.

In practical scenarios, the environment to determine the loss vectors is often unknown. Therefore, policies that can adaptively address both stochastic and adversarial settings have been widely studied, particularly in the context of MAB problems. The Tsallis-INF policy \citep{zimmert2021tsallis}, which is based on Follow-the-Regularized-Leader (FTRL), has demonstrated the ability to achieve the optimality in both settings, a status reffered to as Best-of-Both-Worlds (BOBW, \citealp{bubeck2012best}). For combinatorial semi-bandit problems, there also exists some work on this topic \citep{wei2018more,zimmert2019beating,ito2021hybrid,tsuchiya2023further}.

Howerver, some BOBW policies, such as FTRL, require explicit computation of the arm-selection probability by solving an optimization problem. This leads to computational inefficiencies, especially in combinatorial semi-bandits, where the complexity increases substantially. 
For $m$-set semi-bandits, an FTRL-based policy can efficiently compute the base-arm selection probabilities by Newton's method with per-iteration cost of $O(d)$, and the action sampling at each round requires $O(d\log d)$ time \citep{zimmert2019beating}. 
Though it is empirically observed that $O(1)$ iterations are sufficient, the number of iterations provably required to maintain the theoretical guarantee remains unknown. 
Nevertheless, despite the efficiency of Newton's method, 
solving an optimization problem is still time-consuming and may suffer from numerical instability in practice (see also Section~\ref{sec:experiments}).
In light of this limitation, the Follow-the-Perturbed-Leader (FTPL) policy, which greedily selects an action with the minimum cumulative estimated loss with some random perturbation, has gained wide attention for its optimization-free nature in 
adversarial MAB \citep{poland2005fpl,abernethy2015fighting,kim2019optimality}, linear bandits \citep{mcmahan2004online} and MDP bandits \citep{dai2022follow}. 
In the context of combinatorial semi-bandits, \citet{JMLR:v17:15-091} showed that FTPL can achieve an adversarial regret of $O(m\sqrt{dT\log(d/m)})$. 
In addition, they proposed a technique called Geometric Resampling (GR) to address the high computational cost of loss estimates, with which FTPL exhibits the initial practical applicability in both combinatorial semi-bandits and MAB.

Recently, there has been significant progress in establishing the BOBW guarantee for FTPL policy.
\citet{kim2019optimality} conjectured that if FTPL achieves minimax opitimality for MAB, then the perturbation should be Fr\'{e}chet-type distributions.
Subsequently, this open problem has been rosolved, as it was shown that FTPL with Fr\'{e}chet-type tail perturbations under certain conditions achieves BOBW guarantee for MAB problems \citep{pmlr-v201-honda23a,pmlr-v247-lee24a,lee2025revisiting}.
Besides, there is also some work on BOBW in decoupled bandits \citep{kim2025follow}.
These results naturally motivate the study on the optimality of FTPL in combinatorial semi-bandits.

\paragraph{Contribution of This Paper and Closely Related Work}
This paper investigates the optimality and complexity of FTPL in $m$-set semi-bandit problems. 
In the preliminary version of this paper,\kern0.03em\footnote{Available on arXiv \citep{chen2025note}.} we showed that for adversarial setting, FTPL respectively achieves $O(\sqrt{m^2 d^{1/\alpha}T}+\sqrt{mdT})$ regret with Fr\'{e}chet distribution, and optimal $O(\sqrt{mdT})$ regret with Pareto distribution with shape $\alpha>1$, revealing that Pareto perturbation has specially desirable properties even though the optimal regret bound for the classic MAB has been primarily established under Fr\'{e}chet perturbation \citep{pmlr-v201-honda23a}. 
After the preliminary version of this paper, a closely related work by \citet{zhan2025follow}, which focuses on FTPL with shape-$2$ Fr\'{e}chet distribution in $m$-set semi-bandits, showing the adversarial regret of $O(\sqrt{mdT\log d}+\sqrt{m^{8/3}T})$ and the stochastic regret of $O(\sum_{i:a^*=0}\frac{\log T}{\Delta_i})+O((m^2d\log d+m^{11/3}+md^2)/\Delta)$.

In this version of the paper, we show that FTPL achieves $O(\sqrt{mdT})$ regret in adversarial bandits not only with Pareto distribution but also with Fr\'{e}chet distribution with shape $\alpha>1$. 
We further provide a problem-dependent regret bound for Fr\'{e}chet and Pareto distributions in stochastic bandits, demonstrating that FTPL with these distributions with shape parameter $\alpha=2$ achieves stochastic regret of $O(\sum_{i:a^*_i=0}\frac{\log T}{\Delta_i})+O(\frac{m^3 d}{\Delta})$, thereby establishing the
BOBW guarantees in $m$-set semi-bandits. 
To the best of our knowledge, this is the first work that establishes the adversarial optimality and BOBW guarantee for FTPL in $m$-set semi-bandits. 
Furthermore, we extend the technique for loss estimation called Conditional Geometric Resampling (CGR) \citep{chen2025cgr} to $m$-set semi-bandits, which reduces the computational complexity from $O(d^2)$ of the original GR to $O\qty(md\qty(\log(d/m)+1))$ without sacrificing the regret guarantee of the one with the original GR. 
As a consequence, FTPL with CGR becomes the first policy for $m$-set semi-bandits that simultaneously achieves the BOBW optimality and provably nearly linear dependence on $d$ in computational complexity.

Although \citet{zhan2025follow} has proved that FTPL with shape-$2$ Fr\'{e}chet distribution achieves a logarithmic regret in stochastic $m$-set semi-bandits, their analysis appears to be tailored to the specific form of Fr\'{e}chet distribution, yet derives a relatively loose second-order regret bound (see also Section~\ref{subsec:regret_optimal_action}).
In this paper, we develop a novel analysis technique built upon the common structure of general Fr\'{e}chet-type distributions, which also substantially improves the second-order regret bound compared to the existing result.

\section{Problem Setup}

In this section, we formulate the problem and introduce the framework of FTPL with geometric resampling. 
We consider a $d$-dimensional action set $\nec=\{a\in\{0,1\}^d:\left\lVert a\right\rVert_1=m\}$ with each element $a\in\nec$ called an action (also called a super-arm).
Here, $m$ denotes the size of each action, i.e., the fixed number of base-arms $i\in[d]$ selected at each round.
At each round $t\in[T]=\qty{1,2,\dots,T}$, the environment determines a loss vector $\ell_t=(\ell_{t,1},\ell_{t,2},\dots,\ell_{t,d})\in[0,1]^d$, and the learner selects an action $a_t\in\nec$, with $a_{t,i}=1$ indicating that base-arm $i$ is selected. 
The learner then incurs a loss $\left\langle \ell_t,a_t\right\rangle$ and only observes the losses $\ell_{t,i}$ for which $a_{t,i}=1$.

The loss vector is determined in either a stochastic or
adversarial way.
In stochastic setting, the loss vectors $\{\ell_t\}_{t=1}^T$ are assumed to be i.i.d. from an unknown but fixed distribution $\mathcal{D}$ over $[0,1]^d$. 
We define the expectation of the losses as $\mu=\E_{\ell\sim\mathcal{D}}\qty[\ell]$. 
The suboptimality gap of base-arm $i$ is defined as $\Delta_i=\mu_i-\max_{j:a^*_j=1}\mu_j$, where $a^*=\argmin_{a\in\nec}a^\top \mu$ represents the fixed single optimal action. 
In adversarial setting, the loss vectors $\{\ell_t\}_{t=1}^T$ are not assumed to follow any specific distribution, which are determined in an arbitrary manner and may depend on the past history of the actions and losses $\{(\ell_s,a_s)\}_{s=1}^{t-1}$. 

The performance of the learner is evaluated in terms of the pseudo-regret, which is defined as
\begin{equation*}
    \mathcal{R}(T)=\E\qty[\sum_{t=1}^{T}\left\langle \ell_t,a_t-a^*\right\rangle],\quad a^*\in\min_{a\in\nec}\E\qty[\sum_{t=1}^{T}\left\langle \ell_t,a\right\rangle].
\end{equation*}

\begin{table}[t]
    \centering
    \caption{Notation.}
    \vspace{6pt}
    \label{tab:notation}
    \begin{tabular}{ll}
    \toprule
    \textbf{Symbol} & \textbf{Meaning} \\
    \midrule
    $\nec \subset \{0,1 \}^d$ & Action set \\
    $d \in \mathbb{N}$ & Dimensionality of action set \\
    $m \in [d]$ & $m = \|a\|_1$ for any $a\in\nec$, i.e., the size of each action \\
    $\eta_t \in\mathbb{R}^+$ & Learning rate\\
    $\ell_t\in[0,1]^d$ & Loss vector\\
    $\hat{\ell}_t\in[0,\infty)^d$ & Estimated loss vector\\
    $\hat{L}_t=\sum_{s=1}^{t-1}\hat{\ell}_s\in[0,\infty)^d$ & Cumulative estimated loss vector\\
    \midrule
    $\sigma_i(\bm{u},\B)$ & \parbox[t]{9cm}{Rank of $i$-th element of $\bm{u}$ in $\qty{u_j:j\in\B}$ (descending), 
    $\bm{u}$ (resp. $\B$) omitted when $\bm{u}=-\hat{L}_t$ (resp. $\B=[d]$)}\\
    \midrule
    $\nu\in\mathbb{R}$ & Left-end point of perturbation\\
    $r_t\in[\nu,\infty)^d$ & $d$-dimensional perturbation\\
    \midrule
    $f(x)$ & Probability density function of perturbation \\
    $F(x)$ & Cumulative distribution function of perturbation \\
    \midrule
    $\mathcal{F}_\alpha$ & Fr\'{e}chet distribution with shape $\alpha$ \\
    $\mathcal{P}_\alpha$ & Pareto distribution with shape $\alpha$ \\
    \bottomrule
    \end{tabular}
\end{table}

\subsection{Follow-the-Perturbed-Leader}

\begin{algorithm}[t]
    \caption{Follow-the-Perturbed-Leader}
    \label{alg:FTPL}
    \KwIn{Action set $\nec \subseteq \{0,1\}^d$, learning rate $\eta_t \in \mathbb{R}^+$\;}
    \textbf{Initialization: }$\hat{L}_1 \coloneqq\mathbf{0} \in \mathbb{R}^d$\;
    
    \For{$t=1, \dots, T$}{
        Sample $r_t=\qty(r_{t,1},r_{t,2},\dots,r_{t,d})$ i.i.d.~from $\mathcal{D}$\;
        Choose action $a_t = \argmin_{a \in \nec} \left\{a^\top (\eta_t \hat{L}_t - r_t)\right\}$ and observe $\qty{\ell_{t,i}:a_{t,i}=1}$\;
        Compute an estimator $\widehat{w_{t,i}^{-1}}$ for $w_{t,i}^{-1}$ by geometric resampling for all $i$ such that $a_{t,i}=1$\;
        Set $\hat{\ell}_t\coloneqq\sum_{i:a_{t,i}=1}\ell_{t,i}\widehat{w_{t,i}^{-1}}e_i$ and $\hat{L}_{t+1}\coloneqq\hat{L}_t+\hat{\ell}_t$\;
    }
\end{algorithm}

We consider the Follow-the-Perturbed-Leader (FTPL) policy given in Algorithm~\ref{alg:FTPL}, where $e_i$ denotes $d$-dimensional the unit vector whose $i$-th element is $1$ and the others are $0$. 
Since only partial feedback is available in the semi-bandit setting, FTRL and FTPL policies employ an estimator $\hat{\ell}_t$ for the loss vector $\ell_t$ at each round $t$, which is specified in Sections~\ref{subsec:GR} and \ref{subsec:CGR}.
In combinatorial semi-bandit problems, FTPL policy maintains a cumulative estimated loss $\hat{L}_t=\sum_{s=1}^{t-1}\hat{\ell}_s$ and selects an action
\begin{equation*}
    a_t = \argmin_{a \in \nec} \left\{a^\top (\eta_t \hat{L}_t - r_t)\right\},
\end{equation*}
where $\eta_t\in\mathbb{R}^+$ is the learning rate, and $r_t=\qty(r_{t,1},r_{t,2},\dots,r_{t,d})$ denotes the random perturbation i.i.d. from a common distribution $\mathcal{D}$ with cumulative distribution function $F$.
In this paper, we consider two perturbation distributions, both of which belong to a class of heavy-tailed distributions called Fr\'{e}chet-type distributions.
The first one is Fr\'{e}chet distribution $\Fdis$, with the probability density function $f(x)$ and the cumulative distribution function $F(x)$ given by
\begin{equation}\label{eq:frechet_def}
    f(x)=\alpha x^{-(\alpha+1)}e^{-1/x^{\alpha}},\quad
    F(x)=e^{-1/x^{\alpha}},\quad x\geq 0,\,\alpha>1.
\end{equation}
The second one is Pareto distribution $\Pdis$, whose density and cumulative distribution functions are given by
\begin{equation}\label{eq:pareto_def}
    f(x)=\alpha x^{-(\alpha+1)},\quad
    F(x)=1-x^{-\alpha},\quad x\geq 1,\,\alpha>1.
\end{equation}
Fr\'{e}chet distribution is standard for BOBW literature while Pareto distribution has more desirable properties and the analysis becomes simpler. 

In the following, we denote the rank of $i$-th element of $\bm{u}$ among index set $\B$ in descending order as $\sigma_i(\bm{u},\B)$, which has an important role in the analysis of FTPL. 
We omit $\bm{u}$ (resp.~$\B$) when $\bm{u}=-\hat{L}_t$ (resp.~$\B=[d]$). 
For example, one of the current best base-arms $\hat{i}_t^*\in\argmin_{j}\hat{L}_{t,j}$ satisfies $\sigma_i(-\hat{L}_t,[d])=\sigma_i=1$. 

Table~\ref{tab:notation} summarizes the notation used in this paper.

\subsection{Geometric Resampling}\label{subsec:GR}

\begin{algorithm}[t]
    \caption{Geometric Resampling}
    \label{alg:GR}
    \KwIn{Chosen action $a_t$, action set $\nec$, cumulative loss $\hat{L}_t$, learning rate $\eta_t$\;}
    Set $M_t \coloneqq \bm{0}\in\mathbb{R}^d$; $s\coloneqq a_t$\;
    \Repeat{$s=\bm{0}$}{
        $M_t \coloneqq M_t + s$\;
        Sample $r_t^{\prime} = (r_{t,1}^{\prime}, r_{t,2}^{\prime}, \dots, r_{t,d}^{\prime})$ i.i.d. from $\mathcal{D}$\;
        $a'_t = \argmin_{a \in \nec} \left\{a^\top (\eta_t \hat{L}_t - r'_t)\right\}$\;
        $s\coloneqq s\circ\qty(\bm{1}_d-a'_t)$\tcp*{$\bm{1}_d$ denotes the $d$-dimensional all-ones vector}
    }
    Set $\widehat{w_{t,i}^{-1}} \coloneqq M_{t,i}$ for all $i$ such that $a_{t,i}=1$\;
\end{algorithm}

Since the loss at each round is partially observable in the setting of semi-bandit feedback, many policies use an unbiased (or low-bias) estimator $\hat{\ell}_t$ for the loss vector $\ell_t$. 
In the combinatorial semi-bandit problem, many policies like FTRL \citep{zimmert2019beating} often employ an importance-weighted (IW) estimator 
$\hat{\ell}_t=\sum_{i:a_i=1}\qty(\ell_{t,i}/w_{t,i})e_{t,i}$, when the base-arm selection probability $w_{t,i}$ is explicitly computed. 
For FTPL policy where $w_{t,i}$ is not explicitly available, \citet{JMLR:v17:15-091} proposed a technique called Geometric Resampling (GR). 
In GR, another perturbation $r_t'$ is repeatedly sampled from $\mathcal{D}$ until $i$ becomes a base-arm of $\argmin_{a \in \nec} \{a^\top (\eta_t \hat{L}_t - r_t')\}$. 
Then, by setting $\widehat{w_{t,i}^{-1}}$ as the number of resampling, it serves as an unbiased estimator for $w_{t,i}^{-1}$. 
The procedure of GR is shown in Algorithm~\ref{alg:GR}, where the notation $a\circ b$ denotes the element-wise product of two vectors $a$ and $b$, i.e., $\qty(a\circ b)_i = a_i b_i$ for all $i$. 

For combinatorial semi-bandits with an arbitrary action set, the number of resampling $S_t=\max_{i:a_{t,i}=1}M_{t,i}$ under GR satisfies $\mathbb{E}[S_t|\hat{L}_t]\leq d$ \citep[Proposition~3]{JMLR:v17:15-091}. 
Moreover, each resampling step consists of generating a $d$-dimensional perturbation vector $r_t'$ and finding $\argmin_{a \in \nec} \{a^\top (\eta_t \hat{L}_t - r'_t)\}$, where the former operation requires $O(d)$ time. 
The computational cost of the latter operation depends on the structure of the action set $\nec$. 
For $m$-set semi-bandits, this operation amounts to selecting the top-$m$ smallest components of a $d$-dimensional vector and does not require sorting. 
This requires $O(d\log d)$ time by a naive sorting algorithm, and it can be further improved to $O(d)$ time by using linear-time selection \citep{blum1973time}.
Consequently, the average complexity of GR at each round can be bounded as
\begin{equation*}
    C_{\text{GR}}=\mathbb{E}[S_t|\hat{L}_t]\cdot O(d)\leq O(d^2),
\end{equation*}
which is independent of $w_t$. 

Compared with many other policies, in $m$-set semi-bandit problems, FTPL with GR is computationally more efficient, thanks to its optimization-free nature. 
While this complexity upper bound $O(d^2)$ of GR does not improve in the case of MAB (case of $m=1$), \citet{chen2025cgr} proposed the technique called Conditional Geometric Resampling (CGR), which reduces the complexity to $O(d \log d)$ for MAB. We extend this technique to $m$-set semi-bandits in the next section.

\subsection{Conditional Geometric Resampling for \texorpdfstring{$m$}{m}-Set Semi-Bandits}\label{subsec:CGR}

Building on the idea proposed by \citet{chen2025cgr}, this section introduces an extension of CGR to the $m$-set semi-bandits. 
This algorithm is designed to provide an unbiased estimator $\{w_{t,i}^{-1}:a_{t,i}=1\}$ in a more efficient way based on the following elementary lemma.
\begin{lemma}\label{lem:cgr_general_idea}
    For any base-arm $i\in[d]$ let $\mathcal{E}_{t,i}$ be an arbitrary necessary condition on perturbation vector $r''_t$ for
    \begin{equation}\label{eq:termination}
        \qty[\argmin_{a \in \nec} \left\{a^\top (\eta_t \hat{L}_t - r''_t)\right\}]_i=1.
    \end{equation}
    Consider resampling of $r_t''$ from $\mathcal{D}$ conditioned on
    $\mathcal{E}_{t,i}$ until \eqref{eq:termination} is satisfied and
let $M_{t,i}$ be the number of resampling.
    Then, $\frac{M_{t,i}}{\sP[\mathcal{E}_{t,i}|\hat{L}_t]}$ is an unbiased estimator for $w_{t,i}^{-1}$, that is,
    \begin{equation*}
        \E[M_{t,i}|\hat{L}_t]=\frac{\sP[\mathcal{E}_{t,i}|\hat{L}_t]}{w_{t,i}}.
    \end{equation*}
Furthermore, its variance satisfies
\begin{align}
        \mathrm{Var}\qty[\frac{M_{t,i}}{\sP[\mathcal{E}_{t,i}|\hat{L}_t]}\m\hat{L}_t]
        &=
            \frac{1}{w_{t,i}^2}-\frac{1}{\sP(\mathcal{E}_{t,i})w_{t,i}}
\le
            \frac{1}{w_{t,i}^2}-\frac{1}{w_{t,i}}.\label{variance_cgr}
\end{align}
\end{lemma}
The proof of this lemma is almost the same as that for MAB in \citet{chen2025cgr}
but given in Appendix~\ref{append_proof_cgr} for self-containedness.
From this lemma we can use $M_{t,i}/\sP[\mathcal{E}_{t,i}|\hat{L}_t]$ as
an unbiased estimator for $w_{t,i}$ and the expected number of resampling becomes
small if $\sP[\mathcal{E}_{t,i}|\hat{L}_t]$ is small.
The regret bounds in the subsequent sections hold
regardless of the use of GR (corresponding to $\mathcal{E}_{t,i}=\mathsf{True}$) or CGR
since we only use the last inequality of \eqref{variance_cgr} in the regret analysis.

\begin{algorithm}[t]
    \caption{Conditional Geometric Resampling}
    \label{alg:CGR}
    \KwIn{Chosen action $a_t$, action set $\nec$, cumulative loss $\hat{L}_t$, learning rate $\eta_t$\;}
    Set $M_t \coloneqq \bm{0}\in\mathbb{R}^d$; $s\coloneqq a_t$; $U\coloneqq\varnothing $; $C\coloneqq\mathbf{1}_d\in\mathbb{R}^d$\;
    Set $U\coloneqq\{i\in[d]: a_{t,i}=1, \sigma_i>m\}$ with $C_i\coloneqq \sigma_i/m$ for $i\in U$\;
\label{line:scan_end}

    \Repeat{$s = \bm{0}$\label{line:perturbation_end}}{\label{line:perturbation_begin}
        $M_t \coloneqq M_t + s$\label{line:global_begin}\;

        Sample $r_t^{\prime} = (r_{t,1}^{\prime}, r_{t,2}^{\prime}, \dots, r_{t,d}^{\prime})$ i.i.d. from $\mathcal{D}$\label{line:sample_d}\;

        $a'_t \coloneqq \argmin_{a \in \nec} \left\{a^\top (\eta_t \hat{L}_t - r'_t)\right\}$\;\label{line:global_end}

        Sample $\theta$ from $[m]$ uniformly at random\label{line:theta}\;

        \For{$i\in U$}{
            Find $i'$ such that $r'_{t,i'}$ is the $\theta\text{-th}$ largest in $\qty{r'_{t,j}:\sigma_j\leq\sigma_i}$\label{line:swap_begin}\;
            Set $r_t''\coloneqq r_t'$\;
            Swap $r''_{t,i'}$ and $r''_{t,i}$\label{swap_end}\label{line:swap_end}\;
            $a'_{t,i}\coloneqq\qty[\argmin_{a \in \nec} \left\{a^\top (\eta_t \hat{L}_t - r''_t)\right\}]_i$\label{line:find_a}\;
            \lIf{$a'_{t,i}=1$}{$U\coloneqq U\setminus\{i\}$}\label{line:inner_end}
        }
        $s \coloneqq s \circ \qty(\mathbf{1}_d-a'_t)$\label{line:global_2}\;
        }

    Set $\widehat{w_{t,i}^{-1}} \coloneqq C_i M_{t,i}$ for all $i$ such that $a_{t,i}=1$\;
\end{algorithm}%
The next lemma gives a construction of $\mathcal{E}_{t,i}$
under which $\sP[\mathcal{E}_{t,i}|\hat{L}_t, a_{t,i}]$ is small and sampling from
the conditional distribution can be realized efficiently by Algorithm~\ref{alg:CGR}.
\begin{lemma}\label{lem:CGR}
For any base arm $i\in[d]$ such that $\sigma_i>m$, let
    $$\mathcal{E}_{t,i}=\qty{\left|\qty{ j:r''_{t,j}\leq r''_{t,i},\sigma_j\leq\sigma_i} \right|\leq m}.$$
Then,
the sample $r''_t$ obtained in Line~\ref{line:swap_end} of Algorithm~\ref{alg:CGR} for $i$ such that $\sigma_i>m$ follows the conditional distribution of $\mathcal{D}$ given $\mathcal{E}_{t,i}$.
In addition, for all $i\in [d]$ such that $a_{t,i}=1$,
the expected number of resampling $M_{t,i}$ satisfies
    \begin{equation}\label{eq:resampling_bound}
        \E_{r''_t\sim\mathcal{D}|\mathcal{E}_{t,i}}
\qty[M_{t,i}\m \hat{L}_{t}]=
\left(\frac{\sigma_i}{m}\lor 1\right)\frac{1}{w_{t,i}}.
    \end{equation}
\end{lemma}
The proof of this lemma is given in Appendix~\ref{append_lem_cgr}.
The first part of this lemma shows that sampling from the conditional distribution
is realized by selection of the $\theta$-th largest element for random index $\theta\in[m]$ (Line~\ref{line:swap_begin})
and value swapping (Line~\ref{line:swap_end}),
the former of which is specific to $m$-set semi-bandits.
With this fact and the second part of this lemma on the number of resampling,
we can bound the average complexity of CGR at each round by
\begin{equation}
C_{\text{CGR}}= O\qty(md\qty(\log\qty(d/m)+1)),\n
\end{equation}
which is detailed in Appendix~\ref{append_complexity}.
Though our main contribution lies in the regret analysis,
this complexity guarantee also demonstrates the practicality of FTPL.

\section{Regret Bounds}

In this section, we first give regret bounds of FTPL with perturbation distribution $\dis\in\qty{\Fdis,\Pdis}$ in $m$-set semi-bandit problem for adversarial setting. Then, we present the regret bounds of FTPL for stochastic setting in Theorems~\ref{thm:sto_bound} and \ref{thm:sto_bound_alpha_not2}, respectively corresponding to $\alpha=2$ and $\alpha\in(1,2)\cup(2,\infty)$. 
Note that, by the argument in Lemma~\ref{lem:cgr_general_idea}, the following theorems hold for FTPL with GR or CGR.

\begin{theorem}\label{thm:adv_bound}
    In the adversarial setting, FTPL with $\dis\in\qty{\Fdis,\Pdis}$ and learning rate $\eta_t=\frac{c}{\sqrt{t}}m^{\frac{1}{2}-\frac{1}{\alpha}}d^{\frac{1}{\alpha}-\frac{1}{2}}$ for $c>0$ and $\alpha>1$ satisfies
    \begin{equation*}
        \R(T)\leq O\qty(\sqrt{mdT}).
    \end{equation*}
\end{theorem}
The proof is provided in Sections~\ref{sec:regret_analysis} and \ref{sec:stab_adv}. 
This result shows that FTPL with Fr\'{e}chet or Pareto distributions with shape $\alpha>1$ achieves the optimality in adversarial $m$-set semi-bandit problem. 
The explicit constants are provided in the proof in Section~\ref{sec:proof_adv_bound}, which are simpler for Pareto distribution than for Fr\'{e}chet distribution.
The main challenge of the proof arises from a component called stability term, as the relation between the base-arm selection probability and its derivative is intricate, which is detailed in Section~\ref{sec:stab_adv}.

The following result shows that FTPL with GR or CGR with $\mathcal{F}_2$ or $\mathcal{P}_2$ can achieve a logarithmic regret in stochastic $m$-set semi-bandit problem.
\begin{theorem}\label{thm:sto_bound} 
    Assume that $a^*=\argmin_{a\in\mathcal{A}}a^\top \mu$ is unique and let $\Delta_i=\mu_i-\max_{j:a^*_j=1}\mu_j$, $\Delta=\min_{i:a^*_i=0}\Delta_i$.
    Then, FTPL with $\dis=\qty{\mathcal{F}_2,\mathcal{P}_2}$ and learning rate $\eta_t=\frac{c}{\sqrt{t}}$ for some constant $c>0$ satisfies
    \begin{equation*}
        \R(T)\leq O\qty(\sum_{i:a^*_i=0}\frac{\log T}{\Delta_i})+O\qty(\frac{m^3d}{\Delta}).
    \end{equation*}
\end{theorem}
Theorems~\ref{thm:adv_bound} and \ref{thm:sto_bound} together indicate that FTPL with $\mathcal{F}_2$ or $\mathcal{P}_2$ achieves BOBW optimality in $m$-set semi-bandit problem.
Compared to the results by \citet{zhan2025follow}, our work is not only the first to show that FTPL can achieve the minimax optimality, but also establishes a logarithmic regret with a different second-order term $O(m^3d/\Delta)$ in stochastic setting instead of $O((m^2d\log d+m^{11/3}+md^2)/\Delta)$ by \citet{zhan2025follow}.
Though neither of them dominates the other, our result maintains linear dependence on $d$. 
This becomes favorable in regimes where $m\ll d$, while the adversarial regret of $O(\sqrt{mdT\log d}+\sqrt{m^{8/3}T})$ in \citet{zhan2025follow} requires $m\ll d$ to approach (though not match) the minimax optimality. 

Beyond the case of $\alpha=2$, we also provide the following results. 
\begin{theorem}\label{thm:sto_bound_alpha_not2}
    Assume that $a^*=\argmin_{a\in\mathcal{A}}a^\top \mu$ is unique and let $\Delta_i=\mu_i-\max_{j:a^*_j=1}\mu_j$, $\Delta=\min_{i:a^*_i=0}\Delta_i$.
    Then, FTPL with $\dis\in\qty{\Fdis,\Pdis}$ and learning rate $\eta_t=\frac{c}{\sqrt{t}}m^{\frac{1}{2}-\frac{1}{\alpha}}d^{\frac{1}{\alpha}-\frac{1}{2}}$ for $c>0$ and $\alpha>2$ satisfies
    \begin{equation*}
        \R(T)\leq O\qty(\sum_{i:a^*_i=0}\frac{1}{\alpha-2}\frac{T^{\frac{\alpha-2}{2(\alpha-1)}}}{\Delta_i^{\frac{1}{\alpha-1}}m^{\frac{\alpha-2}{2(\alpha-1)}}d^{\frac{2-\alpha}{2(\alpha-1)}}}).
    \end{equation*}
    If $\alpha\in(1,2)$, then
    \begin{equation*}
        \R(T)\leq O\qty(\sum_{i:a^*_i=0}\frac{1}{2-\alpha}\frac{T^{1-\frac{\alpha}{2}}}{\Delta_i^{\alpha-1} m^{\frac{\alpha}{2}-1} d^{1-\frac{\alpha}{2}}}).
    \end{equation*}
\end{theorem}
Section~\ref{eq:proof_sto_bound_outline} provides an outline of the proofs of Theorems~\ref{thm:sto_bound} and \ref{thm:sto_bound_alpha_not2}, and the explicit constants are given in the detailed proof in Section~\ref{subsec:proofs_sto_bounds}. 
Although our regret bounds for FTPL with $\alpha\in(1,2)\cup(2,\infty)$ are not logarithmic and therefore do not match the regret lower bound in the stochastic setting, this result indicates that FTPL can achieve better dependence on horizon $T$ than $O(\sqrt{T})$ in adversarial setting since $\frac{\alpha-2}{2(\alpha-1)}<\frac{1}{2}$ for $\alpha>2$ and $1-\frac{\alpha}{2}<\frac{1}{2}$ for $\alpha\in(1,2)$.

\section{Experiments}\label{sec:experiments}

\begin{figure}[t]
  \centering
  \begin{minipage}{\textwidth}
    \centering
    \begin{subfigure}{0.43\textwidth}
      \centering
      \includegraphics[width=0.98\linewidth]{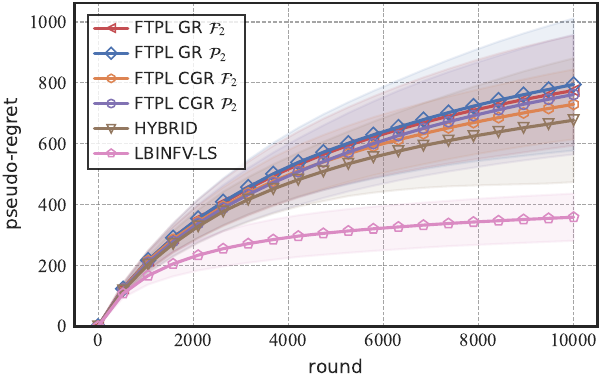}
      \caption{stochastic, $m=3$, $d=16$.}
      \label{fig:regret_sto_3_16}
    \end{subfigure}\quad
    \begin{subfigure}{0.43\textwidth}
      \centering
      \includegraphics[width=0.98\linewidth]{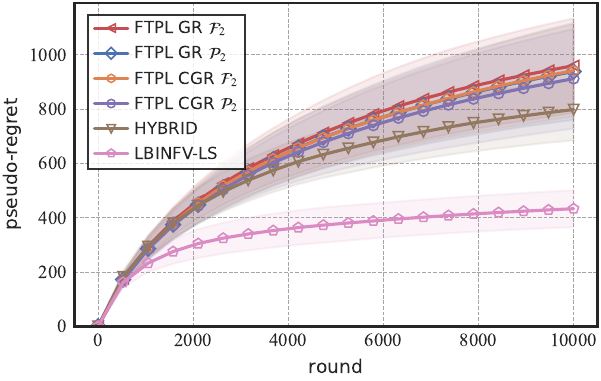}
      \caption{stochastic, $m=5$, $d=20$.}
      \label{fig:regret_sto_5_20}
    \end{subfigure}

    \vskip 2mm

    \begin{subfigure}{0.43\textwidth}
      \centering
      \includegraphics[width=0.98\linewidth]{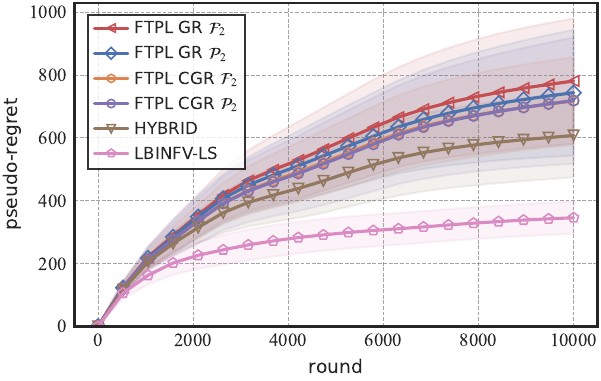}
      \caption{adversarial, $m=3$, $d=16$.}
      \label{fig:regret_adv_3_16}
    \end{subfigure}\quad
    \begin{subfigure}{0.43\textwidth}
      \centering
      \includegraphics[width=0.98\linewidth]{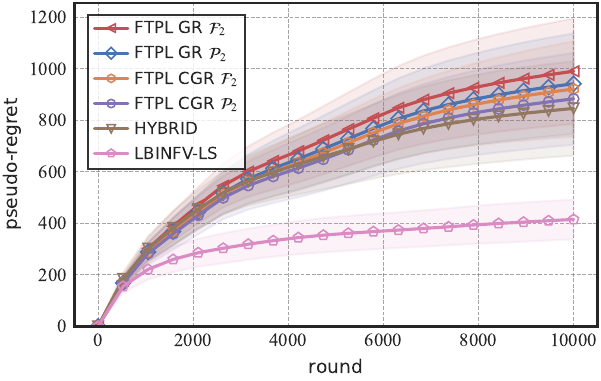}
      \caption{adversarial, $m=5$, $d=20$.}
      \label{fig:regret_adv_5_20}
    \end{subfigure}
    \caption{Pseudo regret.}
    \label{fig:regret}
  \end{minipage}

  \vskip 5mm

  \begin{minipage}{\textwidth}
    \centering
    \includegraphics[width=0.5\textwidth]{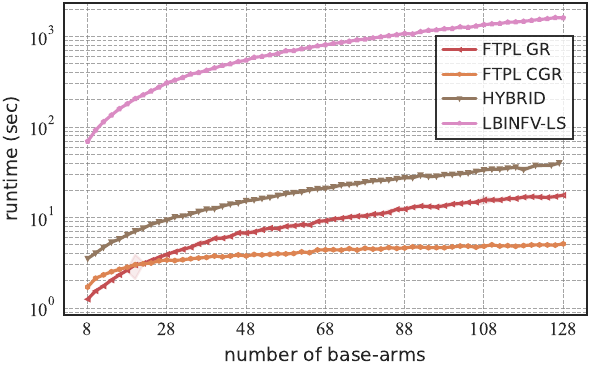}
    \caption{Runtime (sec) for adversarial setting and different $d$.}
    \label{fig:runtime}
  \end{minipage}
\end{figure}

In this section, we present results of experiments conducted to compare the regret and computational efficiency of FTPL with existing BOBW policies, including \HYBRID~\citep{zimmert2019beating} and \LBINFV~\citep{tsuchiya2023further}. 
We examine the performance of FTPL with geometric resampling (\GR) and conditional geometric resampling (\CGR), with perturbation following shape-$2$ Fr\'{e}chet distribution $\mathcal{F}_2$ and Pareto distribution $\mathcal{P}_2$. 
We use the same learning rate $\eta_t=1/\sqrt{t}$ for {\sf{FTPL}} and {\sf{HYBRID}} as that of \citet{zimmert2019beating}. 
Since \LBINFV{} relies on log-barrier regularization, it may suffer from numerical instability with extreme probability values.\kern0.06em\footnote{We used the implementation of \LBINFV{} released by \citet{tsuchiya2023further} available at \url{https://github.com/tsuchhiii/bobw-variance/tree/master}.} 
We also observed errors due to numerical instability for some random seeds in {\sf{HYBRID}} when $d \gtrsim 100$.\kern0.06em\footnote{We used the implementation of {\sf{HYBRID}} released by \citet{zimmert2019beating} available at \url{https://github.com/diku-dk/CombSemiBandits/tree/master}.} 
Therefore, we report the results of \LBINFV{} and {\sf{HYBRID}} over $20$ trials to avoid abnormal termination due to unexpected errors. 
In contrast, {\sf{FTPL}} behaved stably and we report its results over $100$ trials.
The experiments were conducted on an AMD EPYC 7763 CPU, implemented in Python 3.9 using the NumPy library.

\paragraph{Environments} Following \citet{zimmert2019beating} and \citet{tsuchiya2023further}, we consider the stochastic setting and the stochastically constrained adversarial setting, where the losses are sampled from Bernoulli distribution in both settings and the suboptimality gap is decided as $\Delta=0.125$. 
For the stochastic setting, the mean losses are decided as $(1-\Delta)/2$ for $m$ optimal base-arms and $(1+\Delta)/2$ for the remaining $(d-m)$ suboptimal base-arms. 

For the stochastically constrained adversarial setting, the mean losses of the optimal base-arms and suboptimal base-arms switch between $(0,\Delta)$ and $(1-\Delta,1)$, with the duration between alternations increasing exponentially
with a factor of $1.6$ after each switch. 
Both settings are similar to those in \citet{zimmert2019beating} and \citet{tsuchiya2023further}.

\paragraph{Regret}
We compare the empirical regret performance of \GR{} with that of \CGR{}, \HYBRID{} and \LBINFV{}, which is presented in Figure~\ref{fig:regret}.
In all the settings, we can see that \CGR{} always performs sightly better than or almost the same as \GR{}, both of which are a little worse than \HYBRID{}. 
The performance of \LBINFV{} is substantially better than the others, benefiting from its refined regularizer and adaptive learning rate that yields better environment-dependent regret guarantees. 
Extending FTPL to exhibit similar adaptivity remains an important direction for future work.

\paragraph{Computational Efficiency}
Figure~\ref{fig:runtime} shows the runtime for action selection over $10000$ rounds of \GR{}, \CGR{}, \HYBRID{}, and \LBINFV{} for $m=4$ and varying $d$ from $8$ to $128$. 
Since the choice of perturbation distribution has negligible impact on the runtime of {\sf{FTPL}}, we only show the result for $\mathcal{F}_2$ to avoid redundancy. 
When $d$ is small, both \GR{} and \CGR{} run much faster than \HYBRID{} and \LBINFV{}. Especially, the runtime of \LBINFV{} is on the order of tens of times larger than that of the other algorithms.
As $d$ increases, the runtime of \CGR{} remains small thanks to its optimization-free nature and high termination probability, while that of the others grows sharply, resulting in an increasingly pronounced efficiency gap. 

\section{Regret Analysis}\label{sec:regret_analysis}

In this section, we first give the general tools for analysis in Section~\ref{subsec:tool} and the regret decomposition used in both adversarial and stochastic settings in Section~\ref{sec:regret_decomposition}. Next, we give important parts for Theorem~\ref{thm:adv_bound} on the adversarial setting in Sections~\ref{sec:stability}--\ref{sec:proof_adv_bound} and a proof outline for Theorems~\ref{thm:sto_bound} and \ref{thm:sto_bound_alpha_not2} for the stochastic setting in Section~\ref{eq:proof_sto_bound_outline}. The remaining parts of the proofs are given in Sections~\ref{sec:stab_adv} and \ref{sec:stab_sto}.

\subsection{General Tools for Analysis}\label{subsec:tool}
In this part, we introduce some key tools for the subsequent regret analysis of FTPL. 
Recall that we denote the rank of $i$-th element of $\bm{u}$ among index set $\B$ in descending order as $\sigma_i(\bm{u},\B)$, that is, $u_{i}$ is the $\sigma_i(\bm{u},\B)$-th largest among $\{u_j:j\in\B\}$. We omit $\bm{u}$ (resp.~$\B$) when $\bm{u}=-\hat{L}_t$ (resp.~$\B=[d]$).
The probability that $r-\eta_t\hat{L}_t$ becomes the $\theta$-th smallest is written as $\phi_{i,\theta}(\eta_t\hat{L}_t;\dis)$, where for $\lambda\in[0,\infty)^d$
\begin{align*}
    \phi_{i,\theta}(\lambda;\dis)
    \coloneqq
    \sP_{r=\qty(r_1,\dots,r_d) \sim \dis} \qty[ \sigma_i\qty(r-\lambda ) = \theta].
\end{align*}

Then, we can write the probability vector of selecting base-arms as $w_t=\phi(\eta_t\hat{L}_t;\dis)=(\phi_1(\eta_t\hat{L}_t;\dis),\cdots,\phi_d(\eta_t\hat{L}_t;\dis))$, where for $\lambda\in[0,\infty)^d$
\begin{align*}
    \phi_i(\lambda;\dis) &\coloneqq
    \mathbb{P}_{r=\qty(r_1,\dots,r_d) \sim \dis} \qty[ \sigma_i\qty(r-\lambda ) \leq m]
    =\sum_{\theta=1}^{m}\phi_{i,\theta}(\lambda;\dis).
    \numberthis\label{eq:arm_probability}
\end{align*}
In addition, for $\lambda\in[0,\infty)^d$ we define
\begin{align*}
    I_i(\lambda;\dis) &\coloneqq
    \mathbb{P}_{r=\qty(r_1,\dots,r_d) \sim \dis} \qty[ \sigma_i\qty(r-\lambda) \leq m,\,r_i \ge \nu+\lambda_i].
    \numberthis\label{eq:definition_I}
\end{align*}
Since the probabilities of some events for different sizes of each action $\wm$ and base-arm sets $\B\subset[d]$ will be considered in the subsequent analysis, we introduce the parameters $\wm$ and $\B$ into $I_i(\cdot)$. 
For $\B\ni i$ and $\widetilde{\lambda}_1\leq\widetilde{\lambda}_2\leq\cdots\leq\widetilde{\lambda}_d$ denoting the sorted elements of $\lambda\in\mathbb{R}^d$ we define
\begin{align*}
    \phi_i(\lambda;\dis,\wm,\B)
    &\coloneqq
    \sP_{r=\qty(r_1,\dots,r_d) \sim \dis} \qty[ \sigma_i\qty(r-\lambda,\B) \leq \wm]\\
    &=
    \int_{\nu-\widetilde{\lambda}_{\wm}}^\infty 
    \sP_{\{r_k\}_{k\in[d]\setminus\{i\}} \sim \dis,\,r_i=z+\lambda_i}
    \qty[\sigma_i\qty(r-\lambda,\B)\leq\wm]
    \dd F(z+\lambda_i),\numberthis\label{eq:definition_phi_extension_no_underline}
    \displaybreak[0]\\
    &=
    \int_{\nu}^\infty 
    \sP_{\{r_k\}_{k\in[d]\setminus\{i\}} \sim \dis,\,r_i=z+\ul_i}
    \qty[\sigma_i\qty(r-\lambda,\B)\leq\wm]
    \dd F(z+\ul_i),\numberthis\label{eq:definition_phi_extension}
    \displaybreak[0]\\
    I_i(\lambda;\dis,\wm,\B)
    &\coloneqq
    \sP_{r=\qty(r_1,\dots,r_d) \sim \dis} \qty[ \sigma_i\qty(r-\lambda,\B) \leq \wm,\,r_i \ge \nu+\lambda_i]\\
    &=
    \int_{\nu}^\infty 
    \sP_{\{r_k\}_{k\in[d]\setminus\{i\}} \sim \dis,\,r_i=z+\lambda_i}
    \qty[\sigma_i\qty(r-\lambda,\B)\leq\wm]
    \dd F(z+\lambda_i),
\end{align*}
where 
$\nu=0$ (resp. $\nu=1$) for 
Fr\'{e}chet (resp. Pareto) distribution 
denotes the left endpoint of the support of $F$. 
Here, underlines in \eqref{eq:definition_phi_extension} denote the gap of a vector
from its $\wm$-th smallest element, i.e., $\ul_i=\lambda_i-\widetilde{\lambda}_{\wm}$ for all $i\in[d]$. 
The definitions above are extensions of \eqref{eq:arm_probability} and \eqref{eq:definition_I}, and we have
\begin{equation*}
    \phi_i(\lambda;\dis)=\phi_i(\lambda;\dis,m,[d])\quad\text{and}\quad I_i(\lambda;\dis)=I_i(\lambda;\dis,m,[d]).
\end{equation*}
Under the definitions above, 
one can see that
\begin{equation}\label{eq:phi_I_relation}
    \phi_i(\lambda;\dis,\wm,\B)=I_i(\ul;\dis,\wm,\B).
\end{equation}

For the analysis of the derivative, we define
\begin{multline}\label{eq:definition_J}
    J_i(\lambda;\dis,\wm,\B)\coloneqq
    \E_{r=\qty(r_1,\dots,r_d) \sim \dis} \qty[ \frac{1}{r_i}\ind\qty[\sigma_i\qty(r-\lambda; \B) \le \wm,\,r_i\ge \nu+\lambda_i]]\\
    =\int_{\nu}^\infty \frac{1}{z+\lambda_i}\sP_{\{r_k\}_{k\in[d]\setminus\{j\}} \sim \dis,\,r_i=z+\lambda_i}
    \qty[\sigma_i\qty(r-\lambda,\B)\leq\wm] \dd F(z+\lambda_i),
\end{multline}
which corresponds to $\frac{\partial }{\partial \lambda_i}I_i(\lambda;\dis,\wm,\B)$ up to some constant factor. 
When $\wm=m$ and $\B=[d]$, we simply write $J_i(\lambda;\dis)=J_i(\lambda;\dis,m,[d])$.

Next, we define
\begin{equation*}
    \phi_{i,\theta,j}(\lambda;\dis,\B) 
    \coloneqq
    \sP_{r=\qty(r_1,\dots,r_d) \sim \dis}\qty[\sigma_i\qty(r-\lambda,\B)=\theta\leq\sigma_j\qty(r-\lambda,\B)],
\end{equation*}
which can be expressed as
\begin{multline}\label{eq:prob_theta_j}
    \phi_{i,\theta,j}(\lambda;\dis,\B)\\
    =
    \int_{\nu-\widetilde{\lambda}_{\wm}}^\infty F(z+\lambda_j)
    \sP_{\{r_k\}_{k\in[d]\setminus\{i,j\}} \sim \dis,\,r_i=z+\lambda_i}
    \qty[\sigma_i\qty(r-\lambda,\B\setminus\{j\})=\theta]
    \dd F(z+\lambda_i)
\end{multline}
if $\theta\leq\wm$. 
In addition, we define
\begin{multline}\label{eq:I_theta_j}
    I_{i,\theta,j}(\lambda;\dis,\B) 
    \coloneqq
    \sP_{r=\qty(r_1,\dots,r_d) \sim \dis}\qty[\sigma_i\qty(r-\lambda,\B)=\theta\leq\sigma_j\qty(r-\lambda,\B),\,r_i \ge \nu+\lambda_i]\\
    =
    \int_{\nu}^\infty F(z+\lambda_j)
    \sP_{\{r_k\}_{k\in[d]\setminus\{i,j\}} \sim \dis,\,r_i=z+\lambda_i}
    \qty[\sigma_i\qty(r-\lambda,\B\setminus\{j\})=\theta]
    \dd F(z+\lambda_i).
\end{multline}
Corresponding to \eqref{eq:I_theta_j}, we define
\begin{multline}\label{eq:prob_k_J}
    J_{i,\theta,j}(\lambda;\dis,\B)
    \coloneqq
    \E_{r=\qty(r_1,\dots,r_d) \sim \dis} \qty[ \frac{1}{r_i}\ind\qty[\sigma_i\qty(r-\lambda,\B)=\theta\leq\sigma_j\qty(r-\lambda,\B),\,r_i\ge \nu+\lambda_i]]\\
    =
    \int_{\nu}^\infty \frac{1}{z+\lambda_i}F(z+\lambda_j)\sP_{\{r_k\}_{k\in[d]\setminus\{i,j\}} \sim \dis,\,r_i=z+\lambda_i}
    \qty[\sigma_i\qty(r-\lambda,\B\setminus\{j\})=\theta]\dd F(z+\lambda_i).
\end{multline}
We can also express the integrands above in an explicit way by using $F(z+\lambda_k)$ or $F(z+\ul_k)$ for $k\in[d]\setminus \{i\}$, but we do not use such expressions in the following analysis.

The following lemma establishes some relations among the above definitions.
\begin{lemma}\label{lem:tool_relation}
    For $\B\supset  \{i,j\}$ and $\wm\le \left\lvert\B\right\rvert$, we have
    \begin{align*}
        \phi_i(\lambda;\dis,\wm,\B)&=
        \phi_i(\lambda;\dis,\wm-1,\B\setminus\qty{j})+\phi_{i,\wm,j}(\lambda;\dis,\B),\\
        I_i(\lambda;\dis,\wm,\B)&=
        I_i(\lambda;\dis,\wm-1,\B\setminus\qty{j})+I_{i,\wm,j}(\lambda;\dis,\B),\\
        J_i(\lambda;\dis,\wm,\B)&=
        J_i(\lambda;\dis,\wm-1,\B\setminus\qty{j})+J_{i,\wm,j}(\lambda;\dis,\B).
    \end{align*}
\end{lemma}

\begin{proof2}
    Since we have
\begin{align}
\qty{\sigma_i\qty(r-\lambda,\B)\le \wm}
=
\qty{\sigma_i\qty(r-\lambda,\B\setminus\{j\})\le \wm-1}
\cup
\qty{\sigma_i\qty(r-\lambda,\B)= \wm<\sigma_j\qty(r-\lambda,\B)}
\n
\end{align}
and the two events in the RHS are disjoint, we can decompose $\phi_i(\cdot)$ as
    \begin{align*}
        \phi_i(\lambda;\dis,\wm,\B)
        &=
        \sP_{r=\qty(r_1,\dots,r_d) \sim \dis}\qty[\sigma_i\qty(r-\lambda,\B)\le \wm]\\
        &=
        \sP_{r=\qty(r_1,\dots,r_d) \sim \dis}\qty[\sigma_i\qty(r-\lambda,\B\setminus\{j\})\le \wm-1]\\
        &\hspace{2cm}+\sP_{r=\qty(r_1,\dots,r_d) \sim \dis}\qty[\sigma_i\qty(r-\lambda,\B)= \wm<\sigma_j\qty(r-\lambda,\B)]
        \numberthis\label{eq:phi_decompose_prob}
        \\
        &=\phi_i(\lambda;\dis,\wm-1,\B\setminus\qty{j})+\phi_{i,\wm,j}(\lambda;\dis,\B)
    \end{align*}
and similarly
\begin{align*}
    I_i(\lambda;\dis,\wm,\B)&=
        I_i(\lambda;\dis,\wm-1,\B\setminus\qty{j})+I_{i,\wm,j}(\lambda;\dis,\B),\\
    J_i(\lambda;\dis,\wm,\B)
    &=J_i(\lambda;\dis,\wm-1,\B\setminus\qty{j})+J_{i,\wm,j}(\lambda;\dis,\B).\dqed
\end{align*}
\end{proof2}

In the following, we simply write $\sigma_i$ to denote the number of arms (including $i$ itself) whose cumulative losses do
not exceed $\hat{L}_{t,i}$, i.e., $\sigma_i=\sigma_i(-\hat{L}_t,[d])$.
To derive an upper bound, we employ the tools introduced above to provide lemmas related to the relation between the base-arm selection probability and its derivatives.

\subsection{Regret Decomposition}\label{sec:regret_decomposition}

To evaluate the regret of FTPL, we firstly decompose the regret which is expressed as
\begin{equation*}
    \mathcal{R}(T)=\E\qty[\sum_{t=1}^{T}\left\langle \ell_t,a_t-a^*\right\rangle]=\sum_{t=1}^{T}\E\qty[\left\langle \ell_t,a_t-a^*\right\rangle]=\sum_{t=1}^{T}\E\qty[\left\langle \hat{\ell}_t,a_t-a^*\right\rangle].
\end{equation*}
This can be decomposed in the following way. 
Firstly, similarly to Lemma~3 in \citet{pmlr-v201-honda23a}, we prove the general framework of the regret decomposition that can be applied to general distributions.
Our decomposition avoids the extra terms in the earlier analyses \citep{pmlr-v201-honda23a, pmlr-v247-lee24a}, similar in spirit to the recent simplifications \citep{kim2025follow, zhan2025follow}, but achieved through more direct arguments without relying on auxiliary steps through the notions of potential function and its convex conjugate.

\begin{lemma}\label{lem:regret_decomposition}
    For any $\alpha>1$ and $\dis\in\qty{\Fdis,\Pdis}$,
    \begin{align*}
        \R(T) \leq \sum_{t=1}^{T} &\mathbb{E} \qty[ \left\langle \hat{\ell}_t, \phi\qty(\eta_t \hat{L}_t ; \dis) - \phi\qty(\eta_t (\hat{L}_t + \hat{\ell}_t) ; \dis) \right\rangle ]  \\
        &+ \sum_{t=1}^T \qty(\frac{1}{\eta_{t+1}}-\frac{1}{\eta_t}) \E_{r_{t+1}\sim\dis}\qty[\left\langle r_{t+1}, a_{t+1} - a^*\right\rangle]
        + \frac{\E_{r_1 \sim \dis} \qty[ a_1^\top r_1 ]}{\eta_1}.\numberthis\label{eq:regret_decomposition}
    \end{align*}
    Here,
    \begin{equation*}
        \E_{r\sim\dis}\qty[a_1^\top r_1]
        \leq
        \begin{cases}
            \qty(\frac{\alpha}{\alpha-1}(m-1)^{1-\frac{1}{\alpha}}+\Gamma\qty(1-\frac{1}{\alpha}))\qty(d+1)^{\frac{1}{\alpha}}, & \text{if } \dis = \Pdis, \\
            \qty(\frac{\alpha}{\alpha-1}(m-1)^{1-\frac{1}{\alpha}}+\Gamma\qty(1-\frac{1}{\alpha}))\qty(d+1)^{\frac{1}{\alpha}}+m, & \text{if } \dis = \Fdis,
        \end{cases}
    \end{equation*}
    where $\Gamma(\cdot)$ is the gamma function.
\end{lemma}

Following the terminology used in the analysis of BOBW policies \citep{zimmert2019beating,zimmert2021tsallis,pmlr-v201-honda23a,pmlr-v247-lee24a}, we refer to the first and second term of \eqref{eq:regret_decomposition} as \textit{stability term} and \textit{penalty term}, respectively.
\medskip

\begin{proof}
    Let us consider random variable $r \in [0, \infty)^d$ that independently follows $\dis$, and is independent from the randomness $\{\ell_t, r_t\}_{t=1}^{T}$ of the environment and the policy.
    Define for $s \in \mathbb{N}$
    \begin{equation}\label{eq:def_of_h_u}
        h_s(w) : w \in \Delta_d \mapsto \left\langle \hat{L}_t - \frac{r}{\eta_t}, w \right\rangle, \qq{and} u_s := \argmin_{w \in \Delta_d} h_s(w),
    \end{equation}
    where $\Delta_d = \{ p \in [0,1]^d : \sum_{i \in [d]} p_i= m \}$. Then, since $r_t$ and $r$ are identically distributed given $\hat{L}_t$, we have
    \begin{equation}\label{eq:decom_expectation}
    \mathbb{E} [ u_t | \hat{L}_t ] = w_t, \quad \mathbb{E} [ \langle r, u_t \rangle | \hat{L}_t ] = \mathbb{E} [ a_t^\top r | \hat{L}_t ].
    \end{equation}
Denote the optimal action as $a^*$. Recalling $\hat{L}_t = \sum_{s=1}^{t} \hat{\ell}_s$, we have
\begin{align*}
\R(T)= \sum_{t=1}^{T} \E\qty[\left\langle \hat{\ell}_t, w_t - a^* \right\rangle] = \sum_{t=1}^{T} \E\qty[ \E\qty[\left\langle \hat{\ell}_t, u_t - a^* \right\rangle \m \hat{L}_t] ].
\end{align*}
Here, we have
\begin{align*}
    \sum_{t=1}^T \left\langle \hat{\ell}_t, u_t - a^* \right\rangle &= \sum_{t=1}^T \left\langle \hat{\ell}_t, u_t - u_t' \right\rangle + \sum_{t=1}^T \left\langle \hat{\ell}_t, u_t' - a^* \right\rangle \\
    &=  \sum_{t=1}^T \left\langle \hat{\ell}_t, u_t - u_t' \right\rangle +  \sum_{t=1}^T \left\langle \hat{\ell}_t, u_t' - u_{t+1} \right\rangle + \sum_{t=1}^T \left\langle \hat{\ell}_t, u_{t+1} - a^* \right\rangle, \numberthis{\label{eq: decompose overall}}
\end{align*}
where $u_t' = \argmin_{w \in \Delta_d} \left\langle \hat{L}_{t+1} - \frac{r}{\eta_t}, w \right\rangle$ so that $\E\qty[u_t'\m \hat{L}_{t+1}] = \phi(\eta_t \hat{L}_{t+1})$.

For any $w \in \Delta_d$, by definition of $h_t(\cdot)$ in (\ref{eq:def_of_h_u}), we have
\begin{align*}
   \left\langle \hat{\ell}_{t}, w \right\rangle = \left\langle \hat{L}_{t+1} - \hat{L}_t, w \right\rangle = h_{t+1}(w) - h_t(w) + \qty(\frac{1}{\eta_{t+1}}- \frac{1}{\eta_t}) \left\langle r, w \right\rangle,
\end{align*}
which implies
\begin{align*}
    \left\langle \hat{\ell}_t, u_{t+1} - a^* \right\rangle &= \left\langle \hat{L}_{t+1} - \hat{L}_t, u_{t+1} - a^* \right\rangle \\
    &= h_{t+1}(u_{t+1}) - h_t(u_{t+1}) - h_{t+1}(a^*) + h_t(a^*) + \qty(\frac{1}{\eta_{t+1}}- \frac{1}{\eta_t}) \left\langle r, u_{t+1} - a^* \right\rangle.
\end{align*}
Its summation over $t$, which is the third term of the RHS in (\ref{eq: decompose overall}), satisfies
\begin{align*}
    \sum_{t=1}^T \left\langle \hat{\ell}_t, u_{t+1} - a^* \right\rangle &= \sum_{t=1}^T h_t(u_t) - h_t(u_{t+1})  + \qty(\frac{1}{\eta_{t+1}}- \frac{1}{\eta_t}) \left\langle r, u_{t+1} - a^* \right\rangle \\
    &\hspace{3.7cm}+ h_{T+1}(u_{T+1}) - h_{T+1}(a^*) + h_1(a^*) - h_1(u_1) \\
    &\leq \sum_{t=1}^T h_t(u_t) - h_t(u_{t+1})  + \qty(\frac{1}{\eta_{t+1}}- \frac{1}{\eta_t}) \left\langle r, u_{t+1} - a^* \right\rangle + \frac{\left\langle r, u_{1} - a^* \right\rangle}{\eta_1}
\end{align*}
since $u_s = \argmin_w h_s(w)$ and $\hat{L}_1 = 0$, which implies
\begin{equation*}
    h_{T+1}(u_{T+1}) - h_{T+1} (a^*) \leq 0, \qq{and} h_1(a^*) - h_1(u_1) = \frac{\left\langle r, u_{1} - a^* \right\rangle}{\eta_1}.
\end{equation*}
By injecting this result into (\ref{eq: decompose overall}), we obtain
\begin{align*}
    \sum_{t=1}^T \left\langle \hat{\ell}_t, u_t - a^* \right\rangle &\leq \sum_{t=1}^T \left\langle \hat{\ell}_t, u_t - u_t' \right\rangle +  \sum_{t=1}^T \left\langle \hat{\ell}_t, u_t' - u_{t+1} \right\rangle + h_t(u_t) - h_{t+1} (u_{t+1})  \\
     &\qquad+ \sum_{t=1}^T \qty(\frac{1}{\eta_{t+1}}- \frac{1}{\eta_t}) \left\langle r, u_{t+1} - a^* \right\rangle + \frac{\left\langle r, u_{1} - a^* \right\rangle}{\eta_1}. \numberthis{\label{eq: decompose overall2}}
\end{align*}
By definition of $h_t(\cdot)$ and $u_t=\argmin_w h_t(w)$, we have
\begin{align*}
    \left\langle \hat{\ell}_t, u_t' - u_{t+1} \right\rangle + h_t(u_t) - h_{t+1} (u_{t+1}) &\leq \left\langle \hat{\ell}_t, u_t' - u_{t+1} \right\rangle + h_t(u_t') - h_{t+1} (u_{t+1})\\  
    &=  \left\langle \hat{\ell}_t, u_t' - u_{t+1} \right\rangle +  \left\langle \hat{L}_t - \frac{r}{\eta_t}, u_t' - u_{t+1} \right\rangle \\
    &= \left\langle \hat{L}_{t+1}- \frac{r}{\eta_t}, u_t' - u_{t+1} \right\rangle \\
    &\leq 0. \tag{by $u_t' = \argmin_w \langle \hat{L}_{t+1}- r/\eta_t, w \rangle$}
\end{align*}
By injecting this result into (\ref{eq: decompose overall2}), we have
\begin{align*}
     \sum_{t=1}^T \left\langle \hat{\ell}_t, u_t - a^* \right\rangle &\leq \sum_{t=1}^T \left\langle \hat{\ell}_t, u_t - u_t' \right\rangle +\sum_{t=1}^T \qty(\frac{1}{\eta_{t+1}}- \frac{1}{\eta_t}) \left\langle r, u_{t+1} - a^* \right\rangle + \frac{\left\langle r, u_{1} - a^* \right\rangle}{\eta_1} \\
     &\leq \sum_{t=1}^T \left\langle \hat{\ell}_t, u_t - u_t' \right\rangle +\sum_{t=1}^T \qty(\frac{1}{\eta_{t+1}}- \frac{1}{\eta_t}) \left\langle r, u_{t+1} - a^* \right\rangle + \frac{\left\langle r, u_{1} \right\rangle}{\eta_1} 
\end{align*}
when the perturbation is non-negative. 
Note that the expectation of $a_1^\top r_1$ in \eqref{eq:decom_expectation} can be bounded by applying Lemma~\ref{lem:max_sum_perturbation} given in Appendix~\ref{app:lemmas}, since it is not larger than the sum of the largest $m$ perturbations among $r_{1,1},r_{1,2},\dots,r_{1,d}$.
Therefore, using \eqref{eq:decom_expectation} and taking the expectation with respect to $r$ concludes the proof.
\end{proof}

\subsection{Stability Term}\label{sec:stability}

In the standard multi-armed bandit problem, the central challenge in the regret analysis of FTPL lies on the analysis of the relation between arm-selection probability function and its derivatives \citep{abernethy2015fighting,pmlr-v201-honda23a,pmlr-v247-lee24a}. 
This challenge is further amplified in the $m$-set semi-bandit problem, where the base-arm selection probability, whose explicit form is given in \eqref{eq:definition_phi_extension}, exhibits significantly greater complexity. 
To this end, here we provide some key properties of $\phi_i(\lambda;\dis)=I_i(\ul;\dis)$, whose derivation is a central technical difficulty in the analysis of FTPL.

The following lemma analyzes the ratio function $J_i(\ul;\dis)/I_i(\ul;\dis)$, showing that it can be upper bounded by a maximum of several simpler ratio functions, where the parameters satisfy some strong constraint. 
For notational simplicity, we define $[a:b]=\{a,a+1,\dots,b\}$ if $b\ge a$ and $[a:b]=\varnothing$ otherwise.
\begin{lemma}\label{lem:lambda_star}
    For $\lambda\in\mathbb{R}^d$, it holds that
    \begin{equation}\label{eq:lambda_star}
        \frac{J_i(\ul;\dis)}{I_i(\ul;\dis)}\leq
        \max_{(j,\wm)\in[m\land i]^2: j+\wm\le (m\land i)+1}
        \qty{\frac{J_i(\lambda^*;\dis,\wm,[j:i])}{I_i(\lambda^*;\dis,\wm,[j:i])}},
    \end{equation}
    where
    \begin{equation*}
        \lambda_k^* = 
        \begin{cases}
            \ul_i, & \text{if } k\leq i,\\
            \ul_k, & \text{if } k> i.
        \end{cases}
    \end{equation*}
\end{lemma}

This is one of the most challenging parts of the proof in this paper, which also plays a crucial role in the regret analysis of FTPL.
Actually, in multi-armed bandits (i.e. the case of $m=1$), the ratio function $J_i(\lambda;\dis,1,[d])/I_i(\lambda;\dis,1,[d])$ was shown to be monotonically increasaing with respect to $\lambda_j$ ($j\neq i$) under some mild assumptions on the perturbation distribution, which directly leads to a structurally simpler upper bound \citep{pmlr-v201-honda23a,pmlr-v247-lee24a}, compared to the form of \eqref{eq:lambda_star}.
However, in $m$-set semi-bandits, such monotonicity property is very complicated to analyze, which makes the analysis more challenging.
Instead, we show that $J_i(\lambda_;\dis)/I_i(\lambda_;\dis)$ has a property similar to the monotonicity in $\lambda_j$ ($j\neq i$) with an extra component related to the case of $d-1$ base-arms, and we prove this lemma by recursively applying this property. 
Further explanation and the detailed proofs are provided in Sections~\ref{subsec:proof_monotonicity} and \ref{subsec:proof_lambda_star}.

By Lemma~\ref{lem:lambda_star}, the task of upper bounding $J_i(\ul;\dis)/I_i(\ul;\dis)$ is greatly simplified, since it suffices to consider ratio functions $J_i(\lambda^*;\cdot)/I_i(\lambda^*;\cdot)$, which have a simpler structure thanks to the independence on $\ul_j$ for all $j\neq i$. 
Building on the analysis of the upper bound given in Lemma~\ref{lem:lambda_star}, we obtain the following lemma, which is the key to the main theorems.
\begin{lemma}\label{lem:sigma_i}
    For $\dis\in\{\Fdis,\Pdis\}$, 
    \begin{equation*}
        \frac{J_i(\ul;\dis)}{I_i(\ul;\dis)}
        \leq
        \frac{1}{0\lor\ul_i}.
    \end{equation*}
    Moreover, if $\sigma_i=\sigma_i(-\lambda)$, i.e., $\lambda_i$ is the $\sigma_i$-th smallest among $\lambda_1,\dots,\lambda_d$, it holds that
    \begin{equation*}
        \frac{J_i(\ul;\dis)}{I_i(\ul;\dis)}
        \leq 
        \begin{cases}
            \qty(49+\frac{7\cdot 4^{1/\alpha}}{1-1/\alpha})(1\land m/\sigma_i)^{1/\alpha}, & \text{if }\dis=\Fdis,\\
            \frac{9}{2}(1\land m/\sigma_i)^{1/\alpha}, &\text{if } \dis=\Pdis.
        \end{cases}
    \end{equation*}
\end{lemma}

The proof is given in Section~\ref{subsec:proof_sigma_i}. This upper bound follows from taking the maximum of the bounds on the individual terms appearing in the RHS of \eqref{eq:lambda_star}, which reveals that the case $(j,\tilde{m})=(1, m \land i)$ is dominant in \eqref{eq:lambda_star}.

Adopting a more general version of the argument in \citet{pmlr-v201-honda23a} and \citet{pmlr-v247-lee24a} with the result from Lemma~\ref{lem:sigma_i}, we derive the following lemma, which provides an upper bound for each component of the stability term. 

\begin{lemma}\label{lem:adv_bound_term}
    For any $\alpha>1$ and $\dis\in\qty{\Fdis,\Pdis}$, if $\sigma_i=\sigma_i(-\hat{L}_t)$, i.e., $\hat{L}_t$ is the $\sigma_i$-th smallest among $\{\hat{L}_{t,j}\}_{j\in[d]}$, then
    \begin{equation*}
        \E\qty[\hat{\ell}_{t,i}\qty(\phi_i\qty(\eta_t \hat{L}_t;\dis)-\phi_i\qty(\eta_t \qty(\hat{L}_t+\hat{\ell}_t);\dis))\m \hat{L}_t]
        \leq\\
        \frac{2(\alpha+1)}{0\lor\underline{\hat{L}}_{t,i}}
        \land
        \acon\eta_t\qty(1\land\frac{m}{\sigma_i})^{\frac{1}{\alpha}},
    \end{equation*}
    where
    \begin{equation*}
        \acon=
        \begin{cases}
            2(\alpha+1)\qty(49+\frac{7\cdot 4^{1/\alpha}}{1-1/\alpha}), & \text{if }\dis=\Fdis,\\
            9(\alpha+1), &\text{if } \dis=\Pdis.
        \end{cases} 
    \end{equation*}
\end{lemma}

Here, recall that $\underline{\hat{L}}_{t,i}$ denotes the gap between $\hat{L}_{t,i}$ and the $m$-th smallest element of $\hat{L}_t$. 
Therefore, $\underline{\hat{L}}_{t,i}$ can be non-positive, in which case we define $1/(0\lor\underline{\hat{L}}_{t,i})=1/0=\infty$.
By using the second term of this bound we can immediately bound the stability term as shown in Lemma~\ref{lem:stability} below, which is used in both adversarial and stochastic settings. 
Furthermore, we additionally use the first term in stochastic setting to obtain a tighter bound by the self-bounding technique. 

\medskip

\begin{proof}
    Define
    \begin{equation*}
        \underline{\Omega}=\qty{r:\qty[\argmin\limits_{a \in \nec} \left\{a^\top (\eta_t (\hat{L}_t + (\ell_{t,i}\widehat{w_{t,i}^{-1}}) e_i) - r)\right\}]_i=1}
    \end{equation*}
    and
    \begin{equation*}
        \overline{\Omega } =\qty{r:\qty[\argmin\limits_{a \in \nec} \left\{a^\top (\eta_t (\hat{L}_t + \hat{\ell}_t) - r)\right\}]_i=1}.
    \end{equation*}
    Then, we have
    \begin{equation*}
        \phi_i\qty(\eta_t\qty(\hat{L}_t + \qty(\ell_{t,i}\widehat{w_{t,i}^{-1}}) e_i);\dis)=\sP_{r\sim\dis}\qty(\underline{\Omega}),\quad \phi_i\qty(\eta_t\qty(\hat{L}_t+\hat{\ell}_t);\dis)=\sP_{r\sim\dis}\qty(\overline{\Omega}).
    \end{equation*}
    Since $\underline{\Omega}\subset\overline{\Omega}$, we immediately have
    \begin{equation*}
        \phi_i\qty(\eta_t\qty(\hat{L}_t + \qty(\ell_{t,i}\widehat{w_{t,i}^{-1}}) e_i);\dis)\leq\phi_i\qty(\eta_t\qty(\hat{L}_t+\hat{\ell}_t);\dis),
    \end{equation*}
    with which we have
    \begin{align*}
        \phi_i\qty(\eta_t\hat{L}_t;\dis)-\phi_i\qty(\eta_t\qty(\hat{L}_t+\hat{\ell}_t);\dis)
        &\leq\phi_i\qty(\eta_t\hat{L}_t;\dis) - \phi_i\qty(\eta_t\qty(\hat{L}_t + \qty(\ell_{t,i}\widehat{w_{t,i}^{-1}}) e_i);\dis)\\ 
        &= \int_0^{\eta_t\ell_{t,i} \widehat{w_{t,i}^{-1}}} -\frac{\partial \phi_i}{\partial x}\qty(\eta_t\hat{L}_t + x e_i;\dis) \dd x.\numberthis\label{eq:phi_i_diff}
    \end{align*}
    Recalling that $\phi_i(\lambda;\dis)$ can be expressed as \eqref{eq:definition_phi_extension_no_underline}, for $\widetilde{\lambda}_1\leq\widetilde{\lambda}_2\leq\dots\leq\widetilde{\lambda}_d$ denoting the sorted elements of $\lambda\in\mathbb{R}^d$ we have
    \begin{equation*}
        \phi_i(\lambda;\dis)=
        \int_{\nu-\widetilde{\lambda}_{m}}^\infty 
        \sP_{\{r_k\}_{k\in[d]\setminus\{i\}} \sim \dis,\,r_i=z+\lambda_i}
        \qty[\sigma_i\qty(r-\lambda)\leq m]
    \dd F(z+\lambda_i).
    \end{equation*}
    Here, $\sP[\cdot]$ in the integrand above does not depend on $\lambda_i$ because the $i$-th component of $r-\lambda$ is fixed to $z$ under $r_i=z+\lambda_i$. 
    Therefore, 
    one can see that $\phi'_i(\lambda;\dis)=\frac{\partial \phi_i}{\partial \lambda_i}(\lambda;\dis)$ satisfies
    \begin{align*}
        -\phi'_i&(\lambda;\dis)
        =
        \int_{\nu-\widetilde{\lambda}_{m}}^\infty -f'(z+\lambda_i)\sP_{\{r_k\}_{k\in[d]\setminus\{i\}} \sim \dis,\,r_i=z+\lambda_i}
        \qty[\sigma_i\qty(r-\lambda)\leq m] \dd z\\
        &\hspace{3cm}-\ind[\sigma_i(-\lambda)=m]f(\nu)\sP_{\{r_k\}_{k\in[d]\setminus\{i\}} \sim \dis,\,r_i=z+\lambda_i}
        \qty[\sigma_i\qty(r-\lambda)\leq m]\\
        &\leq
        \int_{\nu-\widetilde{\lambda}_{m}}^\infty -f'(z+\lambda_i)\sP_{\{r_k\}_{k\in[d]\setminus\{i\}} \sim \dis,\,r_i=z+\lambda_i}
        \qty[\sigma_i\qty(r-\lambda)\leq m] \dd z\\
        &=
        \int_{\nu}^\infty -f'(z+\ul_i)\sP_{\{r_k\}_{k\in[d]\setminus\{i\}} \sim \dis,\,r_i=z+\ul_i}
        \qty[\sigma_i\qty(r-\ul)\leq m] \dd z,\numberthis\label{eq:phi_derivative_expression}
    \end{align*}
    where the last equality holds by $\ul_i=\lambda_i-\widetilde{\lambda}_m$ for any $i\in[d]$.
    Now we divide the proof into two cases.
    When $\dis=\Fdis$, since for $x>0$,
    \begin{equation*}
        f'(x)=\qty(-\frac{\alpha(\alpha+1)}{x^{\alpha+2}}+\frac{\alpha^2}{x^{2(\alpha+1)}})e^{-1/x^\alpha},
    \end{equation*}
    we have
    \begin{equation}\label{eq:f_derivative_relation}
        -f'(x)=\qty(\frac{\alpha(\alpha+1)}{x^{\alpha+2}}-\frac{\alpha^2}{x^{2(\alpha+1)}})e^{-1/x^\alpha}\leq
        \frac{\alpha(\alpha+1)}{x^{\alpha+2}}e^{-1/x^\alpha}
        =\frac{\alpha+1}{x}f(x).
    \end{equation}
    Recall the definition of $J_i(\cdot)$ given in \eqref{eq:definition_J}. 
    By \eqref{eq:phi_i_diff}--\eqref{eq:f_derivative_relation} we have
    \begin{align*}
        \lefteqn{\phi_i\qty(\eta_t\hat{L}_t;\Fdis)-\phi_i\qty(\eta_t\qty(\hat{L}_t+\hat{\ell}_t);\Fdis)}\\
        &\leq
        \int_0^{\eta_t\ell_{t,i} \widehat{w_{t,i}^{-1}}}\int_{\nu}^\infty \frac{\alpha+1}{z+\underline{\eta_t\hat{L}_t+xe_i}_i}\sP_{\{r_k\}_{k\in[d]\setminus\{j\}} \sim \dis,\,r_i=z+\underline{\eta_t\hat{L}_t+xe_i}_i}
        \qty[\sigma_i\qty(r-\underline{\eta_t\hat{L}_t+xe_i})\leq\wm] \\
        &\hspace{11.5cm}\dd F(z+\underline{\eta_t\hat{L}_t+xe_i}_i)\dd x\\
        &=
        (\alpha+1)\int_0^{\eta_t\ell_{t,i} \widehat{w_{t,i}^{-1}}} J_i\qty(\underline{\eta_t\hat{L}_t + x e_i};\Fdis) \dd x \\
        &\leq (\alpha+1)\int_0^{\eta_t\ell_{t,i} \widehat{w_{t,i}^{-1}}} J_i\qty(\eta_t\underline{\hat{L}}_t;\Fdis) \dd x\numberthis\label{eq:J_mono_frechet}\\
        &=(\alpha+1)\eta_t\ell_{t,i}  J_i\qty(\eta_t\underline{\hat{L}}_t;\Fdis)\widehat{w_{t,i}^{-1}},
    \end{align*}
    where \eqref{eq:J_mono_frechet} follows from the monotonicity of $J_i(\eta_t\underline{\hat{L}}_t;\Fdis)$.
    We obtain the same result for the case of Pareto distribution $\dis=\Pdis$ since $-f'(x)=(\alpha+1)f(x)/x$ holds as an equality instead of the inequality in \eqref{eq:f_derivative_relation}.

    By Lemma~\ref{lem:cgr_general_idea}, the unbiased estimator $\widehat{{w_{t,i}^{-1}}}$ for $w_{t,i}^{-1}$ provided by the original GR or CGR satisfies $\mathrm{Var}[\widehat{{w_{t,i}^{-1}}}| \hat{L}_t,a_{t,i}]\leq 1/w_{t,i}^2$. 
    Therefore, we have
    \begin{equation}\label{eq:square_expectation}
        \E\qty[\widehat{{w_{t,i}^{-1}}}^2\m\hat{L}_t,a_{t,i}]=
        \E^2\qty[\widehat{{w_{t,i}^{-1}}}\m\hat{L}_t,a_{t,i}]+
        \mathrm{Var}\qty[\widehat{{w_{t,i}^{-1}}}\m\hat{L}_t,a_{t,i}]\leq \frac{2}{w_{t,i}^2}.
    \end{equation}
    Since $\hat{\ell}_{t,i}=\ind[a_{t,i}=1]\ell_{t,i}\widehat{{w_{t,i}^{-1}}}$, for $\dis\in\qty{\Fdis,\Pdis}$ we obtain
    \begin{align*}
        \lefteqn{\E\qty[\hat{\ell}_{t,i}\qty(\phi_i\qty(\eta_t\hat{L}_t;\dis)-\phi_i\qty(\eta_t\qty(\hat{L}_t+\hat{\ell}_t);\dis))\m \hat{L}_t]}\\
        &\leq\E\qty[\ind\qty[a_{t,i}=1]\ell_{t,i}\widehat{{w_{t,i}^{-1}}}\cdot(\alpha+1)\eta_t\ell_{t,i}  J_i\qty(\eta_t\underline{\hat{L}}_t;\dis)\widehat{w_{t,i}^{-1}}\m \hat{L}_t]\\
        &\leq2(\alpha+1)\eta_t\E\qty[w_{t,i}\frac{\ell^2_{t,i}J_i\qty(\eta_t\underline{\hat{L}}_t;\dis)}{w^2_{t,i}}\m \hat{L}_t]
        \displaybreak[0]\\
        &\leq2(\alpha+1)\eta_t\E\qty[\frac{J_i\qty(\eta_t\underline{\hat{L}}_t;\dis)}{I_i\qty(\eta_t\underline{\hat{L}}_t;\dis)}\m \hat{L}_t]
        \tag{by $w_{t,i}=\phi_i(\eta_t\hat{L}_t;\dis)$ and \eqref{eq:phi_I_relation}}
        \displaybreak[0]\\
        &\leq
        \frac{2(\alpha+1)}{0\lor\underline{\hat{L}}_{t,i}}
        \land
        \acon\eta_t\qty(1\land\frac{m}{\sigma_i})^{\frac{1}{\alpha}},
    \end{align*}
    where the last inequality follows from Lemma~\ref{lem:sigma_i}.
\end{proof}

\begin{lemma}\label{lem:stability}
    For any $\eta_t\hat{L}_t$, $\alpha>1$ and $\dis\in\qty{\Fdis,\Pdis}$, it holds that
    \begin{equation*}
        \E\qty[\hat{\ell}_t \qty(\phi\qty(\eta_t \hat{L}_t;\dis)-\phi\qty(\eta_t \qty(\hat{L}_t+\hat{\ell}_t);\dis))\m \hat{L}_t]\leq\\
        \acon\eta_t\qty(m+\frac{\alpha}{\alpha-1} m^{\frac{1}{\alpha}}(d-m)^{1-\frac{1}{\alpha}}).
    \end{equation*}
\end{lemma}

\begin{proof2}
    By Lemma~\ref{lem:adv_bound_term}, we immediately have
    \begin{align*}
        \E\Big[\hat{\ell}_t \qty(\phi\qty(\eta_t\hat{L}_t;\dis)-\phi\qty(\eta_t\qty(\hat{L}_t+\hat{\ell}_t);\dis))\Big| &\hat{L}_t\Big]
        \leq
        \acon\eta_t\qty(m+m^{1/\alpha}\sum_{i=m+1}^d i^{-\frac{1}{\alpha}})\\
        &\leq
        \acon\eta_t\qty(m+m^{1/\alpha}\int_m^d x^{-\frac{1}{\alpha}} \dd x)\\
        &=\acon\eta_t\qty(m+\frac{\alpha}{\alpha-1}m^{\frac{1}{\alpha}}(d^{1-\frac{1}{\alpha}}-m^{1-\frac{1}{\alpha}}))\\
        &\leq
        \acon\eta_t\qty(m+\frac{\alpha}{\alpha-1}m^{\frac{1}{\alpha}}(d-m)^{1-\frac{1}{\alpha}}).\dqed
    \end{align*}
\end{proof2}

\subsection{Penalty Term}\label{subsec:penalty_term}
The penalty term is bounded in the following lemma.
\begin{lemma}\label{lem:penalty_adv_bound}
    For any $\hat{L}_t\in\mathbb{R}^d$, $\alpha>1$ and $\dis\in\qty{\Fdis,\Pdis}$, we have
    \begin{multline}\label{eq:penalty_adv_bound}
        \E_{r_{t+1}\sim\dis}\qty[\left\langle r_{t+1}, a_{t+1} - a^*\right\rangle\m\hat{L}_t]
        \\\leq
        \begin{cases}
            \qty(\frac{\alpha}{\alpha-1}(m-1)^{1-\frac{1}{\alpha}}+\Gamma\qty(1-\frac{1}{\alpha}))\qty(d+1)^{\frac{1}{\alpha}}+m, & \text{if } \dis = \Fdis,\\
            \qty(\frac{\alpha}{\alpha-1}(m-1)^{1-\frac{1}{\alpha}}+\Gamma\qty(1-\frac{1}{\alpha}))\qty(d+1)^{\frac{1}{\alpha}}, & \text{if } \dis = \Pdis.
        \end{cases}
    \end{multline}
    In addition, if $\underline{\hat{L}}_{t,i}>0$ for any $i$ with $a^*_i=0$, we also have
    \begin{equation*}
        \E_{r_{t+1}\sim\dis}\qty[\left\langle r_{t+1}, a_{t+1} - a^*\right\rangle\m\hat{L}_t]
        \leq
        \frac{\alpha}{\alpha-1}\sum_{i:a^*_i=0}\frac{1}{(\eta_t\underline{\hat{L}}_{t,i})^{\alpha-1}}.
    \end{equation*} 
\end{lemma}

\begin{proof}
    Let $r_k^*$ be the $k$-th largest perturbation among $r_{1,1},r_{1,2},\dots,r_{1,d}$, which are i.i.d. from $\mathcal{D}_{\alpha}$ for $k\in[d]$ and $\alpha>1$.
    Then, we have
    \begin{equation*}
        \E_{r_{t+1}\sim\dis}\qty[\left\langle r_{t+1}, a_{t+1} - a^*\right\rangle\m\hat{L}_t ]
        \leq
        \E_{r_{t+1}\sim\dis}\qty[a_{t+1}^\top r_{t+1}\m\hat{L}_t]
        \leq
        \E_{r\sim\dis}\qty[\sum_{k=1}^m r_k^*\m\hat{L}_t].
    \end{equation*}
    Note that $\E_{r\sim\dis}\qty[\sum_{k=1}^m r_k^*\m\hat{L}_t]$ does not depend on $\hat{L}_t$.
    Then, we have
    \begin{equation*}
        \E_{r_{t+1}\sim\dis}\qty[\left\langle r_{t+1}, a_{t+1} - a^*\right\rangle\m\hat{L}_t ]
        \leq
        \E_{r\sim\dis}\qty[\sum_{k=1}^m r_k^*].
    \end{equation*}
    Then, by applying Lemma~\ref{lem:max_sum_perturbation} given in Appendix~\ref{app:lemmas}, we obtain \eqref{eq:penalty_adv_bound}.
    
    Next we show the latter part of the lemma, which considers the case $\underline{\hat{L}}_{t,i}>0$ for any $i$ with $a^*_i=0$.
    When $\dis\in\qty{\Fdis,\Pdis}$, since $f(x)\leq \alpha/x^{\alpha+1}$ for $x\geq \nu$, we have
    \begin{align*}
        \E_{r_{t+1}\sim\dis}\qty[\left\langle r_{t+1}, a_{t+1} - a^*\right\rangle\m\hat{L}_t]
        &\leq
        \sum_{i:a^*_i=0}\E_{r_{t+1}\sim\dis}\qty[a_{t+1,i} r_{t+1,i}\m\hat{L}_t]\\
        &=
        \sum_{i:a^*_i=0}\E_{r_{t+1}\sim\dis}\qty[\ind\qty[a_{t+1,i}=1] r_{t+1,i}\m\hat{L}_t]\\
        &\leq
        \sum_{i:a^*_i=0}\E_{r_{t+1}\sim\dis}\qty[\ind\qty[r_{t+1,i}\geq\nu+\eta_t\underline{\hat{L}}_{t,i}] r_{t+1,i}\m\hat{L}_t]\\
        &=
        \sum_{i:a^*_i=0} \int_{\nu}^\infty (z+\eta_t\underline{\hat{L}}_{t,i})f\qty(z+\eta_t\underline{\hat{L}}_{t,i})\dd z\\
        &\leq
        \sum_{i:a^*_i=0} \int_\nu^\infty \frac{\alpha}{(z+\eta_t\underline{\hat{L}}_{t,i})^\alpha}\dd z 
        \tag{by $f(x)\leq \alpha/x^{\alpha+1}$ for $x\geq \nu$}\\
        &\leq
        \frac{\alpha}{\alpha-1}\sum_{i:a^*_i=0}\frac{1}{(\eta_t\underline{\hat{L}}_{t,i})^{\alpha-1}},
    \end{align*}
    which concludes the proof.
\end{proof}

\subsection{Proof of Theorem~\ref{thm:adv_bound}}\label{sec:proof_adv_bound}

By combining Lemmas~\ref{lem:regret_decomposition}, \ref{lem:stability} and \ref{lem:penalty_adv_bound} 
with $\eta_t=\frac{c}{\sqrt{t}}m^{\frac{1}{2}-\frac{1}{\alpha}}d^{\frac{1}{\alpha}-\frac{1}{2}}$, for $\dis=\Pdis$ we have
\begin{align*}
    \R(T)&\leq
    C_1(\Pdis) cm^{\frac{1}{2}-\frac{1}{\alpha}}d^{\frac{1}{\alpha}-\frac{1}{2}}\qty(m+\frac{\alpha}{\alpha-1} m^{\frac{1}{\alpha}}(d-m)^{1-\frac{1}{\alpha}})\sum_{t=1}^T \frac{1}{\sqrt{t}}\\
    &\hspace{2cm}+
    \frac{\qty(\frac{\alpha}{\alpha-1}(m-1)^{1-\frac{1}{\alpha}}+\Gamma\qty(1-\frac{1}{\alpha}))\qty(d+1)^{\frac{1}{\alpha}}}{cm^{\frac{1}{2}-\frac{1}{\alpha}}d^{\frac{1}{\alpha}-\frac{1}{2}}}\sum_{t=1}^T\qty(\sqrt{t+1}-\sqrt{t})
    \\
    &\hspace{2cm}+
    \frac{\qty(\frac{\alpha}{\alpha-1}(m-1)^{1-\frac{1}{\alpha}}+\Gamma\qty(1-\frac{1}{\alpha}))\qty(d+1)^{\frac{1}{\alpha}}}{cm^{\frac{1}{2}-\frac{1}{\alpha}}d^{\frac{1}{\alpha}-\frac{1}{2}}}\numberthis\label{eq:adv_bound_to_replace}
    \displaybreak[0]\\
    &\leq 
        2 C_1(\Pdis)c\qty(\qty(\frac{m}{d})^{1-\frac{1}{\alpha}}+\frac{\alpha}{\alpha-1}\qty(1-\frac{m}{d})^{1-\frac{1}{\alpha}})\sqrt{mdT}
        \\
    &\hspace{2cm}+\frac{\qty(\frac{\alpha}{\alpha-1}(1-\frac{1}{m})^{1-\frac{1}{\alpha}}+\Gamma\qty(1-\frac{1}{\alpha})m^{\frac{1}{\alpha}-1})(1+\frac{1}{d})^{\frac{1}{\alpha}}}{c}
    \qty(\sqrt{mdT}+\sqrt{md}),
\end{align*}
which implies that the regret is of order $O(\sqrt{mdT})$ with $C_1(\Pdis)=9(\alpha+1)$.
The analysis for $\dis=\Fdis$ is similar with the difference that $\frac{\qty(\frac{\alpha}{\alpha-1}(m-1)^{1-\frac{1}{\alpha}}+\Gamma\qty(1-\frac{1}{\alpha}))\qty(d+1)^{\frac{1}{\alpha}}}{cm^{\frac{1}{2}-\frac{1}{\alpha}}d^{\frac{1}{\alpha}-\frac{1}{2}}}$ in \eqref{eq:adv_bound_to_replace} is replaced with $\frac{\qty(\frac{\alpha}{\alpha-1}(m-1)^{1-\frac{1}{\alpha}}+\Gamma\qty(1-\frac{1}{\alpha}))\qty(d+1)^{\frac{1}{\alpha}}+m}{cm^{\frac{1}{2}-\frac{1}{\alpha}}d^{\frac{1}{\alpha}-\frac{1}{2}}}$ coming from Lemmas~\ref{lem:regret_decomposition} and \ref{lem:penalty_adv_bound}. 
This leads to the bound of
\begin{multline*}
    \R(T)\leq 
        2 C_1(\Fdis)c\qty(\qty(\frac{m}{d})^{1-\frac{1}{\alpha}}+\frac{\alpha}{\alpha-1}\qty(1-\frac{m}{d})^{1-\frac{1}{\alpha}})\sqrt{mdT}
        \\
    +\frac{\qty(\frac{\alpha}{\alpha-1}(1-\frac{1}{m})^{1-\frac{1}{\alpha}}+\Gamma\qty(1-\frac{1}{\alpha})m^{\frac{1}{\alpha}-1})(1+\frac{1}{d})^{\frac{1}{\alpha}}+\qty(\frac{m}{d})^{\frac{1}{\alpha}}}{c}
    \qty(\sqrt{mdT}+\sqrt{md}),
\end{multline*}
which also implies that the regret is of order $O(\sqrt{mdT})$. 
\hfill\BlackBox
\smallskip

\subsection{Outline for Stochastic Regret Bound}\label{eq:proof_sto_bound_outline}

In this part, we explain how to derive the regret bound for stochastic setting, noting in particular that the bound becomes logarithmic when $\alpha=2$.
Although the overall proof consists of two cases respectively corresponding to perturbation following $\Fdis$ or $\Pdis$ with shape $\alpha\geq 2$, and those with shape $\alpha\in(1,2)$, each case is established in essentially the same way.
Therefore,
we only provide the outline for the case of $\Fdis$ or $\Pdis$ with $\alpha\geq 2$ here. 

Firstly, the regret in stochastic setting can be expressed as
\begin{equation*}
    \R(T)=\E\qty[\sum_{t=1}^T \sum_{i:a^*_i=0}w_{t,i}\Delta_i].
\end{equation*}
As an extension of the technique applied in \cite{pmlr-v201-honda23a} and \citet{pmlr-v247-lee24a}, we will define an event $F_t$ (resp. $D_t$) for the perturbation distribution $\Fdis$ (resp. $\Pdis$), under which $\hat{L}_{t,i}$ of non-optimal base-arm $i$ (i.e., $i$ with $a^*_i=0$) is sufficiently large compared with $\hat{L}_{t,j}$ for any base-arm $j$ of the optimal action (i.e., $j$ with $a^*_j=1$) so that $\eta_t\underline{\hat{L}}_{t,i}\geq m^{1/(\alpha-1)}$ for any $i$ with $a^*_i=0$. 
In the remainder of this section, we implicitly consider $\Fdis$ using $F_t$, but the overall outline itself is the same under $\Pdis$ with $D_t$.

As discussed in Section~\ref{subsec:regret_optimal_action}, the most challenging part of the analysis is to address the stability term associated with the optimal action, which is expressed as 
\begin{equation}
    S\qty(\hat{\ell}_t;\eta_t,\hat{L}_t)\coloneqq
    \E\qty[\sum_{i:a^*_i=1}\hat{\ell}_{t,i}\qty(\phi_i(\eta_t\hat{L}_t)-\phi_i(\eta_t(\hat{L}_t+\hat{\ell}_t)))\m\hat{L}_t].\label{expression_S}
\end{equation}
Since $\underline{\hat{L}}_{t,i}\leq 0$ for all $i$ with $a^*_i=1$ on $F_t$, the bound $O(1/(0\lor\underline{\hat{L}}_{t,i}))$ in Lemma~\ref{lem:adv_bound_term} becomes meaningless when applying the self-bounding technique. 
Therefore, we need to express $S(\hat{\ell}_t;\eta_t,\hat{L}_t)$ in terms of statstics of the remaining base-arms, i.e., those with $a^*_i=0$. 
In the MAB (case of $m=1$), we can easily achieve this goal by applying a trivial relation between arm-selection probabilities expressed as
\begin{equation}\label{eq:optimal_action_key_equation}
    \sum_{i:a^*_i=1}\qty(\phi_i(\eta_t\hat{L}_t)-\phi_i(\eta_t(\hat{L}_t+\hat{\ell}_t)))=\sum_{i:a^*_i=0}\qty(\phi_i(\eta_t(\hat{L}_t+\hat{\ell}_t))-\phi_i(\eta_t\hat{L}_t))
\end{equation} 
because there is only one $i$ such that $a^*_i=1$.
However, in $m$-set semi-bandits,
it is difficult to bound \eqref{expression_S} by a form such that
\eqref{eq:optimal_action_key_equation} is applicable,
because each component of the summation in \eqref{expression_S}
is not always increasing or decreasing $\hat{\ell}_{t,i}$
due to the complicated behavior of $\phi_i(\eta_t(\hat{L}_t+\hat{\ell}_t))$ in the combinatorial bandits,
while it is always nondecreasing when $m=1$.

To address this difficulty, we conduct a detailed analysis on the gap of $\phi_i(\cdot)$ between different inputs in
Lemmas~\ref{lem:optimal_action_add} and \ref{lem:optimal_action_combine}, which reveals that
$S(\hat{\ell}_t;\eta_t,\hat{L}_t)$ can be bounded as
\begin{equation}\label{eq:outline_bound_first_term}
    S\qty(\hat{\ell}_t;\eta_t,\hat{L}_t)\leq\E\qty[
\sum_{i:a_i^*=1}\delta_i\qty(\phi_i(\eta_t\hat{L}_t)-\phi_i(\eta_t(\hat{L}_t+\delta)))\m\hat{L}_t]+
    O\qty(\sum_{i:a^*_i=0}\frac{1}{\underline{\hat{L}}_{t,i}}),
\end{equation}
where $\delta=\max_{i:a_i^*=1}\hat{\ell}_{t,i}\sum_{i:a_i^*=1}e_i$. 
Here, the second term is specific to our semi-bandit setting
and deriving this bound is the main contribution in the analysis for the stochastic setting.
Thanks to this bound,
we can apply \eqref{eq:optimal_action_key_equation}
to reduce the remaining analysis to a one similar to
\citet{pmlr-v247-lee24a}.

Now we give an outline for the remaining steps.
By \eqref{eq:outline_bound_first_term}
we can show that on $F_t$
\begin{equation*}
    S\qty(\hat{\ell}_t;\eta_t,\hat{L}_t)=\E\qty[\sum_{i:a^*_i=1}\hat{\ell}_{t,i}\qty(\phi_i(\eta_t\hat{L}_t)-\phi_i(\eta_t(\hat{L}_t+\hat{\ell}_t)))\m\hat{L}_t]\leq
    O\qty(\sum_{i:a^*_i=0}\frac{1}{\underline{\hat{L}}_{t,i}}).
\end{equation*}
as shown in Lemma~\ref{lem:optimal_action_bound}. 
By combining this with Lemmas~\ref{lem:adv_bound_term}--\ref{lem:penalty_adv_bound}, we obtain
\begin{equation*}
    \R(T)\leq
    \E\qty[\sum_{t=1}^T O\qty(\ind\qty[F_t]\sum_{i:a^*_i=0}\frac{1}{\underline{\hat{L}}_{t,i}} + \ind\qty[F_t^c]\sqrt{\frac{md}{t}})].
\end{equation*}

Next, by analyzing the regret lower bound on $F_t$ and $F_t^c$ separately, we can obtain
\begin{equation*}
    \R(T)\geq
    \E\qty[\sum_{t=1}^T \Omega\qty(\ind\qty[F_t]\sum_{i:a^*_i=0}\frac{\Delta_i t^\frac{\alpha}{2}}{m^{\frac{\alpha}{2}-1}d^{1-\frac{\alpha}{2}}\underline{\hat{L}}_{t,i}^\alpha} + \ind\qty[F_t^c]\frac{\Delta}{m^{\frac{\alpha}{\alpha-1}}})]
\end{equation*} 
as detailed in Section~\ref{subsec:regret_lower_bounds}, 
where $\Delta=\min_{i:a^*_i=0}\Delta_i$ and $\Omega$ denotes the big-Omega notation. 

By combining these results, we obtain the bound in a form ready to apply the self-bounding technique given by
\begin{equation*}
    \frac{\R(T)}{2}\leq
    \E\qty[\sum_{t=1}^T O\qty(\ind\qty[F_t]\qty(\sum_{i:a^*_i=0}\frac{1}{\underline{\hat{L}}_{t,i}}-\frac{\Delta_i t^\frac{\alpha}{2}}{2m^{\frac{\alpha}{2}-1}d^{1-\frac{\alpha}{2}}\underline{\hat{L}}_{t,i}^\alpha}) + \ind\qty[F_t^c]\qty(\sqrt{\frac{md}{t}}-\frac{\Delta}{2m^{\frac{\alpha}{\alpha-1}}}))].
\end{equation*}
Since $Ax-Bx^\alpha\leq A\frac{\alpha}{\alpha-1}(\frac{A}{\alpha B})^{\frac{1}{\alpha-1}}$ for any $A,B>0$ and $\alpha>1$, we can further bound the above by
\begin{equation*}
    \R(T)\leq
    \sum_{t=1}^TO\qty(\sum_{i:a^*_i=0}\frac{1}{
        \Delta_i^{\frac{1}{\alpha-1}}
        m^{\frac{2-\alpha}{2(\alpha-1)}}
        d^{\frac{\alpha-2}{2(\alpha-1)}}
        t^{\frac{\alpha}{2(\alpha-1)}}
        })
    +O\qty(\frac{m^{\frac{2\alpha-1}{\alpha-1}}d}{\Delta}),
\end{equation*}
which yields a logarithmic regret bound when $\alpha=2$. By formalizing the outline above, we complete the proofs of Theorems~\ref{thm:sto_bound} and \ref{thm:sto_bound_alpha_not2} in Section~\ref{subsec:proofs_sto_bounds}.

\section{Analysis on Stability Term}\label{sec:stab_adv}

This section provides the technical details underlying the proof of upper bounds on the stability term. 
We first provide some properties of function $J_i(\lambda;\dis,\wm,\B)/I_i(\lambda;\dis,\wm,\B)$ in Lemmas~\ref{lem:important_inequality} and \ref{lem:infty}, which are instrumental in deriving the upper bound in Lemma~\ref{lem:lambda_star}. 
Then, we present detailed proofs of Lemmas~\ref{lem:lambda_star}--\ref{lem:stability}.

\subsection{Proof of Supporting Properties}\label{subsec:proof_monotonicity}

We first present Lemmas~\ref{lem:both_increasing} and \ref{lem:bgaf}, which will be used in the proof of Lemma~\ref{lem:important_inequality}. 
\begin{lemma}\label{lem:both_increasing}
    Let $\psi (x):[\nu,\infty)\to\mathbb{R}$ denote a non-negative function, independent of $\lambda_j$, and satisfying $\psi(x)=0$ for all $x\in[\nu, \nu-\lambda_i]$ whenever $\lambda_i<0$. 
    If $j\neq i$ and $F(x)$ is the cumulative distribution function of Fr\'{e}chet or Pareto distributions, then
    \begin{equation*}
        \frac{\int_{\nu}^\infty \psi(z)F(z+\lambda_j)/(z+\lambda_i) \dd z}{\int_{\nu}^\infty \psi(z)F(z+\lambda_j) \dd z}
    \end{equation*}
    is monotonically increasing in $\lambda_j$.
\end{lemma}

Lemma~\ref{lem:both_increasing} is a variant of Lemma~9 in \citet{pmlr-v247-lee24a}, and the proof follows similar arguments. We provide the detailed proof in Appendix~\ref{app:proof_both_increasing}.

\begin{lemma}\label{lem:bgaf}
    Let $a,b>0$, $f(x),g(x)>0$, where $x\in\mathbb{R}$. If both $f(x)$ and $g(x)/f(x)$ are monotonically increasing in $x$, then for any $x_1<x_2$, we have 
    \begin{equation*}
        \frac{b+g(x_1)}{a+f(x_1)}\leq\frac{b}{a}\lor\frac{b+g(x_2)}{a+f(x_2)}.
    \end{equation*}
    Provided that $\lim_{x\to\infty}\qty(b+g(x))/\qty(a+f(x))$ exists, for any $x_0\in\mathbb{R}$ we have
    \begin{equation*}
        \frac{b+g(x_0)}{a+f(x_0)}\leq\frac{b}{a}\lor\lim_{x\to\infty}\frac{b+g(x)}{a+f(x)}.
    \end{equation*}
\end{lemma}

\begin{proof}
    According to the assumption, we have 
    \begin{equation*}
        f(x_1)\leq f(x_2) \text{ and } \frac{g(x_1)}{f(x_1)}\leq \frac{g(x_2)}{f(x_2)}.
    \end{equation*}
    If $b/a>g(x_2)/f(x_2)$, then we have
        \begin{equation*}
            \frac{b+g(x_1)}{a+f(x_1)}\leq\frac{b}{a}\lor\frac{g(x_1)}{f(x_1)}\leq\frac{b}{a}\lor\frac{g(x_2)}{f(x_2)}\leq\frac{b}{a}\leq\frac{b}{a}\lor\frac{b+g(x_2)}{a+f(x_2)}.
        \end{equation*}
    On the other hand, if $b/a\leq g(x_2)/f(x_2)$, then we have
        \begin{equation*}
            \frac{b+g(x_1)}{a+f(x_1)}=\frac{b+f(x_1)\frac{g(x_1)}{f(x_1)}}{a+f(x_1)}\leq\frac{b+f(x_1)\frac{g(x_2)}{f(x_2)}}{a+f(x_1)}.
        \end{equation*}
        Let $h(z)=\qty(b+\frac{g(x_2)}{f(x_2)}z)/\qty(a+z)=\frac{b}{a}+\qty(\frac{g(x_2)}{f(x_2)}-\frac{b}{a})z/(a+z)$, which is monotonically increasing in $z\in[f(x_1),f(x_2)]$ when $b/a\leq g(x_2)/f(x_2)$. 
        Therefore, we have
        \begin{equation*}
            \frac{b+f(x_1)\frac{g(x_2)}{f(x_2)}}{a+f(x_1)}\leq\frac{b+f(x_2)\frac{g(x_2)}{f(x_2)}}{a+f(x_2)}=\frac{b+g(x_2)}{a+f(x_2)}\leq\frac{b}{a}\lor\frac{b+g(x_2)}{a+f(x_2)}.
        \end{equation*}
    Combining both cases, we have
    \begin{equation*}
        \frac{b+g(x_1)}{a+f(x_1)}\leq\frac{b}{a}\lor\frac{b+g(x_2)}{a+f(x_2)}
    \end{equation*}
    for any $x_1<x_2$. If $\lim_{x\to\infty}\qty(b+g(x))/\qty(a+f(x))$ exists, the result for the infinite case follows directly by taking the limit of $x_2\to\infty$.
\end{proof}

Now we present the following lemma, which shows an important property of the function $J_i(\lambda;\dis,\wm,\B)/I_i(\lambda;\dis,\wm,\B)$.

\begin{lemma}\label{lem:important_inequality}
For $\wm\ge 1$
and $j\in\B\setminus\qty{i}$, let $\lambda'\in\mathbb{R}^d$ be such that $\lambda'_j\geq \lambda_j$ and $\lambda_k'=\lambda_k$ for all $k\neq j$.
    Then, we have
    \begin{equation*}
        \frac{J_i(\lambda;\dis,\wm,\B)}{I_i(\lambda;\dis,\wm,\B)}\leq\frac{J_i(\lambda';\dis,\wm-1,\B\setminus\qty{j})}{I_i(\lambda';\dis,\wm-1,\B\setminus\qty{j})}\lor\frac{J_i(\lambda';\dis,\wm,\B)}{I_i(\lambda';\dis,\wm,\B)}.
    \end{equation*}
\end{lemma}

\begin{proof}
    By Lemma~\ref{lem:tool_relation}, we have
    \begin{equation}\label{eq:J_phi_decomose}
        \frac{J_i(\lambda;\dis,\wm,\B)}{I_i(\lambda;\dis,\wm,\B)}
        =
        \frac{J_i(\lambda;\dis,\wm-1,\B\setminus\qty{j})+J_{i,\wm,j}(\lambda;\dis,\B)}{I_i(\lambda;\dis,\wm-1,\B\setminus\qty{j})+I_{i,\wm,j}(\lambda;\dis,\B)}.
    \end{equation}

    Note that
    $J_{i,\wm,j}(\lambda;\dis,\B)/I_{i,\wm,j}(\lambda;\dis,\B)$ is monotonically increasing in $\lambda_j$ because we can apply Lemma~\ref{lem:both_increasing} with 
    \begin{equation*}
        \psi(z)\coloneqq f(z+\lambda_i)\sP_{\{r_k\}_{k\in[d]\setminus\{i,j\}} \sim \mathcal{D},\,r_i=z+\lambda_i}
        \qty[\sigma_i\qty(r-\lambda,\B\setminus\{j\})=\theta]
    \end{equation*}
    by the representations given in \eqref{eq:I_theta_j} and \eqref{eq:prob_k_J}.
    Then, by noting that $J_i(\ul;\dis,\wm-1,\B\setminus\{j\})$ and $I_i(\ul;\dis,\wm-1,\B\setminus\{j\})$ do not depend on $\ul_j$ and applying Lemma~\ref{lem:bgaf} to \eqref{eq:J_phi_decomose} we obtain
    \begin{align*}
        \frac{J_i(\lambda;\dis,\wm,\B)}{I_i(\lambda;\dis,\wm,\B)}
        &\leq \frac{J_i(\lambda;\dis,\wm-1,\B\setminus\qty{j})}{I_i(\lambda;\dis,\wm-1,\B\setminus\qty{j})}
        \lor
        \frac{J_i(\lambda';\dis,\wm-1,\B\setminus\qty{j})+J_{i,\wm,j}(\lambda';\dis,\B)}{I_i(\lambda';\dis,\wm-1,\B\setminus\qty{j})+I_{i,\wm,j}(\lambda';\dis,\B)}\\
        &=
        \frac{J_i(\lambda;\dis,\wm-1,\B\setminus\qty{j})}{I_i(\lambda;\dis,\wm-1,\B\setminus\qty{j})}
        \lor
        \frac{J_i(\lambda';\dis,\wm,\B)}{I_i(\lambda';\dis,\wm,\B)}\numberthis\label{eq:monotonic_intermediate}\\
        &=
        \frac{J_i(\lambda';\dis,\wm-1,\B\setminus\qty{j})}{I_i(\lambda';\dis,\wm-1,\B\setminus\qty{j})}
        \lor
        \frac{J_i(\lambda';\dis,\wm,\B)}{I_i(\lambda';\dis,\wm,\B)},
    \end{align*}
    where we again used the fact that $J_i(\lambda;\dis,\wm-1,\B\setminus\qty{j})$ and $I_i(\lambda;\dis,\wm-1,\B\setminus\qty{j})$
do not depend on $\lambda_j$ in the last equality.
\end{proof}

The following lemma presents a special case of Lemma~\ref{lem:important_inequality}, where we take $\lambda'_j\to\infty$. 
\begin{lemma}\label{lem:infty}
Let $\wm\geq 1$. If $\B\supset  \{i,j\}$ and $\wm\le \left\lvert\B\right\rvert$, we have
    \begin{equation*}
        \frac{J_i(\lambda;\dis,\wm,\B)}{I_i(\lambda;\dis,\wm,\B)}\leq\frac{J_i(\lambda;\dis,\wm-1,\B\setminus\qty{j})}{I_i(\lambda;\dis,\wm-1,\B\setminus\qty{j})}\lor\frac{J_i(\lambda;\dis,\wm,\B\setminus\qty{j})}{I_i(\lambda;\dis,\wm,\B\setminus\qty{j})}.
    \end{equation*}
\end{lemma}

\begin{proof}
    By \eqref{eq:monotonic_intermediate} we have
    \begin{align}
        \frac{J_i(\lambda;\dis,\wm,\B)}{I_i(\lambda;\dis,\wm,\B)}
        &\leq \frac{J_i(\lambda;\dis,\wm-1,\B\setminus\qty{j})}{I_i(\lambda;\dis,\wm-1,\B\setminus\qty{j})}\lor\frac{J_i(\lambda';\dis,\wm,\B)}{I_i(\lambda';\dis,\wm,\B)}
        \n
    \end{align}
    for any
    $\lambda'\in\mathbb{R}^d$ such that $\lambda'_j\geq \lambda_j$ and $\lambda_k'=\lambda_k$ for all $k\neq j$.
    Then we have
    \begin{align}
        \frac{J_i(\lambda;\dis,\wm,\B)}{I_i(\lambda;\dis,\wm,\B)}
        &\leq \frac{J_i(\lambda;\dis,\wm-1,\B\setminus\qty{j})}{I_i(\lambda;\dis,\wm-1,\B\setminus\qty{j})}\lor
    \limsup_{\lambda_j'\to \infty}
    \frac{J_i(\lambda';\dis,\wm,\B)}{I_i(\lambda';\dis,\wm,\B)}
    \label{eq:ratio_limit}
    \end{align}
    by taking the limit.

    Since we have
    \begin{align*}
        &\lefteqn{\lim_{\lambda_j\to \infty}
        \sP_{r=\qty(r_1,\dots,r_d) \sim \mathcal{D}}\qty[\sigma_i\qty(r-\lambda,\B)= \wm<\sigma_j\qty(r-\lambda,\B)]}\\
        &\le
        \lim_{\lambda_j\to \infty}
        \sP_{r=\qty(r_1,\dots,r_d) \sim \mathcal{D}}\qty[\wm<\sigma_j\qty(r-\lambda,\B)]
        =0
    \end{align*}
    by the continuity of a probability measure,
    we obtain from \eqref{eq:phi_decompose_prob} that
    \begin{align*}
    \lefteqn{\lim_{\lambda_j\to \infty}
    I_i(\lambda;\dis,\wm,\B)}\\
    &=
    \sP_{r=\qty(r_1,\dots,r_d) \sim \mathcal{D}}\qty[\sigma_i\qty(r-\lambda,\B\setminus\{j\})\le \wm-1,\,r_i\ge \nu+\lambda_i]\\
    &\hspace{1.8cm}+
    \lim_{\lambda_j\to \infty}
    \sP_{r=\qty(r_1,\dots,r_d) \sim \mathcal{D}}\qty[\sigma_i\qty(r-\lambda,\B)= \wm<\sigma_j\qty(r-\lambda,\B),\,r_i\ge \nu+\lambda_i]\\
    &=
    \sP_{r=\qty(r_1,\dots,r_d) \sim \mathcal{D}}\qty[\sigma_i\qty(r-\lambda,\B\setminus\{j\})\le \wm-1,\,r_i\ge \nu+\lambda_i]\\
    &=
    I_i(\lambda;\dis,\wm,\B\setminus\{j\})\numberthis\label{eq:limit_phi}
    \end{align*}
    and similarly we have
    \begin{align}
    \lim_{\lambda_j\to \infty}
    J_i(\lambda;\dis,\wm,\B)
    &=
    J_i(\lambda;\dis,\wm,\B\setminus\{j\}).\label{eq:limit_J}
    \end{align}
    We obtain the lemma by combining \eqref{eq:ratio_limit} with \eqref{eq:limit_phi} and \eqref{eq:limit_J}.
\end{proof}

In the proofs of Lemmas~\ref{lem:lambda_star} and \ref{lem:sigma_i}, without loss of generality, we assume $\lambda_1\leq\lambda_2\leq\cdots\leq\lambda_d$ (ties are broken arbitrarily) so that $\sigma_i=i$ for notational simplicity.

\subsection{Proof of Lemma~\ref{lem:lambda_star}}\label{subsec:proof_lambda_star}

\textbf{Lemma~\ref{lem:lambda_star} (Restated)}
\textit{For $\lambda\in\mathbb{R}^d$, it holds that
    \begin{equation*}
        \frac{J_i(\ul;\dis)}{I_i(\ul;\dis)}\leq
        \max_{(j,\wm)\in[m\land i]^2: j+\wm\le (m\land i)+1}
        \qty{\frac{J_i(\lambda^*;\dis,\wm,[j:i])}{I_i(\lambda^*;\dis,\wm,[j:i])}},
    \end{equation*}
    where
    \begin{equation*}
        \lambda_k^* = 
        \begin{cases}
            \ul_i, & \text{if } k\leq i,\\
            \ul_k, & \text{if } k> i.
        \end{cases}
    \end{equation*}
}    

\begin{proof}
    In this proof, we locally use $\lambda\pn{0},\lambda\pn{1},\lambda\pn{2},\dots,\lambda\pn{i-1}$ to denote a sequence of $d$-dimensional vectors defined as follows.  
    Fix $i$. 
    Define $\lambda\pn{0}=\ul$, for $j\leq i-1$ we define
    \begin{equation*}
        \lambda\pn{j}_k = 
        \begin{cases}
            \ul_i, & \text{if } k\in[j]\cup\qty{i},\\
            \ul_k, & \text{otherwise.}
        \end{cases}
    \end{equation*}
    Consequently, we have $\lambda\pn{i-1}=\lambda^*.$

First we show
for $k\in \{0,1,\dots,i-1\}$ that
\begin{align}
\frac{J_i(\ul;\dis)}{I_i(\ul;\dis)}
&
\le \max_{\wm \in[m-k:m]}\left\{
\frac{J_i(\lambda\pn{k};\dis,\wm,[1+m-\wm:d])}{I_i(\lambda\pn{k};\dis,\wm,[1+m-\wm:d])}
\right\}
\label{induction0}
\end{align}
by induction. 
The statement \eqref{induction0} becomes trivial when $k=0$.
Assume that the statement holds for
$k= k_0<i-1$. 
Then we obtain from Lemma~\ref{lem:important_inequality} that
\begin{align}
\frac{J_i(\ul;\dis)}{I_i(\ul;\dis)}
&\le
 \max_{\wm \in[m-k_0:m]}\left\{
\frac{J_i(\lambda\pn{k_0};\dis,\wm,[1+m-\wm:d])}{I_i(\lambda\pn{k_0};\dis,\wm,[1+m-\wm:d])}
\right\}
\nn
&\le
 \max_{\wm \in[m-k_0:m]}\left\{
\frac{J_i(\lambda\pn{k_0+1};\dis,\wm-1,[1+m-\wm:d]\setminus\{k_0+1\})}{I_i(\lambda\pn{k_0+1};\dis,\wm-1,[1+m-\wm:d]\setminus\{k_0+1\})}
\right\}
\nn
&\qquad\lor
 \max_{\wm \in[m-k_0:m]}\left\{
\frac{J_i(\lambda\pn{k_0+1};\dis,\wm,[1+m-\wm:d])}{I_i(\lambda\pn{k_0+1};\dis,\wm,[1+m-\wm:d])}
\right\}.
\label{reduction1}
\end{align}
Note that we have
\begin{align}
\lefteqn{
 \max_{\wm \in[m-k_0:m]}\left\{
\frac{J_i(\lambda\pn{k_0+1};\dis,\wm-1,[1+m-\wm:d]\setminus\{k_0+1\})}{I_i(\lambda\pn{k_0+1};\dis,\wm-1,[1+m-\wm:d]\setminus\{k_0+1\})}
\right\}
}
\nn
&=
 \max_{\wm \in[m-k_0:m]}\left\{
\frac{J_i(\lambda\pn{k_0+1};\dis,\wm-1,[2+m-\wm:d])}{I_i(\lambda\pn{k_0+1};\dis,\wm-1,[2+m-\wm:d])}
\right\}
\label{la_swap}
\\
&=
 \max_{\wm \in[m-k_0-1:m-1]}\left\{
\frac{J_i(\lambda\pn{k_0+1};\dis,\wm,[1+m-\wm:d])}{I_i(\lambda\pn{k_0+1};\dis,\wm,[1+m-\wm:d])}
\right\},
\label{reduction2}
\end{align}
where the first equality holds since
$\lambda\pn{k_0+1}_{k_0+1}=\lambda\pn{k_0+1}_{1+m-\wm}=\ul_i$ holds in \eqref{la_swap} by $1+m-\wm\le k_0+1$.
By \eqref{reduction1} and \eqref{reduction2}
we have
\begin{align}
\frac{J_i(\ul;\dis)}{I_i(\ul;\dis)}
&\le
 \max_{\wm \in[m-k_0-1:m-1]}\left\{
\frac{J_i(\lambda\pn{k_0+1};\dis,\wm,[1+m-\wm:d])}{I_i(\lambda\pn{k_0+1};\dis,\wm,[1+m-\wm:d])}
\right\}
\nn
&\qquad\lor
\max_{\wm \in[m-k_0:m]}\left\{
\frac{J_i(\lambda\pn{k_0+1};\dis,\wm,[1+m-\wm:d])}{I_i(\lambda\pn{k_0+1};\dis,\wm,[1+m-\wm:d])}
\right\}
\nn
&=
\max_{\wm \in[m-k_0-1:m]}\left\{
\frac{J_i(\lambda\pn{k_0+1};\dis,\wm,[1+m-\wm:d])}{I_i(\lambda\pn{k_0+1};\dis,\wm,[1+m-\wm:d])}
\right\},
\n
\end{align}
which completes the induction step.

Now, by letting $k=i-1$ in \eqref{induction0} we have
\begin{align}
\frac{J_i(\ul;\dis)}{I_i(\ul;\dis)}
&
\le \max_{\wm \in[m-i+1:m]}\left\{
\frac{J_i(\lambda^*;\dis,\wm,[1+m-\wm:d])}{I_i(\lambda^*;\dis,\wm,[1+m-\wm:d])}
\right\}.
\label{induction3}
\end{align}
Here, by repeatedly applying Lemma~\ref{lem:infty} with $j=d, d-1,\dots,i+1$
we have
\begin{align}
\lefteqn{
\frac{J_i(\lambda^*;\dis,\wm,[1+m-\wm:d])}{I_i(\lambda^*;\dis,\wm,[1+m-\wm:d])}
}\nn
&\le
\frac{J_i(\lambda^*;\dis,\wm-1,[1+m-\wm:d-1])}{I_i(\lambda^*;\dis,\wm-1,[1+m-\wm:d-1])}
\lor
\frac{J_i(\lambda^*;\dis,\wm,[1+m-\wm:d-1])}{I_i(\lambda^*;\dis,\wm,[1+m-\wm:d-1])}
\nn
&\le
\max_{k\in\{0,1,2\}}
\frac{J_i(\lambda^*;\dis,\wm-k,[1+m-\wm:d-2])}{I_i(\lambda^*;\dis,\wm-k,[1+m-\wm:d-2])}
\nn
&\le \cdots
\nn
&\le
\max_{k\in\{0,1,\dots,d-i\}}
\frac{J_i(\lambda^*;\dis,\wm-k,[1+m-\wm:i])}{I_i(\lambda^*;\dis,\wm-k,[1+m-\wm:i])}.
\label{induction4}
\end{align}
By combining \eqref{induction3} and \eqref{induction4} we have
\begin{align}
\frac{J_i(\ul;\dis)}{I_i(\ul;\dis)}
&
\le \max_{\wm \in[m-i+1:m]}
\max_{k\in\{0,1,\dots,d-i\}}
\frac{J_i(\lambda^*;\dis,\wm-k,[1+m-\wm:i])}{I_i(\lambda^*;\dis,\wm-k,[1+m-\wm:i])}
\nn
&
\le
\max_{\wm\in [m], j\in[i],\,j+\wm \le m+1}
\frac{J_i(\lambda^*;\dis,\wm,[j:i])}{I_i(\lambda^*;\dis,\wm,[j:i])}
\nn
&
=
\max_{\wm\in [m], j\in[i],\,j+\wm\le (m\land i)+1}
\frac{J_i(\lambda^*;\dis,\wm,[j:i])}{I_i(\lambda^*;\dis,\wm,[j:i])}
\label{induction5}
\\
&
=
\max_{(\wm,j)\in[m\land i]^2,\,j+\wm\le (m\land i)+1}
\frac{J_i(\lambda^*;\dis,\wm,[j:i])}{I_i(\lambda^*;\dis,\wm,[j:i])},
\n
\end{align}
where \eqref{induction5} follows since
$\frac{J_i(\lambda^*;\dis,\wm,[j:i])}{I_i(\lambda^*;\dis,\wm,[j:i])}=-\infty$
if $w\ge |[j:i]|$.
\end{proof}

\subsection{Proof of Lemma~\ref{lem:sigma_i}}\label{subsec:proof_sigma_i}

When $\ul_i \leq 0$, the bound by $1/0=\infty$ trivially holds. 
Threfore, it suffices to consider the case $\ul_i > 0$. In this case, 
the first statement is immediate from
\begin{align*}
\lefteqn{
J_i(\ul;\dis,\B)
}\\
&=
    \sum_{\theta=1}^{m}\int_{\nu}^\infty \frac{1}{z+\ul_i}\sum_{\bm{v}\in\mathcal{S}_{i,\theta}^\B}\qty(\prod_{j:v_j=1}\qty(1-F(z+\ul_j))\prod_{j:v_j=0,j\in\B\setminus\qty{i}}F(z+\ul_j)) \dd F(z+\ul_i)
\\
&\le
\max_{w\ge \nu}\frac{1}{w+\ul_i}\sum_{\theta=1}^{m}
    \int_{\nu}^\infty \sum_{\bm{v}\in\mathcal{S}_{i,\theta}^\B}\qty(\prod_{j:v_j=1}\qty(1-F(z+\ul_j))\prod_{j:v_j=0,j\in\B\setminus\qty{i}}F(z+\ul_j)) \dd F(z+\ul_i)
\\
&=
\frac{1}{\nu+\ul_i}
I_i(\ul;\dis,\B).
\end{align*}

On the other hand,
by noticing that the elements of $\lams$ are the same for arms $k \in [i]$ 
in Lemma~\ref{lem:lambda_star}, we obtain
\begin{align}
    \frac{J_i(\ul;\dis)}{I_i(\ul;\dis)}
    &\leq
    \max_{(j,\wm)\in[m\land i]^2: j+\wm\le (m\land i)+1}
    \qty{\frac{J_i(\lams;\dis,\wm,[j:i])}{I_i(\lams;\dis,\wm,[j:i])}}
\nn
    &=
    \max_{(j,\wm)\in[m\land i]^2: j+\wm\le (m\land i)+1}
    \qty{\frac{J_i(\lams;\dis,\wm,[i-j+1])}{I_i(\lams;\dis,\wm,[i-j+1])}}
\nn
&=
\begin{cases}
    \max_{\wm\in[i], k\in[\wm: i]}
    \qty{\frac{J_i(\lams;\dis,\wm,[k])}{I_i(\lams;\dis,\wm,[k])}},
&i\le m,\\
    \max_{\wm\in[m], k\in[i-m+\wm: i]}
    \qty{\frac{J_i(\lams;\dis,\wm,[k])}{I_i(\lams;\dis,\wm,[k])}},
&i> m.\\
\end{cases}
\label{decomp_jphi1}
\end{align}
Here each component in the maximum is bounded by
\begin{align}
    \frac{J_i(\lambda^*;\dis,\wm,[k])}{I_i(\lambda^*;\dis,\wm,[k])}
    &=
    \frac{\sum_{\theta\in[\wm]}\binom{k-1}{\theta-1}\int_{\nu}^\infty \frac{f(z+\ul_i)}{z+\ul_i}\qty(1-F(z+\ul_i))^{\theta-1}F^{k-\theta}(z+\ul_i)\dd z}
    {\sum_{\theta\in[\wm]}\binom{k-1}{\theta-1}\int_{\nu}^\infty f(z+\ul_i)\qty(1-F(z+\ul_i))^{\theta-1}F^{k-\theta}(z+\ul_i)\dd z}
\nn
    &=
    \frac{\sum_{\theta\in[\wm]}\binom{k-1}{\theta-1}\int_{\nu+\ul_i}^\infty \frac{f(z)}{z}\qty(1-F(z))^{\theta-1}F^{k-\theta}(z)\dd z}
    {\sum_{\theta\in[\wm]}\binom{k-1}{\theta-1}\int_{\nu+\ul_i}^\infty f(z)\qty(1-F(z))^{\theta-1}F^{k-\theta}(z)\dd z}
\nn
    &\le
    \frac{\sum_{\theta\in[\wm]}\binom{k-1}{\theta-1}\int_{\nu}^\infty \frac{f(z)}{z}\qty(1-F(z))^{\theta-1}F^{k-\theta}(z)\dd z}
    {\sum_{\theta\in[\wm]}\binom{k-1}{\theta-1}\int_{\nu}^\infty f(z)\qty(1-F(z))^{\theta-1}F^{k-\theta}(z)\dd z}
\label{jphi1}
\\
    &=
\frac{k}{\wm}
\sum_{\theta\in[\wm]}\binom{k-1}{\theta-1}\int_{\nu}^\infty \frac{f(z)}{z}\qty(1-F(z))^{\theta-1}F^{k-\theta}(z)\dd z, 
\label{jphi2}
\\
    &=
\begin{cases}
\frac{k}{\wm}
\sum_{\theta\in[\wm]}\binom{k-1}{\theta-1}
\int_{0}^1 \left(\log 1/u\right)^{1/\alpha}\qty(1-u)^{\theta-1}u^{k-\theta}\dd u, 
&\dis=\Fdis,\\
\frac{k}{\wm}
\sum_{\theta\in[\wm]}\binom{k-1}{\theta-1}
\int_{0}^1\qty(1-u)^{\theta-1+1/\alpha}u^{k-\theta}\dd u, 
&\dis=\Pdis,
\end{cases}
\label{eq:lambda_star_bound_theta_F}
\end{align}
where 
\eqref{jphi1} holds by Lemma~\ref{lem:mono_B} with $g(x):=\sum_{\theta\in[\wm]}\binom{k-1}{\theta-1}f(x)\qty(1-F(x))^{\theta-1}F^{k-\theta}(x)$ and $f(x):=\frac{1}{x}$,
and \eqref{jphi2} holds
because the denominator is the probability that an element of $k$ i.i.d.~RVs
becomes at least the $\wm$-th largest.
We can obtain the last equality by
letting $u=F(z)$ with the densities of Fr\'echet and Pareto distributions
given in \eqref{eq:frechet_def} and \eqref{eq:pareto_def} respectively.

\subsubsection{Pareto Distribution}
Let us consider the case $\dis=\Pdis$.
We can bound the sum of integrals in \eqref{eq:lambda_star_bound_theta_F}
by
{\allowdisplaybreaks[4]%
\begin{align}
\lefteqn{
\sum_{\theta\in[\wm]}\binom{k-1}{\theta-1}
\int_{0}^1\qty(1-u)^{\theta-1+1/\alpha}u^{k-\theta}\dd u 
}\nn
    &=
\sum_{\theta\in[\wm]}\binom{k-1}{\theta-1}
B\qty(\theta+\frac{1}{\alpha},k+1-\theta)
\nn
    &=
\sum_{\theta\in[\wm]}\binom{k-1}{\theta-1}
\frac{\Gamma(\theta+\frac{1}{\alpha})\Gamma(k+1-\theta)}{\Gamma(k+1+\frac{1}{\alpha})} 
\nn
    &=
\sum_{\theta\in[\wm]}
\frac{\Gamma(k)}{\Gamma(\theta)\Gamma(k+1-\theta)}
\frac{\Gamma(\theta+\frac{1}{\alpha})\Gamma(k+1-\theta)}{\Gamma(k+1+\frac{1}{\alpha})} 
\nn
    &=
\sum_{\theta\in[\wm]}
\frac{\Gamma(k)}{(k+\frac{1}{\alpha})\Gamma(k+\frac{1}{\alpha})} 
\frac{\Gamma(\theta+\frac{1}{\alpha})}{\Gamma(\theta)}
    \tag*{(by $\Gamma(n) = (n - 1)\Gamma(n - 1)$)}
\nn
    &\le
k^{-1-1/\alpha}
\sum_{\theta\in[\wm]}
(\theta+1/\alpha)^{1/\alpha}
    \tag*{(by Gautschi's inequality)}
\nn
    &\le
k^{-1-1/\alpha}
\int_{1}^{\wm+1}(x+1/\alpha)^{1/\alpha}\dd x
\nn
    &\le
k^{-1-1/\alpha}
\frac{(\wm+1+1/\alpha)^{1+1/\alpha}}{1+1/\alpha}
\nn
    &\le
\frac{(\wm+1+1/\alpha)^{1+1/\alpha}}{(1+1/\alpha)(\wm)^{1+1/\alpha}}
(\wm/k)^{1+1/\alpha}
\nn
    &\le
\frac{9}{2}(m/k)^{1+1/\alpha},
\label{sum_twice}
\end{align}}%
where $B\qty(a,b)=\int_0^1 (1-t)^{a-1}t^{b-1}\dd t=\frac{\Gamma(a)\Gamma(b)}{\Gamma(a+b)}$
denotes the Beta function, and the last inequality holds
because $\sup_{x\ge 1, y\in[0,1]}\frac{(x+1+y)^{1+y}}{(1+y)x^{1+y}}= 9/2$ holds by an elementary calculation.
Combining this with \eqref{decomp_jphi1} and 
\eqref{eq:lambda_star_bound_theta_F}
we obtain
\begin{align}
    \frac{J_i(\ul;\Pdis)}{I_i(\ul;\Pdis)}
&\le
\begin{cases}
\frac{9}{2}
    \max_{\wm\in[i], k\in[\wm: i]}
\left(\wm/k\right)^{1/\alpha},
&i\le m,\\
\frac{9}{2}\max_{\wm\in[m], k\in[i-m+\wm: i]}
\left(\wm/k\right)^{1/\alpha},
&i> m,\\
\end{cases}
\nn
&\le
\begin{cases}
\frac{9}{2},
&i\le m,\\
\frac{9}{2}
    \max_{\wm\in[m]}
\left(
\frac{\wm}{i-m+\wm}
\right)^{1/\alpha},
&i> m,\\
\end{cases}
\nn
&=
\frac{9}{2}
(1\land m/i)^{1/\alpha}.\n
\end{align}

\subsubsection{Fr\'{e}chet Distribution}
We decompose the summation in \eqref{eq:lambda_star_bound_theta_F}
by denoting $\wm_k =\wm\land \lfloor k/4\rfloor$ as
\begin{align}
\lefteqn{
\frac{k}{\wm}\sum_{\theta\in[\wm]}\binom{k-1}{\theta-1}
\int_0^1 \qty(\log\frac{1}{u})^{\frac{1}{\alpha}}\qty(1-u)^{\theta-1}u^{k-\theta}\dd u
}\nn
&=
\frac{k}{\wm}\sum_{\theta\in[\wm_k ]}\binom{k-1}{\theta-1}
\int_0^{\frac{1}{2}} \qty(\log\frac{1}{u})^{\frac{1}{\alpha}}\qty(1-u)^{\theta-1}u^{k-\theta}\dd u
\nn
&\quad+
\frac{k}{\wm}\sum_{\theta\in[\wm_k ]}\binom{k-1}{\theta-1}
\int_{\frac{1}{2}}^1 \qty(\log\frac{1}{u})^{\frac{1}{\alpha}}\qty(1-u)^{\theta-1}u^{k-\theta}\dd u
\nn
&\quad+
\frac{k}{\wm}\sum_{\theta\in[\wm]\setminus[\wm_k ]}\binom{k-1}{\theta-1}
\int_0^1 \qty(\log\frac{1}{u})^{\frac{1}{\alpha}}\qty(1-u)^{\theta-1}u^{k-\theta}\dd u.
\label{decomp_binom}
\end{align}

For the first term, we bound the integral by
\begin{align}
\int_0^{\frac{1}{2}} \qty(\log\frac{1}{u})^{\frac{1}{\alpha}}\qty(1-u)^{\theta-1}u^{k-\theta}\dd u
&\le
    \int_0^\frac{1}{2} \qty(1-u)^{\theta-1}u^{k-\theta-1/\alpha}\dd u
\nn
&\le
    2\int_0^\frac{1}{2} \qty(1-u)^{\theta}u^{k-\theta-1/\alpha}\dd u,
\label{fre_case1}
\end{align}
where we used $\log\frac{1}{u}\leq \frac{1}{u}$ for $u\in (0,1]$.
Here
the integrand is unimodal and takes the maximum at 
\begin{align}
u=\frac{k-\theta-1/\alpha}{k-1/\alpha},
\n
\end{align}
which becomes greater than or equal to $1/2$ when $k\geq 2/\alpha$ by $\theta\le k/4$. 
Moreover, the case $k<2/\alpha$ cannot appear because $1\leq\theta\leq k/4$ cannot hold if $k<2/\alpha$ for $\alpha>1$. 
This means that the integrand is increasing in $u\in[0,1/2]$ and therefore
\begin{align}
2\int_0^\frac{1}{2} \qty(1-u)^{\theta}u^{k-\theta-1/\alpha}\dd u
\le
2\cdot 1/2\cdot \qty(1/2)^{\theta}(1/2)^{k-\theta-1/\alpha}
=
2^{1/\alpha-k}\le 2^{-(k-1)}.\n
\end{align}

By combining this fact with \eqref{fre_case1},
the first term of \eqref{decomp_binom} is bounded by
\begin{align}
\frac{k}{\wm}\sum_{\theta\in[\wm_k ]}\binom{k-1}{\theta-1}
\int_0^{\frac{1}{2}} \qty(\log\frac{1}{u})^{\frac{1}{\alpha}}\qty(1-u)^{\theta-1}u^{k-\theta}\dd u
&\le
\frac{k}{\wm}
\sum_{\theta\in[\wm_k ]}\binom{k-1}{\theta-1}2^{-(k-1)}
\nn
&\le
\frac{k}{\wm}
\mathrm{e}^{-2(k-1)(1/2-1/4)^2}
\label{binary_hoeffding}
\\
&=
\frac{k}{\wm}
\mathrm{e}^{-(k-1)/8}
\nn
&\le
\frac{40}{\wm k}
\le
40(\wm/k)^{1/\alpha}.
\label{fre_term1}
\end{align}
Here, in \eqref{binary_hoeffding}
we used the fact that $\sum_{\theta\in[\wm_k ]}\binom{k-1}{\theta-1}2^{-(k-1)}$ is the probability that $k-1$ i.i.d.~Bernoulli RVs following $\mathrm{Ber}(1/2)$ take the mean at most $(\wm_k -1)/(k-1)\le 1/4$
and applied Hoeffding's inequality.
The last inequality follows from $x\mathrm{e}^{-(x-1)}\le 40/x$ for $x>0$.

The second term of \eqref{decomp_binom} is bounded by
\begin{align}
\lefteqn{
\frac{k}{\wm}\sum_{\theta\in[\wm_k ]}\binom{k-1}{\theta-1}
\int_{\frac{1}{2}}^1 \qty(\log\frac{1}{u})^{\frac{1}{\alpha}}\qty(1-u)^{\theta-1}u^{k-\theta}\dd u
}\nn
&\le
\frac{2^{1/\alpha}k}{\wm}\sum_{\theta\in[\wm_k ]}\binom{k-1}{\theta-1}
\int_\frac{1}{2}^1 \qty(1-u)^{\theta-1+1/\alpha}u^{k-\theta}\dd u
\label{log12}
\\
&\le
\frac{2^{1/\alpha}k}{\wm}\sum_{\theta\in[\wm_k ]}\binom{k-1}{\theta-1}
B(\theta+1/\alpha, k-\theta+1)
\nn
&\le
\frac{2^{1/\alpha}k}{\wm}
\cdot \frac{9}{2}
\left(\frac{\wm_k }{k}\right)^{1+1/\alpha}
\tag*{(by \eqref{sum_twice})}
\nn
&\le
9(\wm/k)^{1/\alpha},
\label{fre_term2}
\end{align}
where \eqref{log12} follows from
$\log\frac{1}{u}
    =
    \log\qty(1+\frac{1-u}{u})
    \leq
    \frac{1-u}{u}
    \leq
    2(1-u)$ for $u\in[1/2,1]$.

Now we bound the third term of \eqref{decomp_binom}
in a way similar to \eqref{sum_twice}.
Note that this term is not vacant only when
$\wm_k <\wm$, that is,
$\lfloor k/4\rfloor \le \wm-1$. 
Then
\begin{align}
\lefteqn{
\frac{k}{\wm}
\sum_{\theta\in[\wm]\setminus[\wm_k ]}\binom{k-1}{\theta-1}
\int_0^1 \qty(\log\frac{1}{u})^{\frac{1}{\alpha}}\qty(1-u)^{\theta-1}u^{k-\theta}\dd u
}
\nn
&\le
\frac{k}{\wm}
\ind[\lfloor k/4\rfloor \le \wm-1]
\sum_{\theta=\lfloor k/4\rfloor+1}^{\wm}\binom{k-1}{\theta-1}
\int_0^1 \qty(1-u)^{\theta-1}u^{k-\theta-1/\alpha}\dd u
\tag*{(by $\log 1/u\le 1/u$)}
\nn
&=
\frac{k}{\wm}
\ind[\lfloor k/4\rfloor \le \wm-1]
\sum_{\theta=\lfloor k/4\rfloor+1}^{\wm}
\frac{\Gamma(k)}{\Gamma(\theta)\Gamma(k-\theta+1)}
\frac{\Gamma(\theta)\Gamma(k-\theta+1-1/\alpha)}{\Gamma(k+1-1/\alpha)}
\nn
&\le
\frac{k}{\wm}
\ind[k/4\le \wm]
k^{-1+1/\alpha}
\sum_{\theta=\lfloor k/4\rfloor+1}^{\wm}
(k-\theta+1-1/\alpha)^{-1/\alpha}
\tag*{(by Gautschi's inequality)}
\nn
&\le
4
\ind[k/4\le \wm]
k^{1/\alpha-1}
\frac{(k-\lfloor k/4\rfloor-1/\alpha)^{1-1/\alpha}}{1-1/\alpha}
\nn
&\le
4
\ind[k/4\le \wm]
\frac{(3/4+(1-1/\alpha)/k)^{1-1/\alpha}}{1-1/\alpha} 
\nn
&\le
\frac{7}{1-1/\alpha}
\ind[k/4\le \wm].
\label{fre_term3}
\end{align}

By substituting\ \eqref{fre_term1}, \eqref{fre_term2} and \eqref{fre_term3} into \eqref{decomp_binom}
we obtain
\begin{align}
\frac{k}{\wm}\sum_{\theta\in[\wm]}\binom{k-1}{\theta-1}
\int_0^1 \qty(\log\frac{1}{u})^{\frac{1}{\alpha}}\qty(1-u)^{\theta-1}u^{k-\theta}\dd u
&\le
49
(\wm/k)^{1/\alpha}
+
\frac{7}{1-1/\alpha}\ind[k/4\le \wm].
\n
\end{align}
Putting this inequality,
\eqref{decomp_jphi1} and \eqref{eq:lambda_star_bound_theta_F} together 
we obtain
\begin{align*}
\frac{J_i(\ul;\Fdis)}{I_i(\ul;\Fdis)}
&\le
\begin{cases}
{\displaystyle\max_{\wm\in[i], k\in[\wm: i]}}
\left\{
49
(\wm/k)^{1/\alpha}
+
\frac{7}{1-1/\alpha}\ind[k/4\le \wm]
\right\},
&i\le m,\\
{\displaystyle \max_{\wm\in[m], k\in[i-m+\wm: i]}}
\left\{
49
(\wm/k)^{1/\alpha}
+
\frac{7}{1-1/\alpha}\ind[k/4\le \wm]
\right\},
&i> m,\\
\end{cases}
\nn
&\le
\begin{cases}
{\displaystyle\max_{\wm\in[i]}}
\left\{
49
+
\frac{7}{1-1/\alpha}
\right\},
&i\le m,\\
{\displaystyle \max_{\wm\in[m]}}
\left\{
49
(\wm/(i-m+\wm))^{1/\alpha}
+
\frac{7}{1-1/\alpha}\ind[(i-m+\wm)/4\le \wm]
\right\},
&i> m,\\
\end{cases}
\nn
&=
\begin{cases}
49
+
\frac{7}{1-1/\alpha},
&i\le m,\\
49
(m/i)^{1/\alpha}
+
\frac{7}{1-1/\alpha}\ind[i/4\le m],
&i> m,\\
\end{cases}
\nn
&\le
\left(49+\frac{7\cdot 4^{1/\alpha}}{1-1/\alpha}\right)\left(1\land \frac{m}{i}\right)^{1/\alpha}.\dqed
\end{align*}

\section{Regret Bound for Stochastic Bandits}\label{sec:stab_sto}
In this section, we provide the proofs of Theorems~\ref{thm:sto_bound} and \ref{thm:sto_bound_alpha_not2} on the stochastic regret bounds by applying the self-bounding technique \citep{zimmert2021tsallis}, which is a typical tool for establishing BOBW property of FTRL. 
In Sections~\ref{subsec:regret_lower_bounds} and \ref{subsec:regret_optimal_action}, we extend the techniques of \citet{pmlr-v201-honda23a} and \citet{pmlr-v247-lee24a} to the $m$-set semi-bandit problem, deriving a regret lower bound and regret for the optimal action of the policy.
Both are key components required for the regret analysis by the self-bounding technique, with detail presented in Section~\ref{subsec:proofs_sto_bounds}.
In this proof, we consider the event $F_t$ (resp. $D_t$) for $\Fdis$ (resp. $\Pdis$), which is defined as
\begin{align*}
    F_t &\coloneqq\qty{\sum_{i:a_i^*=0}\frac{1}{(\eta_t\underline{\hat{L}}_{t,i})^\alpha}\leq\frac{1}{m^{\alpha/(\alpha-1)}}},\\
    D_t &\coloneqq\qty{\sum_{i:a_i^*=0} \frac{1}{(2^{\frac{1}{\alpha}}+\eta_t\underline{\hat{L}}_{t,i})^\alpha}
    \leq 
    \frac{1}{(2^{\frac{1}{\alpha}}+m^{1/(\alpha-1)})^\alpha}}.
\end{align*} 
One can see that on both events, there exists the key property that
\begin{equation}\label{eq:key_property}
    \eta_t\underline{\hat{L}}_{t,i}\leq 0,\,\forall i:a_i^*=1,\quad
    \text{and}\quad
    \eta_t\underline{\hat{L}}_{t,i}\geq m^{1/(\alpha-1)},\,\forall i:a_i^*=0.
\end{equation}
\subsection{Regret Lower Bounds}\label{subsec:regret_lower_bounds}
Here, we provide the regret lower bounds on $\sum_{i:a^*_i=0}\Delta_i w_{t,i}$ for $\Fdis$ and $\Pdis$ in Lemmas~\ref{lem:lower_bound_Ft} and \ref{lem:lower_bound_Dt}, respectively.
\begin{lemma}\label{lem:lower_bound_Ft}
    Let $\Delta\coloneqq\min_{i:a_i^*=0}\Delta_i$. 
    Then, there exists some distribution-dependent constant $c_{s,1}(\Fdis),c_{s,2}(\Fdis)\in\qty(0,1)$ such that
    \begin{enumerate}[label=(\roman*)]
        \item On $F_t$, $\sum_{i:a^*_i=0}\Delta_i w_{t,i}\geq c_{s,1}(\Fdis)\sum_{i:a^*_i=0}\frac{\Delta_i}{(\eta_t\underline{\hat{L}}_{t,i})^\alpha}$ and $w_{t,i}\geq 1/e$ for any $i$ with $a^*_i=1$.
        \item On $F_t^c$, $\sum_{i:a^*_i=0}\Delta_i w_{t,i}\geq c_{s,2}(\Fdis)\Delta/m^{\alpha/(\alpha-1)}$.
    \end{enumerate}
\end{lemma}

\begin{proof}
Recall that on $F_t$, we have $\eta_t\underline{\hat{L}}_{t,i}\leq 0$ for any $i$ with $a^*_i=1$ and $\eta_t\underline{\hat{L}}_{t,i}> m^{1/(\alpha-1)}$ for any $i$ with $a^*_i=0$. 
    Then, by letting $j^*=\argmax_{j:a^*_j=1}\hat{L}_{t,j}$ so that $\hat{L}_{t,j^*}$ is the $m$-th smallest cumulative loss among $[d]$ on $F_t$,
    for any $i$ with $a^*_i=0$ or $i=j^*$ we have
    \begin{align*} 
        w_{t,i}&=\phi_{i}(\eta_t\hat{L}_t;m,[d])\\
        &\geq \lim_{\max_{j:a_j^*=1,j\neq j^*}\hat{L}_{t,j}\to -\infty}\phi_i(\eta_t\hat{L}_t;m,[d])\\
        &=\lim_{\max_{j:a_j^*=1,j\neq j^*}\hat{L}_{t,j}\to -\infty}\phi_{i,m}(\eta_t\hat{L}_t;m,[d])\tag*{(by $\phi_{i,\theta}(\eta_t\hat{L}_t;m,[d])\to 0$ for $\theta\in[m-1]$)}\\
        &=\phi_i(\eta_t\hat{L}_t;1,\qty{i:a_i^*=0}\cup\qty{j^*}),\numberthis\label{eq:lower_bound_key_Ft}
    \end{align*}
where recall that $\phi_{i}(\lambda;\wm, \mathcal{B})$ was defined in \eqref{eq:definition_phi_extension}. 
Here, $\phi_i(\eta_t\hat{L}_t;1,\qty{i:a_i^*=0}\cup\qty{j^*})$ is nothing but
the arm-selection probability of arm $i$ among arm set $\mathcal{B}=\qty{i:a_i^*=0}\cup\qty{j^*}$ in the MAB with the same cumulative loss vector.
Furthremore,
$j^*$ is the arm with the $m$-th smallest cumulative loss $\hat{L}_{t,i}$ among arm set $[d]$
and
is also the arm with the smallest cumulative loss $\hat{L}_{t,i}$ among arm set $\mathcal{B}=\qty{i:a_i^*=0}\cup\qty{j^*}$.
Therefore, recalling that $\underline{\hat{L}}_{t,i}=\hat{L}_{t,i}-\mbox{($m$-th smallest $\hat{L}_{t,i'}$ among the arm set)}$ we see that, 
for any $i\in\mathcal{B}$, $\underline{\hat{L}}_{t,i}$ under arm set $\mathcal{B}$ with $m=1$ (MAB) is identical to $\underline{\hat{L}}_{t,i}$ under arm set $[d]$ with the original $m$.

Now we apply a known result for the MAB
to the above setting.
Let $\mathcal{B}$ be an arm set and $j\in\mathcal{B}$ be arbitrary.
Lemma~21 of \citet{pmlr-v247-lee24a} shows that, in the case of $m=1$,
if $\qty{\sum_{i\in \mathcal{B}\setminus\{j\}}\frac{1}{(\eta_t\underline{\hat{L}}_{t,i})^\alpha}\leq 1}$ holds then\footnote{%
In Lemma~21 of \citet{pmlr-v247-lee24a} the arm $j$ is specified to be the optimal arm, but this property is used nowhere in their proof.
Similarly, the property of $\Delta_i\ge 0$ is also used nowhere in their proof, and thus it also holds under $\Delta_i$ determined based
on the $m$-set setting.}
    \begin{equation}\label{eq:lower_bounds_mab_Ft}
        \sum_{i\in \mathcal{B}\setminus\{j\}}\Delta_i w_{t,i}\geq
c_{s,1}(\Fdis)\sum_{i\in \mathcal{B}\setminus\{j\}}\frac{\Delta_i}{(\eta_t\underline{\hat{L}}_{t,i})^\alpha}
        \quad\text{and}\quad
        w_{t,j}\geq \frac{1}{e},
    \end{equation}
where
$c_{s,1}(\Fdis)=\max_{b>0} \frac{\exp\qty( -\qty( 1 + \frac{1}{b^\alpha} ))}{(1+b)^\alpha} \in (0,1)$.

We can apply this result with $\mathcal{B}:=\qty{i:a_i^*=0}\cup\qty{j^*}$ and $j:=j^*$ to the above setting
because
the condition for \eqref{eq:lower_bounds_mab_Ft} indeed holds under $F_t$ since
\begin{align}
        F_t=\qty{\sum_{i:a_i^*=0}\frac{1}{(\eta_t\underline{\hat{L}}_{t,i})^\alpha}\leq\frac{1}{m^{\alpha/(\alpha-1)}}}
        &\subset 
        \qty{\sum_{i:a_i^*=0}\frac{1}{(\eta_t\underline{\hat{L}}_{t,i})^\alpha}\leq 1}
=
        \qty{\sum_{i\in \mathcal{B}\setminus \{j\}}\frac{1}{(\eta_t\underline{\hat{L}}_{t,i})^\alpha}\leq 1}.\n
\end{align}
Then we obtain the following two results on $F_t$ for $m$-set semi-bandits.
    First, we have
    \begin{align*}
        \sum_{i:a_i^*=0} \Delta_i w_{t,i}
        &\geq
        \sum_{i:a_i^*=0} \Delta_i \phi_i(\eta_t\hat{L}_t;1,\qty{i:a_i^*=0}\cup\qty{j^*})\tag{by \eqref{eq:lower_bound_key_Ft}}\\
        &\geq c_{s,1}(\Fdis)\sum_{i:a_i^*=0}\frac{\Delta_i}{(\eta_t\underline{\hat{L}}_{t,i})^\alpha}.\tag{by the first inequality in \eqref{eq:lower_bounds_mab_Ft}}
    \end{align*}
    Second, for any $i$ with $a^*_i=1$, we have
    \begin{equation*}
        w_{t,i}\geq w_{t,j^*}\geq \phi_{j^*}(\eta_t\hat{L}_t;1,\qty{i:a_i^*=0}\cup\qty{j^*})\geq \frac{1}{e},
    \end{equation*}
    where the second inequality follows from \eqref{eq:lower_bound_key_Ft} and the last inequality follows from the second inequality in \eqref{eq:lower_bounds_mab_Ft}.

    Next, we consider the lower bound on $F_t^c$. 
    Define
    \begin{align*}
        F_{t,1}^c &\coloneqq \qty{\exists j:a_j^*=0, \eta_t\underline{\hat{L}}_{t,j}\leq 0},\\
        F_{t,2}^c &\coloneqq \qty{\sum_{i:a_i^*=0}\frac{1}{(\eta_t\underline{\hat{L}}_{t,i})^\alpha}>\frac{1}{m^{\alpha/(\alpha-1)}}\mbox{ and }\forall i:a_i^*=0,\,\eta_t\underline{\hat{L}}_{t,i}>0}.
    \end{align*}
    Then, we have
    \begin{equation*}
        F_t^c=F_{t,1}^c\cup F_{t,2}^c.
    \end{equation*}
    Now, let us consider these two subevents separately. 
    On the first subevent $F_{t,1}^c$, since there exists a base-arm $j_0$ with $a^*_{j_0}=0$ satisfying $\eta_t\underline{\hat{L}}_{t,j_0}\leq 0$, there must exist at least one base-arm $j_1$ with $a^*_{j_1}=1$ such that $\eta_t\underline{\hat{L}}_{t,j_1}\geq 0\geq \eta_t\underline{\hat{L}}_{t,j_0}$, leading to $\eta_t\hat{L}_{t,j_1}\geq \eta_t\hat{L}_{t,j_0}$. 
    Therefore, we have
    \begin{align*}
        \sum_{i:a_i^*=0} \Delta_i w_{t,i}
        &\geq
        \Delta\sum_{i:a_i^*=0} \phi_i(\eta_t\hat{L}_t;m,[d])\\
        &\geq 
        \lim_{\min_{i:a_i^*=0,i\neq j_0}\hat{L}_{t,i}\to +\infty}\Delta\sum_{i:a_i^*=0} \phi_i(\eta_t\hat{L}_t;m,[d])\tag*{(by Lemma~\ref{lem:phi_monotone})}\\
        &=
        \Delta\cdot\phi_{j_0}(\eta_t\hat{L}_t;m,\qty{i:a^*_i=1}\cup\qty{j_0})\\
        &\geq 
        \lim_{\max_{i:a_i^*=1,i\neq j_1}\hat{L}_{t,i}\to -\infty} \Delta\cdot\phi_{j_0}(\eta_t\hat{L}_t;m,\qty{i:a^*_i=1}\cup\qty{j_0})\tag*{(by Lemma~\ref{lem:phi_monotone})}\\
        &=
        \Delta\cdot\phi_{j_0}(\eta_t\hat{L}_t;1,\qty{j_0,j_1})\\
        &\geq 
        \frac{\Delta}{2},\numberthis\label{eq:lower_bound_Ftc_1}
    \end{align*}
    where the last inequality holds since $\eta_t\hat{L}_{t,j_1}\geq \eta_t\hat{L}_{t,j_0}$.
    On the other hand, on $F_{t,2}^c$, by letting $j^*=\argmax_{j:a^*_j=1}\hat{L}_{t,j}$ we have
    \begin{align*}
        \sum_{i:a_i^*=0} w_{t,i}&=
        \sum_{i:a_i^*=0} \phi_i(\eta_t\hat{L}_t;m,[d])\\
        &\geq \lim_{\max_{i:a_i^*=1,i\neq j^*}\hat{L}_{t,i}\to -\infty}\sum_{i:a_i^*=0}\phi_i(\eta_t\hat{L}_t;m,[d])\tag*{(by Lemma~\ref{lem:phi_monotone})}\\
        &=\sum_{i:a_i^*=0}\phi_i(\eta_t\hat{L}_t;1,\qty{i:a_i^*=0}\cup\qty{j^*}).\numberthis\label{eq:lower_bound_Ftc_2_intermediate}
    \end{align*}
    Note that for any $i$ with $a^*_i=0$, the value of $\underline{\hat{L}}_{t,i}$ appearing in the integral expression of $\phi_{i}(\eta_t\hat{L}_t;m,[d])$ is the same as that in $\phi_i(\eta_t\hat{L}_t;1,\qty{i:a_i^*=0}\cup\qty{j^*})$, where the definition of $\phi_i(\cdot)$ is given in \eqref{eq:definition_phi_extension}. The latter is exactly the arm-selection probability of arm $i$ among $d-m+1$ arms $\qty{i:a_i^*=0}\cup\qty{j^*}$ in MAB setting with the same cumulative loss vector.
    Then, following the proof of \citet{pmlr-v247-lee24a}, we give a lower bound for $\sum_{i:a^*_i=0}\Delta_i w_{t,i}$ on $F_{t,2}^c$ as follows. 
    The argument is identical to theirs, except that we use a different lower bound for $\eta_t \underline{\hat{L}}_{t,i}$ as given in \eqref{eq:key_property}.
    
    Let 
    $\underline{\hat{L}}'=\min_{i:a_i^*=0}\underline{\hat{L}}_{t,i}$. 
    When 
    $\sum_{i:a_i^*=0}\frac{1}{(\eta_t\underline{\hat{L}}_{t,i})^\alpha}>1/m^{\alpha/(\alpha-1)}$ 
    and 
    $\eta_t\underline{\hat{L}}_{t,i}>0$ for any $i$ with $a^*_i=0$, 
    for any $z>\eta_t \underline{\hat{L}}'$ we have
    \begin{align*}
    \sum_{i\in\{i:a_i^*=0\}\cup\{j^*\}} \frac{1}{(z + \eta_t \underline{\hat{L}}_{t,i})^\alpha}
    &\leq
    \sum_{i:a_i^*=0}\frac{1}{(z + \eta_t \underline{\hat{L}}_{t,i})^\alpha}
    + \frac{1}{z^\alpha} 
    \displaybreak[0]\\
    &\leq
    \sum_{i:a^*_i=0} \frac{1}{(z + \eta_t \underline{\hat{L}}_{t,i})^\alpha}
    + \frac{1}{\qty(\frac{z + \eta_t \underline{\hat{L}}'}{2})^\alpha}
    \displaybreak[0]\\
    &\leq
    \sum_{i:a^*_i=0} \frac{1}{(z + \eta_t \underline{\hat{L}}_{t,i})^\alpha}
    + \sum_{i:a^*_i=0} \frac{2^\alpha}{(z + \eta_t \underline{\hat{L}}_{t,i})^\alpha}
    \displaybreak[0]\\
    &=
    \sum_{i:a^*_i=0} \frac{2^\alpha + 1}{(z + \eta_t \underline{\hat{L}}_{t,i})^\alpha}.
    \end{align*}
    Therefore, by applying the above inequality to the integral expression of $\phi_i(\cdot)$ given in \eqref{eq:definition_phi_extension}, 
    we obtain
    \begin{align*}
        \sum_{i:a^*_i=0}\Delta_i w_{t,i}
        &\geq
        \sum_{i:a^*_i=0}\Delta_i \phi_i(\eta_t\hat{L}_t;1,\qty{i:a_i^*=0}\cup\qty{j^*})\tag{by \eqref{eq:lower_bound_Ftc_2_intermediate}}\\
        &=
        \int_0^{\infty}
        \sum_{i:a^*_i=0} \Delta_i f\qty(z + \eta_t \underline{\hat{L}}_{t,i})
        \prod_{i'\in\{j:a_j^*=0,\,j\neq i\}\cup\{j^*\}} F\qty(z + \eta_t \underline{\hat{L}}_{t,i'})\dd z\\
        &=
        \int_0^{\infty}
        \sum_{i:a^*_i=0} \frac{\Delta_i}{(z + \eta_t \underline{\hat{L}}_{t,i})^{\alpha+1}}
        \exp\qty(-\sum_{i'\in\{i:a_i^*=0\}\cup\{j^*\}} \frac{1}{(z + \eta_t \underline{\hat{L}}_{t,i'})^\alpha}) \dd z \\
        &\geq
        \alpha \Delta \int_{\eta_t \underline{\hat{L}}'}^\infty
        \qty(\sum_{i:a^*_i=0} \frac{1}{(z + \eta_t \underline{\hat{L}}_{t,i})^{\alpha+1}})
        \exp\qty(-\sum_{i:a^*_i=0} \frac{2^\alpha + 1}{(z + \eta_t \underline{\hat{L}}_{t,i})^\alpha}) \dd z 
        \displaybreak[0]\\
        &=
        \frac{\Delta}{2^\alpha + 1}
        \qty(1-\exp\qty(-\sum_{i:a^*_i=0}\frac{2^\alpha + 1}{(\eta_t \underline{\hat{L}}' + \eta_t \underline{\hat{L}}_{t,i})^\alpha})) 
        \displaybreak[0]\\
        &\geq
        \frac{\Delta}{2^\alpha + 1}
        \qty(1-\exp\qty(-\sum_{i:a^*_i=0}\frac{2^\alpha + 1}{2^\alpha (\eta_t\underline{\hat{L}}_{t,i})^\alpha})) 
        \displaybreak[0]\\
        &\geq
        \frac{\Delta}{2^\alpha + 1}\qty(1-\exp\qty(-\frac{2^\alpha + 1}{2^\alpha m^{\alpha/(\alpha-1)}})) 
        \displaybreak[0]\\
        &\geq
        \frac{\Delta}{2^\alpha + 1}\frac{2^\alpha + 1}{2^\alpha \qty(m^{\alpha/(\alpha-1)}+1)+1}
        \displaybreak[0]\\
        &=
        \frac{\Delta}{2^\alpha \qty(m^{\alpha/(\alpha-1)}+1)+1},\numberthis\label{eq:lower_bound_Ftc_2}
    \end{align*}
    where the last inequality follows from $1-e^{-x}\geq x/(1+x)$ for $x\geq -1$. Note that for $\alpha>1$ and $m\geq 1$, it holds that 
    \begin{equation*}
        \frac{\Delta}{2}>
        \frac{\Delta}{2^\alpha \qty(m^{\alpha/(\alpha-1)}+1)+1}.
    \end{equation*}
    Therefore, by combining \eqref{eq:lower_bound_Ftc_1} and \eqref{eq:lower_bound_Ftc_2}, we have
    \begin{equation*}
        \sum_{i:a_i^*=0} \Delta_i w_{t,i}\geq 
        \frac{\Delta}{2^\alpha \qty(m^{\alpha/(\alpha-1)}+1)+1}=c_{s,2}(\Fdis)\frac{\Delta}{m^{\alpha/(\alpha-1)}},
    \end{equation*}
    where $c_{s,2}(\Fdis)=\frac{m^{\alpha/(\alpha-1)}}{2^\alpha \qty(m^{\alpha/(\alpha-1)}+1)+1}\in\qty(\frac{1}{2^{\alpha+1}+1},\frac{1}{2^\alpha})$.
\end{proof}

\begin{lemma}\label{lem:lower_bound_Dt}
    Let $\Delta\coloneqq\min_{i:a_i^*=0}\Delta_i$. 
    Then, there exist some distribution-dependent constants $c_{s,1}(\Pdis),c_{s,2}(\Pdis)\in\qty(0,1)$ such that
    \begin{enumerate}[label=(\roman*)]
        \item On $D_t$, $\sum_{i:a^*_i=0}\Delta_i w_{t,i}\geq c_{s,1}(\Pdis)\sum_{i:a^*_i=0}\frac{\Delta_i}{(\eta_t\underline{\hat{L}}_{t,i})^\alpha}$ and $w_{t,i}\geq 0.14$ for any $i$ with $a^*_i=1$.
        \item On $D_t^c$, $\sum_{i:a^*_i=0}\Delta_i w_{t,i}\geq c_{s,2}(\Pdis)\Delta/m^{\alpha/(\alpha-1)}$.
    \end{enumerate}
\end{lemma}

\begin{proof}
    Recall that on $D_t$, we have $\eta_t\underline{\hat{L}}_{t,i}\leq 0$ for any $i$ with $a^*_i=1$ and $\eta_t\underline{\hat{L}}_{t,i}> m^{1/(\alpha-1)}$ for any $i$ with $a^*_i=0$. 
    Similarly to the proof of Lemma~\ref{lem:lower_bound_Ft}, by letting $j^*=\argmax_{j:a^*_j=1}\hat{L}_{t,j}$ so that $\hat{L}_{t,j^*}$ is the $m$-th smallest cumulative loss among $[d]$ on $D_t$, for any $i$ with $a^*_i=0$ or $i=j^*$ we have
    \begin{equation}\label{eq:lower_bound_key_Dt}
        w_{t,i}
        =
        \phi_{i}(\eta_t\hat{L}_t;m,[d])
        \geq
        \phi_i(\eta_t\hat{L}_t;1,\qty{i:a_i^*=0}\cup\qty{j^*}),
    \end{equation}
    where the value of $\underline{\hat{L}}_{t,i}$ appearing in the integral expression of $\phi_{i}(\eta_t\hat{L}_t;m,[d])$ is the same as that in $\phi_i(\eta_t\hat{L}_t;1,\qty{i:a_i^*=0}\cup\qty{j^*})$, both of which take the value $\underline{\hat{L}}_{t,i}=\hat{L}_{t,i}-\hat{L}_{t,j^*}$. 
    In addition, one can see that $\phi_i(\eta_t\hat{L}_t;1,\qty{i:a_i^*=0}\cup\qty{j^*})$ is exactly the arm-selection probability of arm $i$ among $d-m+1$ arms $\{j: j=i\text{ or }j:a^*_j=0\}$ in MAB setting with the same cumulative loss vector.
    Note that since $F\qty(2^{\frac{1}{\alpha}}+m^{1/(\alpha-1)})\geq F\qty(2^{\frac{1}{\alpha}}+1)$, the event $D_t$ defined here satisfies
    \begin{multline*}
        D_t
        =
        \qty{\sum_{i:a_i^*=0} \frac{1}{(2^{\frac{1}{\alpha}}+\eta_t\underline{\hat{L}}_{t,i})^\alpha}
        \leq 
        \frac{1}{(2^{\frac{1}{\alpha}}+m^{1/(\alpha-1)})^\alpha}}
        \\
        \subset\qty{\sum_{i:a_i^*=0} \frac{1}{(2^{\frac{1}{\alpha}}+\eta_t\underline{\hat{L}}_{t,i})^\alpha}
        \leq 
        \frac{1}{(2^{\frac{1}{\alpha}}+1)^\alpha}}.
    \end{multline*}
    Here, the RHS is actually the event considered in Lemma~22 of \citet{pmlr-v247-lee24a} for MAB setting.
    Then, on this event, 
    we apply the lower bounds established in their Lemma~22, which we restate in our notation as follows.

    Lemma~22 of \citet{pmlr-v247-lee24a} shows that, in the case of $m=1$ where $\|a^*\|_1=1$, by letting base-arm $j^*$ be such that
    $a^*_{j^*}=1$, we have
    \begin{equation}\label{eq:lower_bounds_mab_Dt}
        \sum_{i:a^*_i=0}\Delta_i w_{t,i}\geq c_{s,1}(\Pdis)\sum_{i:a_i^*=0}\frac{\Delta_i}{(\eta_t\underline{\hat{L}}_{t,i})^\alpha}
        \quad\text{and}\quad
        w_{t,j^*}\geq 0.14,
    \end{equation}
    where $c_{s,1}(\Pdis)$ is given as $\alpha/e(2^{1/\alpha}+1)^\alpha\approx\frac{e^{-1}}{2}$.
    
    Similarly to Lemma~\ref{lem:lower_bound_Ft}, as a direct consequence, we have the following two results on $F_t$ for $m$-set semi-bandits.
    First, we have
    \begin{align*}
        \sum_{i:a_i^*=0} \Delta_i w_{t,i}
        &\geq
        \sum_{i:a_i^*=0} \Delta_i \phi_i(\eta_t\hat{L}_t;1,\qty{i:a_i^*=0}\cup\qty{j^*})\tag{by \eqref{eq:lower_bound_key_Dt}}\\
        &\geq c_{s,1}(\Pdis)\sum_{i:a_i^*=0}\frac{\Delta_i}{(\eta_t\underline{\hat{L}}_{t,i})^\alpha}.\tag{by the first inequality in \eqref{eq:lower_bounds_mab_Dt}}
    \end{align*}
    The second is that, for any $i$ with $a^*_i=1$,
    \begin{equation*}
        w_{t,i}\geq w_{t,j^*}\geq \phi_{j^*}(\eta_t\hat{L}_t;1,\qty{i:a_i^*=0}\cup\qty{j^*})\geq 0.14,
    \end{equation*}
    where the second inequality follows from \eqref{eq:lower_bound_key_Dt} and the last inequality follows from the second inequality in \eqref{eq:lower_bounds_mab_Dt}. 
    
    Next, we consider the lower bound on $D_t^c$. Define 
    \begin{align*}
        D_{t,1}^c &\coloneqq \qty{\exists j:a_j^*=0, \eta_t\underline{\hat{L}}_{t,j}\leq 0},\displaybreak[0]\\
        D_{t,2}^c &\coloneqq \Bigg\{\sum_{i:a_i^*=0} \frac{1}{(2^{\frac{1}{\alpha}}+\eta_t\underline{\hat{L}}_{t,i})^\alpha}
        >
        \frac{1}{(2^{\frac{1}{\alpha}}+m^{1/(\alpha-1)})^\alpha}
        \mbox{ and }\forall i:a_i^*=0,\,\eta_t\underline{\hat{L}}_{t,i}>0\Bigg\}.
    \end{align*}
    Then, we have
    \begin{equation*}
        D_t^c=D_{t,1}^c\cup D_{t,2}^c.
    \end{equation*}
    On $D_{t,1}^c$, by the same reasoning as in the proof of Lemma~\ref{lem:lower_bound_Ft}, which establishes \eqref{eq:lower_bound_Ftc_1}, we have
    \begin{equation}\label{eq:lower_bound_Dtc_1}
        \sum_{i:a_i^*=0} \Delta_i w_{t,i}\geq \frac{\Delta}{2}.
    \end{equation}
    On the other hand, on $D_{t,2}^c$, similarly to \eqref{eq:lower_bound_Ftc_2_intermediate} in the proof of Lemma~\ref{lem:lower_bound_Ft}, by letting $j^*=\argmax_{j:a^*_j=1}\hat{L}_{t,j}$ we have
    \begin{equation*}
        \sum_{i:a_i^*=0} w_{t,i}=\sum_{i:a_i^*=0}\phi_{i}(\eta_t\hat{L}_t;m,[d])\geq\sum_{i:a_i^*=0}\phi_i(\eta_t\hat{L}_t;1,\qty{i:a_i^*=0}\cup\qty{j^*}).
    \end{equation*}
    Here, for any $i$ with $a^*_i=0$, the value of $\underline{\hat{L}}_{t,i}$ appearing in the integral expression of $\phi_{i}(\eta_t\hat{L}_t;m,[d])$ is the same as that in $\phi_i(\eta_t\hat{L}_t;1,\qty{i:a_i^*=0}\cup\qty{j^*})$, where the latter is actually the arm-selection probability of arm $i$ among $d-m+1$ arms $\qty{i:a_i^*=0}\cup\qty{j^*}$ in MAB setting with the same cumulative loss vector.
    Similarly, following the proof of \citet{pmlr-v247-lee24a}, we give a lower bound for $\sum_{i:a^*_i=0}\Delta_i w_{t,i}$ on $D_{t,2}^c$ as follows. 
    The argument is essentially the same as theirs, except that a different lower bound for $\eta_t \underline{\hat{L}}_{t,i}$ given in \eqref{eq:key_property} is used here.
    
    \begin{align*}
        &\lefteqn{\sum_{i:a^*_i=0} \Delta_i w_{t,i}
        \geq 
        \sum_{i:a_i^*=0}\Delta_i\phi_i(\eta_t\hat{L}_t;1,\qty{i:a_i^*=0}\cup\qty{j^*})}
        \displaybreak[0]\\
        &=
        \int_{1}^{\infty}
        \sum_{i:a^*_i=0} \qty(\Delta_i f\qty(z + \eta_t \underline{\hat{L}}_{t,i})
        \prod_{j \neq i} F\qty(z + \eta_t \underline{\hat{L}}_{t,j}))\dd z
        \displaybreak[0]\\
        &\geq
        \int_{1}^{\infty}
        \qty(\sum_{i:a^*_i=0} \Delta_i f\qty(z + \eta_t \underline{\hat{L}}_{t,i}))
        \prod_{j \in \{i:a_i^*=0\}\cup\{j^*\}} F\qty(z + \eta_t \underline{\hat{L}}_{t,j})\dd z
        \displaybreak[0]\\
        &\geq
        \int_{1}^{\infty}
        \qty(\sum_{i:a^*_i=0} \Delta_i f\qty(z + \eta_t \underline{\hat{L}}_{t,i}))
        \exp\qty(- \sum_{j \in \{i:a_i^*=0\}\cup\{j^*\}}\frac{1 - F\qty(z + \eta_t \underline{\hat{L}}_{t,j})}
        {F\qty(z + \eta_t \underline{\hat{L}}_{t,j})})\dd z\numberthis\label{eq:lower_bound_Dtc_2_intermediate1}
        \displaybreak[0]\\
        &\geq
        \int_{1}^{\infty}\qty(\sum_{i:a^*_i=0} \Delta_i f\qty(z + \eta_t \underline{\hat{L}}_{t,i}))
        \exp\qty(- \sum_{j:a^*_j=0}\frac{1 - F\qty(z + \eta_t \underline{\hat{L}}_{t,j})}{F\qty(z + \eta_t \underline{\hat{L}}_{t,j})})\exp\qty(- \frac{1 - F(z)}{F(z)})\dd z,
    \end{align*}
    where \eqref{eq:lower_bound_Dtc_2_intermediate1} holds since $e^{-\frac{x}{1-x}}<1-x$ for $x< 1$. Then, similarly, we have
    \begin{align*}
        &
        \lefteqn{\int_{1}^{\infty}\qty(\sum_{i:a^*_i=0} \Delta_i f\qty(z + \eta_t \underline{\hat{L}}_{t,i}))
        \exp\qty(- \sum_{j:a^*_j=0}\frac{1 - F\qty(z + \eta_t \underline{\hat{L}}_{t,j})}{F\qty(z + \eta_t \underline{\hat{L}}_{t,j})})\exp\qty(- \frac{1 - F(z)}{F(z)})\dd z}\\
        &\geq
        \Delta\int_{2^{\frac{1}{\alpha}}}^{\infty}\qty(\sum_{i:a^*_i=0}  f\qty(z + \eta_t \underline{\hat{L}}_{t,i}))
        \exp\qty(- \sum_{j:a^*_j=0}\frac{1 - F\qty(z + \eta_t \underline{\hat{L}}_{t,j})}{F\qty(z + \eta_t \underline{\hat{L}}_{t,j})})\exp\qty(- \frac{1 - F(z)}{F(z)})\dd z
        \displaybreak[0]\\
        &\geq
        \Delta e^{-1}\int_{2^{\frac{1}{\alpha}}}^{\infty}\qty(\sum_{i:a^*_i=0}  f\qty(z + \eta_t \underline{\hat{L}}_{t,i}))
        \exp\qty(- 2\sum_{j:a^*_j=0}\qty(1 - F\qty(z + \eta_t \underline{\hat{L}}_{t,j})))\dd z\numberthis\label{eq:lower_bound_Dtc_2_intermediate2}
        \displaybreak[0]\\
        &=
        \Delta \frac{e^{-1}}{2}\qty(1-\exp\qty(- 2\sum_{j:a^*_j=0}\qty(1 - F\qty(2^{\frac{1}{\alpha}} + \eta_t \underline{\hat{L}}_{t,j}))))
        \displaybreak[0]\\
        &\geq
        \Delta \frac{e^{-1}}{2}\qty(1-\exp\qty(- 2\qty(1 - F\qty(2^{\frac{1}{\alpha}} + m^{1/(\alpha-1)})))),\numberthis\label{eq:lower_bound_Dtc_2}
    \end{align*}
    where \eqref{eq:lower_bound_Dtc_2_intermediate2}
    holds since $e^{-\frac{1-x}{x}}$ is increasing with respect to $x\in(0,1)$ and $F(z)\geq F(2^{\frac{1}{\alpha}})=1/2$ for any $z\geq 2^{\frac{1}{\alpha}}$.
    Combining the above two cases, we have
    \begin{equation*}
        \sum_{i:a_i^*=0} \Delta_i w_{t,i}\geq \Delta\min\qty{\frac{1}{2},\frac{e^{-1}}{2}\qty(1 - \exp\qty(-2\qty(1 - F\qty(2^{\frac{1}{\alpha}} + m^{1/(\alpha-1)}))))}.
    \end{equation*} 
    Here, we have
    \begin{equation*}
        1 - \exp\qty(-2\qty(1 - F\qty(2^{\frac{1}{\alpha}} + m^{1/(\alpha-1)})))
        =
        1-\exp\qty(-\frac{2}{\qty(2^{1/\alpha}+m^{1/(\alpha-1)})^\alpha}).
    \end{equation*}
    Note that, on the one hand,
    \begin{multline*}
        \Delta\qty(1-\exp\qty(-\frac{2}{\qty(2^{1/\alpha}+m^{1/(\alpha-1)})^\alpha}))\\
        \leq \Delta\qty(1-\exp\qty(-\frac{2}{(2^{1/\alpha}+1)^\alpha}))
        <\Delta\qty(1-e^{-2/3})<\frac{\Delta}{2}.
    \end{multline*}
    Therefore, by combining \eqref{eq:lower_bound_Dtc_1} and \eqref{eq:lower_bound_Dtc_2}, it suffices to give a lower bound for
    \begin{equation*}
        \sum_{i:a_i^*=0} \Delta_i w_{t,i}\geq
        \Delta\frac{e^{-1}}{2}\qty(1 - \exp\qty(-2\qty(1 - F\qty(2^{\frac{1}{\alpha}} + m^{1/(\alpha-1)})))).
    \end{equation*}
    On the other hand, there exists a constant $c_{s,2}(\Pdis)=\frac{2m^{\alpha/(\alpha-1)}}{2^{\alpha-1}(2+m^{\alpha/(\alpha-1)})+2}\in\qty(\frac{1}{3\cdot 2^{\alpha-2}+1},\frac{1}{2^{\alpha-2}})$ such that
    \begin{align*}
        &
        \Delta\qty(1-\exp\qty(-\frac{2}{\qty(2^{1/\alpha}+m^{1/(\alpha-1)})^\alpha}))\\
        &\geq
        \frac{2\Delta}{\qty(2^{1/\alpha}+m^{1/(\alpha-1)})^\alpha+2}\tag*{(by $1-e^{-x}\geq x/(1+x)$ for $x\geq 0$)}\\
        &\geq 
        \frac{2\Delta}{2^{\alpha-1}(2+m^{\alpha/(\alpha-1)})+2}
        \tag*{(by $\qty(\frac{x^{1/\alpha}+y^{1/\alpha}}{2})\leq \qty(\frac{x+y}{2})^{1/\alpha}$)}\\
        &=c_{s,2}(\Pdis)\frac{\Delta}{m^{\alpha/(\alpha-1)}},
    \end{align*}
    which concludes the proof.
\end{proof}

\subsection{Regret for the Optimal Action}\label{subsec:regret_optimal_action}
To apply the self-bounding technique, we need to express the regret of the optimal action in terms of the statistics of the remaining base-arms, i.e., those with $a^*_i=0$.
Compared to the analysis of single optimal arm in MAB problem studied by \citet{pmlr-v201-honda23a} and \citet{pmlr-v247-lee24a}, the analysis of optimal action is substantially more challenging, since the optimal action comprises multiple base-arms with inherent dependencies.

To tackle this challenge, we first establish an intermediate upper bound on the stability term associated with base-arms of the optimal action. 
This bound is then related to the remaining base-arms, thereby completing the argument.

It is worth noting that a closely related work by \citet{zhan2025follow}, focusing on \mbox{shape-$2$} Fr\'{e}chet distribution, develops a different line of analysis of the regret for the optimal action. 
In their method, the stability term associated with the optimal action is bounded by $O(\eta_t d \sum_{i:a^*_i=0}w_{t,i})$ with $\eta_t=c/\sqrt{t}$, which incurs an additional factor of $d$ in the second-order term of the regret bound. 
Furthermore, their method appears to be tailored to the specific form of Fr\'{e}chet distribution and does not readily extend to general 
Fr\'{e}chet-type distributions. 
Our method maintains the linear dependence on $d$ and instead incurs an additional factor of $m^{\alpha/(\alpha-1)}$. 
Moreover, our method actually leverages the common structure shared by Fr\'{e}chet-type distributions, allowing for further generalization in future studies.

\begin{lemma}\label{lem:optimal_action_add}
    Define $\delta_{\B}=\eta_t\max_{i:a^*_i=1}\hat{\ell}_{t,i}\sum_{i\in\B}e_i$ for any set $\B\subset[d]$. 
    Then, for any $\lambda\in\mathbb{R}^d$,  $\B\subsetneq \B^k\subset \qty{i:a^*_i=1}$ where 
    $\B^k\setminus \B=\{k\}$, on $F_t$ (resp. $D_t$) for $\dis=\Fdis$ (resp. $\dis=\Pdis$) it holds that
    \begin{multline*}
        \sum_{i \in \B}\phi_i(\eta_t\underline{\hat{L}}_t)-\sum_{i \in \B}\phi_i\qty(\eta_t\underline{\hat{L}}_t+\delta_\B)
        \\ \leq
        \sum_{i \in \B^k}\phi_i(\eta_t\underline{\hat{L}}_t)-\sum_{i \in \B^k}\phi_i\qty(\eta_t\underline{\hat{L}}_t+\delta_{\B^k})+\sum_{j:a_j^*=0}\frac{c_{s,3}(\dis)\eta_t}{(\eta_t\underline{\hat{L}}_{t,j})^\alpha}\max_{i:a^*_i=1}\hat{\ell}_{t,i},
    \end{multline*}
    where
    \begin{equation*}
        c_{s,3}(\dis)=
        \begin{cases}
            4/e^2, & \dis=\Fdis,\\
            \alpha, & \dis=\Pdis.
        \end{cases}
    \end{equation*}
\end{lemma}

\begin{proof}
    Firstly, by rearranging the terms we obtain
\begin{multline*}
    \qty(\sum_{i \in \B^k}\phi_i(\eta_t\underline{\hat{L}}_t)-\sum_{i \in \B^k}\phi_i(\eta_t\underline{\hat{L}}_t+\delta_{\B^k}))-\qty(\sum_{i \in \B}\phi_i(\eta_t\underline{\hat{L}}_t)-\sum_{i \in \B}\phi_i(\eta_t\underline{\hat{L}}_t+\delta_\B))\\
    =\qty(\phi_k(\eta_t\underline{\hat{L}}_t)-\phi_k(\eta_t\underline{\hat{L}}_t+\delta_{\B^k}))-
    \sum_{i \in \B}\qty(\phi_i(\eta_t\underline{\hat{L}}_t+\delta_{\B^k})-\phi_i(\eta_t\underline{\hat{L}}_t+\delta_\B)).
\end{multline*}
Then, we have
\begin{align*}
    \phi_k(\eta_t\underline{\hat{L}}_t)-\phi_k(\eta_t\underline{\hat{L}}_t+\delta_{\B^k})
    &=
    \sP\qty[\sigma_k(r-\eta_t\underline{\hat{L}}_t)\leq m]-
    \sP\qty[\sigma_k(r-(\eta_t\underline{\hat{L}}_t+\delta_{\B^k}))\leq m
    ]\\
    &=
    \sP\qty[
    \sigma_k(r-\eta_t\underline{\hat{L}}_t)\leq m,\,
    \sigma_k(r-(\eta_t\underline{\hat{L}}_t+\delta_{\B^k}))>m].
\end{align*}
Similarly, we have
\begin{equation*}
    \phi_i(\eta_t\underline{\hat{L}}_t+\delta_{\B^k})-\phi_i(\eta_t\underline{\hat{L}}_t+\delta_\B)
    =
    \sP\qty[
    \sigma_i(r-(\eta_t\underline{\hat{L}}_t+\delta_{\B^k}))\leq m,\,
    \sigma_i(r-(\eta_t\underline{\hat{L}}_t+\delta_\B))>m].
\end{equation*}
Then we have
\begin{align*}
    \lefteqn{\qty(\phi_k(\eta_t\underline{\hat{L}}_t)-\phi_k(\eta_t\underline{\hat{L}}_t+\delta_{\B^k}))-
    \sum_{i \in \B}\qty(\phi_i(\eta_t\underline{\hat{L}}_t+\delta_{\B^k} )-\phi_i(\eta_t\underline{\hat{L}}_t+\delta_\B))}\\
    &=
    \E\bigg[\ind\qty[\sigma_k(r-\eta_t\underline{\hat{L}}_t)\leq m,\,\sigma_k(r-(\eta_t\underline{\hat{L}}_t+\delta_{\B^k}))
    >m]\\
    &\hspace{3cm}-
    \sum_{i\in \B}\ind\qty[\sigma_i(r-(\eta_t\underline{\hat{L}}_t+\delta_{\B^k}))\leq m,\,\sigma_i(r-(\eta_t\underline{\hat{L}}_t+\delta_\B))>m]\bigg].\numberthis\label{eq:expect_diff}
\end{align*}
Here, note that $r-(\eta_t\underline{\hat{L}}_t+\delta_{\B^k})$ and $r-(\eta_t\underline{\hat{L}}_t+\delta_\B)$ are different only for one element.
Therefore, the event
\begin{equation*}
    \qty{\sigma_i(r-(\eta_t\underline{\hat{L}}_t+\delta_{\B^k}))\leq m,\,\sigma_i(r-(\eta_t\underline{\hat{L}}_t+\delta_\B))> m}
\end{equation*}
can occur for at most one $i\in \B$,
which we will denote (if exists) by $i_k\in \B$ and must satisfy
\begin{align}
&\sigma_{i_k}(r-\eta_t\underline{\hat{L}}_t-\delta_{\B^k})= m,\label{opt_sto3}\\
&\sigma_{i_k}(r-\eta_t\underline{\hat{L}}_t-\delta_\B)= m+1.\label{opt_sto4}\\
&\sigma_{k}(r-\eta_t\underline{\hat{L}}_t-\delta_{\B^k})\ge m+1,\label{opt_sto5}\\
&\sigma_{k}(r-\eta_t\underline{\hat{L}}_t-\delta_\B)\le m.\label{opt_sto6}
\end{align}

Accordingly, we can see that if
\begin{multline}\label{target_one}
\ind\qty[\sigma_k(r-\eta_t\underline{\hat{L}}_t)\leq m,\,\sigma_k(r-\eta_t\underline{\hat{L}}_t-\delta_{\B^k})>m]
\\
-\sum_{i\in \B}\ind\qty[\sigma_i(r-\eta_t\underline{\hat{L}}_t-\delta_{\B^k})\leq m,\,\sigma_i(r-\eta_t\underline{\hat{L}}_t-\delta_\B)> m]
\end{multline}
is negative, it must be $-1$, which can occur only when there exists $i_k\in \B$ satisfying
\begin{align}
&\sigma_k(r-\eta_t\underline{\hat{L}}_t)> m \;\mbox{ or }\;\sigma_k(r-\eta_t\underline{\hat{L}}_t-\delta_{\B^k})\le m,\label{opt_sto2}
\end{align}
and \eqref{opt_sto3}--\eqref{opt_sto6}.
Here, \eqref{opt_sto5} and \eqref{opt_sto2} implies
$\sigma_k(r-\eta_t\underline{\hat{L}}_t)> m$. 
Therefore, by $a_k^*=1$ there must exist
$j$ such that $a_j^*=0$ and
\begin{align}
\sigma_j(r-\eta_t\underline{\hat{L}}_t)\le m,
\label{opt_km}
\end{align}
which trivially implies
\begin{align}
r_{t,j}-\eta_t\hat{L}_{t,j}
\geq
\nu-\eta_t\max_{i:a_i^*=0}\hat{L}_{t,i}.
\label{opt_rj}
\end{align}
Since $\sigma_j(r-\eta_t\underline{\hat{L}}_t)\le \sigma_j(r-\eta_t\underline{\hat{L}}_t-\delta_{\B^k})$ by $j\notin \B^k$ from $a_j^*=0$,
we also see from \eqref{opt_sto3} and \eqref{opt_km} that
\begin{align}
\sigma_j(r-\eta_t\underline{\hat{L}}_t-\delta_{\B^k})< m.
\label{opt_km2}
\end{align}
By \eqref{opt_sto3} and \eqref{opt_sto5}
and
by \eqref{opt_sto4} and \eqref{opt_sto6},
we have
\begin{align}
(r-\eta_t\underline{\hat{L}}_t-\delta_{\B^k})_{i_k}>(r-\eta_t\underline{\hat{L}}_t-\delta_{\B^k})_{k}
\;\mbox{ and }\;
(r-\eta_t\underline{\hat{L}}_t-\delta_{\B})_{i_k}<(r-\eta_t\underline{\hat{L}}_t-\delta_{\B})_{k},
\n
\end{align}
respectively.
This is equivalent to
$0<(r-\eta_t\underline{\hat{L}}_t)_{i_k}-(r-\eta_t\underline{\hat{L}}_t)_{k}< \eta_t\max_{i:a^*_i=1}\hat{\ell}_{t,i}$
or equivalently
\begin{align}
r_k\in
\left[
r_{i_k}-\eta_t\underline{\hat{L}}_{t,i_k}+\eta_t\underline{\hat{L}}_{t,k}-\eta_t\max_{i:a^*_i=1}\hat{\ell}_{t,i}
,\;
r_{i_k}-\eta_t\underline{\hat{L}}_{t,i_k}+\eta_t\underline{\hat{L}}_{t,k}
\right].
\label{opt_sto7}
\end{align}

Finally, by \eqref{opt_sto3}, \eqref{opt_sto5} and \eqref{opt_km2} we have
\begin{align}
\sigma_{i_k}(r-\eta_t\underline{\hat{L}}_t+\delta_{\B^k},[d]\setminus\qty{j,k})= m-1.
\label{opt_sto9}
\end{align}

Note that \eqref{opt_sto9} can hold for at most one $i_k\in \B\setminus\{j\}$, and it is determined only by
$(r-\eta_t\underline{\hat{L}}_t+\delta_{\B^k})_{[d]\setminus \{j,k\}}$.
Therefore, from the above discussions we have
\begin{align}
-\eqref{target_one}
&\le
\sum_{j: a_j^*=0}
\ind\qty[\eqref{opt_rj},\,\mbox{$i_k$ s.t. $\eqref{opt_sto9}$ exists}, \eqref{opt_sto7}]
\n
\end{align}
and by taking the expectation with respect to $r_t$ we have
\begin{equation}\label{eq:target_one_as_prob}
\E[-\eqref{target_one}]
\le
\sum_{j: a_j^*=0}
\sP\left[\eqref{opt_rj}\right]\cdot
\E\Big[
\ind\qty[\mbox{$i_k$ s.t. $\eqref{opt_sto9}$ exists}]
\;\sP\qty[\eqref{opt_sto7}|\{r_{t,i}\}_{i\in [d]\setminus\{j,k\}}]
\Big].
\end{equation}
Here, on $F_t$ or $D_t$ we have
\begin{equation}\label{eq:first_pr}
    \sP\qty[\eqref{opt_rj}]=\sP\qty[r_{t,j}\geq \nu+\eta_t\underline{\hat{L}}_{t,j}]
    \leq\frac{1}{(\eta_t\underline{\hat{L}}_{t,j})^\alpha},
\end{equation}
where the last inequality holds since
\begin{equation*}
    \sP\qty[r_{t,j}\geq \eta_t\underline{\hat{L}}_{t,j}]=
    1-F\qty(\eta_t\underline{\hat{L}}_{t,j})=1-\exp\qty{-\frac{1}{(\eta_t\underline{\hat{L}}_{t,j})^\alpha}}\leq\frac{1}{(\eta_t\underline{\hat{L}}_{t,j})^\alpha}
\end{equation*}
for $\dis=\Fdis$, and
\begin{equation*}
    \sP\qty[r_{t,j}\geq 1+\eta_t\underline{\hat{L}}_{t,j}]=
    1-F\qty(1+\eta_t\underline{\hat{L}}_{t,j})=\frac{1}{(1+\eta_t\underline{\hat{L}}_{t,j})^\alpha}<\frac{1}{(\eta_t\underline{\hat{L}}_{t,j})^\alpha}
\end{equation*}
for $\dis=\Pdis$.

Then, by letting $f_{k}(x)$ be the pdf of $r_{k}$,
we have
\begin{align}\label{eq:second_pr}
\lefteqn{
\sP\qty[\eqref{opt_sto7}|\{r_{t,i}\}_{i\in [d]\setminus\{j,k\}}]
}\nn
&=
\sP\left[
\left.r_k\in
\left[
r_{i_k}-\eta_t\underline{\hat{L}}_{t,i_k}+\eta_t\underline{\hat{L}}_{t,k}-\eta_t\max_{i:a^*_i=1}\hat{\ell}_{t,i}
,\;
r_{i_k}-\eta_t\underline{\hat{L}}_{t,i_k}+\eta_t\underline{\hat{L}}_{t,k}
\right]
\right|
\{r_{t,i}\}_{i\in [d]\setminus\{j,k\}}
\right]
\nn
&\le
\eta_t\max_{i:a^*_i=1}\hat{\ell}_{t,i}
\max_{x\geq \nu}f_k(x)
\leq
\begin{cases}
    \frac{4}{e^2}\eta_t\max_{i:a^*_i=1}\hat{\ell}_{t,i}, & \dis=\Fdis,\\
    \alpha\eta_t\max_{i:a^*_i=1}\hat{\ell}_{t,i}, & \dis=\Pdis,
\end{cases}
\end{align}
where $f_k(x)=f(x)$ since the condition is independent of $r_k$.

The proof is completed by substituting \eqref{eq:first_pr} and \eqref{eq:second_pr} into \eqref{eq:target_one_as_prob} and combining the resulting expression with \eqref{eq:expect_diff}.
\end{proof}

By applying Lemma~\ref{lem:optimal_action_add} repeatedly, we can bound the stability term associated with base-arms of the optimal action $a^*$ as follows.
\begin{lemma}\label{lem:optimal_action_combine}
    On $F_t$ (resp. $D_t$) for $\dis=\Fdis$ (resp. $\dis=\Pdis$) we have
    \begin{multline*}
        \E\qty[\sum_{i:a_i^*=1}\hat{\ell}_{t,i}\qty(\phi_i(\eta_t\hat{L}_t)-\phi_i(\eta_t(\hat{L}_t+\hat{\ell}_t)))\m\hat{L}_t]
        \\\leq
        \E\qty[\max_{i:a_i^*=1}\hat{\ell}_{t,i}\sum_{i:a_i^*=1}\qty(\phi_i(\eta_t\hat{L}_t)-\phi_i(\eta_t(\hat{L}_t+\max_{i:a_i^*=1}\hat{\ell}_{t,i}\sum_{i:a_i^*=1}e_i)))\m\hat{L}_t]
        +\sum_{j:a_j^*=0}\frac{c_{s,4}(\dis)}{\underline{\hat{L}}_{t,j}},
    \end{multline*}
    where 
    \begin{equation*}
        c_{s,4}(\dis)=
        \begin{cases}
            4/e, & \dis=\Fdis,\\
            7.2\alpha, & \dis=\Pdis.
        \end{cases}
    \end{equation*}
\end{lemma}

\begin{proof}
    Firstly, by letting $\B^{+}=\qty{i\in[d]:a_i^*=1,\,\phi_i(\eta_t\hat{L}_t)-\phi_i(\eta_t(\hat{L}_t+\hat{\ell}_t))\geq 0}$, we have
    \begin{align*}
        \E\qty[\sum_{i:a_i^*=1}\hat{\ell}_{t,i}\qty(\phi_i(\eta_t\hat{L}_t)-\phi_i(\eta_t(\hat{L}_t+\hat{\ell}_t)))\m\hat{L}_t]
        &\leq
        \E\qty[\sum_{i\in \B^{+}}\hat{\ell}_{t,i}\qty(\phi_i(\eta_t\hat{L}_t)-\phi_i(\eta_t(\hat{L}_t+\hat{\ell}_t)))\m\hat{L}_t]\\
        &\hspace{-1cm}\leq
        \E\qty[\max_{i:a_i^*=1}\hat{\ell}_{t,i}\sum_{i\in \B^{+}}\qty(\phi_i(\eta_t\hat{L}_t)-\phi_i(\eta_t(\hat{L}_t+\hat{\ell}_t)))\m\hat{L}_t].
    \end{align*}
    Then, by Lemma~\ref{lem:phi_monotone}, since 
    $\hat{\ell}_{t,i}\leq\max_{i:a_i^*=1}\hat{\ell}_{t,i}$ for $i\in \B^{+}$, and 
    $\hat{\ell}_{t,i}\geq 0$ for $i\in[d]\setminus \B^{+}$ 
    we obtain that
    \begin{equation*}
        \sum_{i\in \B^{+}}\phi_i(\eta_t(\hat{L}_t+\hat{\ell}_t))\geq
        \sum_{i\in \B^{+}}\phi_i\qty(\eta_t(\hat{L}_t+\max_{i:a_i^*=1}\hat{\ell}_{t,i}\sum_{i\in \B^{+}}e_i)).
    \end{equation*}
    Therefore, we have
    \begin{multline}\label{eq:optimal_action_intermediate_positive}
        \E\qty[\sum_{i:a_i^*=1}\hat{\ell}_{t,i}\qty(\phi_i(\eta_t\hat{L}_t)-\phi_i(\eta_t(\hat{L}_t+\hat{\ell}_t)))\m\hat{L}_t]
        \\\leq
        \E\qty[\max_{i:a_i^*=1}\hat{\ell}_{t,i}\sum_{i\in \B^{+}}\qty(\phi_i(\eta_t\hat{L}_t)-\phi_i\qty(\eta_t(\hat{L}_t+\max_{i:a_i^*=1}\hat{\ell}_{t,i}\sum_{i\in \B^{+}}e_i)))\m\hat{L}_t].
    \end{multline}
    Then, by repeatedly applying Lemma~\ref{lem:optimal_action_add} to the right-hand side of the above inequality for at most $m$ steps, when $\dis=\Fdis$, we have
    \begin{multline*}
        \eqref{eq:optimal_action_intermediate_positive}
        \leq
        \E\qty[\max_{i:a_i^*=1}\hat{\ell}_{t,i}\sum_{i:a_i^*=1}\qty(\phi_i(\eta_t\hat{L}_t)-\phi_i(\eta_t(\hat{L}_t+\max_{i:a_i^*=1}\hat{\ell}_{t,i}\sum_{i:a_i^*=1}e_i)))\m\hat{L}_t]\\
        +\E\qty[\sum_{j:a_j^*=0}\frac{4m\eta_t}{e^2(\eta_t\underline{\hat{L}}_{t,j})^\alpha}\max_{i:a^*_i=1}\hat{\ell}_{t,i}\m\hat{L}_t].
    \end{multline*}
    Then, we have
    \begin{align*}
        \E\qty[\sum_{j:a_j^*=0}\frac{4m\eta_t}{e^2(\eta_t\underline{\hat{L}}_{t,j})^\alpha}\max_{i:a^*_i=1}\hat{\ell}_{t,i}\m\hat{L}_t]
        &\leq
        \E\qty[\sum_{j:a_j^*=0}\frac{4m\eta_t}{e^2(\eta_t\underline{\hat{L}}_{t,j})^\alpha}\frac{\max_{i:a^*_i=1}\ell_{t,i}}{\min_{i:a^*_i=1}w_{t,i}}\m\hat{L}_t]\\
        &=
        \sum_{j:a_j^*=0}\frac{4m\eta_t}{e(\eta_t\underline{\hat{L}}_{t,j})^\alpha}\tag{$\min_{i:a^*_i=1}w_{t,i}\geq\frac{1}{e}$  on $F_t$ by Lemma~\ref{lem:lower_bound_Ft}}\\
        &\leq
        \sum_{j:a_j^*=0}\frac{4}{e\underline{\hat{L}}_{t,j}},
    \end{align*}
    where the last inequality holds since 
    $(\eta_t\underline{\hat{L}}_{t,j})^{\alpha-1}\geq m$ for all $j$ such that $a_j^*=0$ on $F_t$.

    Similarly, when $\dis=\Pdis$, on $D_t$ we have
    \begin{multline*}
        \eqref{eq:optimal_action_intermediate_positive}
        \leq
        \E\qty[\max_{i:a_i^*=1}\hat{\ell}_{t,i}\sum_{i:a_i^*=1}\qty(\phi_i(\eta_t\hat{L}_t)-\phi_i(\eta_t(\hat{L}_t+\max_{i:a_i^*=1}\hat{\ell}_{t,i}\sum_{i:a_i^*=1}e_i)))\m\hat{L}_t]\\
        +\E\qty[\sum_{j:a_j^*=0}\frac{\alpha m\eta_t}{(\eta_t\underline{\hat{L}}_{t,j})^\alpha}\max_{i:a^*_i=1}\hat{\ell}_{t,i}\m\hat{L}_t].
    \end{multline*}
    Here, 
    \begin{equation*}
        \E\qty[\sum_{j:a_j^*=0}\frac{\alpha m\eta_t}{(\eta_t\underline{\hat{L}}_{t,j})^\alpha}\max_{i:a^*_i=1}\hat{\ell}_{t,i}\m\hat{L}_t]\leq \sum_{j:a_j^*=0}\frac{7.2\alpha}{\underline{\hat{L}}_{t,j}},
    \end{equation*}
    where $7.2$ is introduced by bounding $\min_{i:a^*_i=1}w_{t,i}\geq 1/0.14$ on $D_t$ from Lemma~\ref{lem:lower_bound_Dt}.
\end{proof}

The following lemma extends the findings in \citet{pmlr-v201-honda23a}, providing the bound of $\E[ \ind[\hat{\ell}_{t,i} > \zeta/\eta_t  ]\hat{\ell}_{t,i} | \hat{L}_t ]$ under Fr\'{e}chet and Pareto distributions for $m$-set semi-bandits. 
In addition, we provide a bound for $\E[ \ind[\hat{\ell}_{t,j} > \zeta/\eta_t  ]\hat{\ell}_{t,i} | \hat{L}_t ]$ for $i\neq j$, which follows from an analysis. 
Both results will be used in the proof of Lemma~\ref{lem:optimal_action_bound}.

\begin{lemma}[Extensions of partial results of Lemma~11 in \citet{pmlr-v201-honda23a}]\label{lem:optimal_action_bound_case2}
    Let $\wlow$ denote the lower bound of $\min_{i:a^*_i=1}w_{t,i}$.
    Then,
    on $F_t$ (resp. $D_t$) for $\dis=\Fdis$ (resp. $\dis=\Pdis$), for any $\hat{L}_t$, $\zeta \in (0,1)$ and $i\neq j$, it holds that
    \begin{equation*}
        \E\left[ \ind\qty[\hat{\ell}_{t,i} > \frac{\zeta}{\eta_t}  ]\hat{\ell}_{t,i} \m \hat{L}_t \right]\leq \frac{1}{1 - \wlow} (1 - \wlow)^{\tfrac{\zeta}{\eta_t}} \qty( \frac{\zeta}{\eta_t} + \frac{1}{\wlow} )
    \end{equation*}
    and
    \begin{equation*}
        \E\left[ \ind\qty[\hat{\ell}_{t,j} > \frac{\zeta}{\eta_t}  ]\hat{\ell}_{t,i} \m \hat{L}_t \right]\leq
        (1 - \wlow)^{\frac{\zeta}{\eta_t} -1},
    \end{equation*}
    where
    \begin{equation*}
        \wlow=
        \begin{cases}
            1/e, & \dis=\Fdis,\\
            0.14, & \dis=\Pdis.
        \end{cases}
    \end{equation*}
    Furthermore, when $\eta_t = cm^{\frac{1}{2}-\frac{1}{\alpha}}d^{\frac{1}{\alpha}-\frac{1}{2}}/\sqrt{t}$, we have
    \begin{equation*}
        \sum_{t=1}^{\infty}\frac{1}{1 - \wlow} (1 - \wlow)^{\tfrac{\zeta}{\eta_t}} \qty( \frac{\zeta}{\eta_t} + \frac{1}{\wlow} )
        \leq O\qty( c^2m^{1-\frac{2}{\alpha}}d^{\frac{2}{\alpha}-1}  ),
    \end{equation*}
    and
    \begin{equation*}
        \sum_{t=1}^{\infty}(1 - \wlow)^{\frac{\zeta}{\eta_t} -1}
        \leq O\qty( c^2m^{1-\frac{2}{\alpha}}d^{\frac{2}{\alpha}-1}  ).
    \end{equation*}
\end{lemma}

\begin{proof}
    The proof of the results for 
    $\E\left[ \ind\qty[\hat{\ell}_{t,i} > \frac{\zeta}{\eta_t}  ]\hat{\ell}_{t,i} \m \hat{L}_t \right]$ 
    has been provided in \citet{pmlr-v201-honda23a}. 
    Here, we provide the proof of the results for 
    $\E\left[ \ind\qty[\hat{\ell}_{t,j} > \frac{\zeta}{\eta_t}  ]\hat{\ell}_{t,i} \m \hat{L}_t \right]$.
    \begin{align*}
        \E\qty[\ind\qty[\hat{\ell}_{t,j}> \zeta/\eta_t]\hat{\ell}_{t,i}]
        &=\E\qty[\ind\qty[a_{t,j}=1,\,\hat{\ell}_{t,j}> \zeta/\eta_t]\ell_{t,i}]\\
        &\leq \E\qty[\ind\qty[a_{t,j}=1,\,\hat{\ell}_{t,j}> \zeta/\eta_t]]\tag*{(by $\ell_{t,i}\leq 1$)}\\
        &\leq \sP\qty[\widehat{w_{t,j}^{-1}}> \zeta/\eta_t]\tag*{(by $\hat{\ell}_{t,j}=\ell_{t,i}\widehat{w_{t,j}^{-1}}\leq \widehat{w_{t,j}^{-1}}$ when $a_{t,j}=1$)}\\
        &\leq (1-w_{t,j})^{\lfloor\zeta/\eta_t \rfloor}\\
        &\leq (1 - \wlow)^{\frac{\zeta}{\eta_t} -1},
    \end{align*}
    where $\wlow$ denotes the lower bound of $\min_{i:a^*_i=1}w_{t,i}$ on $F_t$ or $D_t$. By Lemma~\ref{lem:lower_bound_Ft} (resp.~Lemma~\ref{lem:lower_bound_Dt}), we can set $w^*(\Fdis)=1/e$ (resp.~$w^*(\Pdis)=0.14$).

    When $\eta_t=cm^{\frac{1}{2}-\frac{1}{\alpha}}d^{\frac{1}{\alpha}-\frac{1}{2}}/\sqrt{t}$, we have
    \begin{align*}
        \sum_{t=1}^{\infty}\E\qty[\ind\qty[\hat{\ell}_{t,j}> \zeta/\eta_t]\hat{\ell}_{t,i}]&\leq \sum_{t=1}^\infty (1 - \wlow)^{\frac{\zeta}{\eta_t} -1}\\
        &= \frac{1}{1 - \wlow}\sum_{t=1}^\infty (1 - \wlow)^{\frac{\zeta}{\eta_t}}\\
        &\leq \frac{1}{1 - \wlow}\int_0^\infty (1 - \wlow)^{\zeta\sqrt{t}/cm^{\frac{1}{2}-\frac{1}{\alpha}}d^{\frac{1}{\alpha}-\frac{1}{2}}} \dd t\\
        &= \frac{2}{1 - \wlow}\qty(\frac{cm^{\frac{1}{2}-\frac{1}{\alpha}}d^{\frac{1}{\alpha}-\frac{1}{2}}}{\zeta\log\frac{1}{1 - \wlow}})^2\\
        &=O\qty(c^2m^{1-\frac{2}{\alpha}}d^{\frac{2}{\alpha}-1}),
    \end{align*}
    which concludes the proof.
\end{proof}

The following lemma provides an upper bound on the regret of the optimal action under Fr\'{e}chet or Pareto distributions with shape $\alpha>1$.

\begin{lemma}\label{lem:optimal_action_bound}
    Let $\zeta\in(0,1)$ and $\alpha>1$. On $F_t$ (resp. $D_t$) for $\dis=\Fdis$ (resp. $\dis=\Pdis$) we have
    \begin{align*}
        \lefteqn{\E\qty[\sum_{i:a_i^*=1}\hat{\ell}_{t,i}\qty(\phi_i(\eta_t\hat{L}_t)-\phi_i(\eta_t(\hat{L}_t+\hat{\ell}_t)))\m\hat{L}_t]\leq\qty(\frac{2\alpha }{\wlow(1-\zeta)^{\alpha+1}}+c_{s,4}(\dis))\sum_{i:a_i^*=0}\frac{1}{\underline{\hat{L}}_{t,i}}}\\
        &\hspace{3.2cm}+\frac{m}{1 - \wlow} (1 - \wlow)^{\tfrac{\zeta}{\eta_t}} \qty( \frac{\zeta}{\eta_t} + \frac{1}{\wlow} )+
        m^2(1 - \wlow)^{\frac{\zeta}{\eta_t} -1},
    \end{align*}
    where $c_{s,4}(\dis)$ and $\wlow$ are defined in Lemmas~\ref{lem:optimal_action_combine} and \ref{lem:optimal_action_bound_case2}, respectively.
\end{lemma}

\begin{proof}
    Recall \eqref{eq:key_property}, which shows that, on the event $F_t$ or $D_t$, we have 
    $\eta_t\underline{\hat{L}}_{t,i}\leq 0$ 
    for all $i$ with $a_i^*=1$ and 
    $\eta_t\underline{\hat{L}}_{t,i}\geq m^{1/(\alpha-1)}$ 
    for all $i$ with $a_i^*=0$. 
    We consider the cases 
    (a) $\max_{i:a^*_i=1}\widehat{w_{t,i}^{-1}}\leq\zeta/\eta_t$ 
    and 
    (b) $\max_{i:a^*_i=1}\widehat{w_{t,i}^{-1}}>\zeta/\eta_t$ separately.
    
    (a) Let us consider the first case, which implies 
    $\max_{i:a^*_i=1}\hat{\ell}_{t,i}\leq \max_{i:a^*_i=1}\widehat{w_{t,i}^{-1}}\leq\zeta/\eta_t$.
    Let $b_1,b_2,\dots,b_m$ be an arbitrary ordering of the set of the base-arms of the optimal action $a^*$.
    Then, denote 
    $\hat{L}\pn{k}_t=\hat{L}_t+\max_{i:a^*_i=1}\hat{\ell}_{t,i}\cdot\sum_{j=1}^k e_{b_j}$ 
    and 
    $\hat{L}\pn{0}_t=\hat{L}_t$. 
    Note that $\sum_{i=1}^d\phi_{i}(\lambda)=m$ for any $\lambda\in\mathbb{R}^d$. 
    By Lemma~\ref{lem:optimal_action_combine}, we have
    \begin{align*}
        \lefteqn{\E\qty[\ind\qty[\max_{i:a^*_i=1}\hat{\ell}_{t,i}\leq\zeta/\eta_t]\sum_{i:a_i^*=1}\hat{\ell}_{t,i}\qty(\phi_i(\eta_t\hat{L}_t)-\phi_i(\eta_t(\hat{L}_t+\hat{\ell}_t)))\m\hat{L}_t]}\\
        &\leq
        \E\qty[\ind\qty[\max_{i:a^*_i=1}\hat{\ell}_{t,i}\leq\zeta/\eta_t]\max_{i:a_i^*=1}\hat{\ell}_{t,i}\sum_{i:a_i^*=1}\qty(\phi_i(\eta_t\hat{L}_t)-\phi_i(\eta_t(\hat{L}_t+\max_{i:a_i^*=1}\hat{\ell}_{t,i}\sum_{i:a_i^*=1}e_i)))\m\hat{L}_t]\\
        &\hspace{11cm}+\sum_{i:a_i^*=0}\frac{c_{s,4}(\dis)}{\underline{\hat{L}}_{t,i}}
        \displaybreak[0]\\
        &=
        \E\qty[\ind\qty[\max_{i:a^*_i=1}\hat{\ell}_{t,i}\leq\zeta/\eta_t]\max_{i:a_i^*=1}\hat{\ell}_{t,i}\sum_{i:a_i^*=0}\qty(\phi_i(\eta_t(\hat{L}_t+\max_{i:a_i^*=1}\hat{\ell}_{t,i}\sum_{i:a_i^*=1}e_i))-\phi_i(\eta_t\hat{L}_t))\m\hat{L}_t]\\
        &\hspace{11cm}+
        \sum_{i:a_i^*=0}\frac{c_{s,4}(\dis)}{\underline{\hat{L}}_{t,i}}
        \displaybreak[0]\\
        &=
        \E\qty[\ind\qty[\max_{i:a^*_i=1}\hat{\ell}_{t,i}\leq\zeta/\eta_t]\max_{i:a_i^*=1}\hat{\ell}_{t,i}\sum_{i:a_i^*=0}\sum_{j=1}^m\qty(\phi_{i}(\eta_t\hat{L}\pn{j}_t)-\phi_{i}(\eta_t\hat{L}\pn{j-1}_t))\m\hat{L}_t]\\
        &\hspace{11cm}+\sum_{i:a_i^*=0}\frac{c_{s,4}(\dis)}{\underline{\hat{L}}_{t,i}}.\numberthis\label{eq:optimal_action_frechet_transfer}
    \end{align*}
    Note that 
    $\hat{L}_{t,i'}-\hat{L}_{t,i}=\underline{\hat{L}}_{t,i'}-\underline{\hat{L}}_{t,i}$ 
    for $i\neq i'$. 
    For any $i$ with $a^*_i=1$ and $i'$ with $a^*_{i'}=0$, we have
    \begin{equation*}
        \hat{L}_{t,i}-\hat{L}_{t,i'}
        =
        \underline{\hat{L}}_{t,i}-\underline{\hat{L}}_{t,i'}\geq 
        \min_{i':a^*_{i'}=0}\underline{\hat{L}}_{t,i'}-\max_{i:a^*_i=1}\underline{\hat{L}}_{t,i}\geq m^{1/(\alpha-1)}/\eta_t.
    \end{equation*}
    Moreover, since 
    $\max_{i:a^*_i=1}\widehat{w_{t,i}^{-1}}\leq\zeta/\eta_t$, we have 
    $\max_{i:a^*_i=1}\hat{\ell}_{t,i}\leq \zeta/\eta_t$. Therefore, for any 
    $x\leq \max_{i:a^*_i=1}\hat{\ell}_{t,i}$ and $j\in[m]$ we have
    \begin{align*}
        (\hat{L}_t\pn{j-1}+xe_{b_j})_{i'}- (\hat{L}_t\pn{j-1}+xe_{b_j})_{i}
        &=
        \hat{L}_{t,i'}- (\hat{L}_t+\max_{i:a^*_i=1}\hat{\ell}_{t,i}\sum_{i\in[j-1]}e_{u,i}+xe_{b_j})_{i}\\
        &\geq
        \hat{L}_{t,i'}- (\hat{L}_t+\max_{i:a^*_i=1}\hat{\ell}_{t,i}\sum_{i\in[j]}e_{u,i})_{i}\\
        &=(\hat{L}_{t,i'}-\hat{L}_{t,i})-\max_{i:a^*_i=1}\hat{\ell}_{t,i}\\
        &\geq \qty(\eta_t\underline{\hat{L}}_{t,i'}-\zeta m^{1/(\alpha-1)})/\eta_t\\
        &\geq (1-\zeta)(\eta_t\underline{\hat{L}}_{t,i'})/\eta_t> 0.\numberthis\label{eq:zeta_bound_frechet}
    \end{align*}
    Here, since $i$ (resp. $i'$) is arbitrary that satisfies $a_{t,i}=1$ (resp. $a_{t,i'}=0$), one can see that the components of $(\hat{L}_t\pn{j-1}+xe_{b_j})$ corresponding to the $m$ base-arms with $a_i^*=1$ remain the top-$m$ smallest among all base-arms. 
    It follows that for any $x\leq \zeta/\eta_t$,
    \begin{align*}
        \underline{\hat{L}_t\pn{j-1}+xe_{b_j}}_{i'}
        &=\min_{i:a^*_i=1}\qty{(\hat{L}_t\pn{j-1}+xe_{b_j})_{i'}- (\hat{L}_t\pn{j-1}+xe_{b_j})_{i}}\\
        &\geq (1-\zeta)(\eta_t\underline{\hat{L}}_{t,i'})/\eta_t> 0.
    \end{align*}
    In addition, for any $x\leq \zeta/\eta_t$, $i$ with $a_i^*=0$, and any $z\geq \nu$ it holds that
    \begin{align*}
        f\qty(z+\eta_t\underline{\hat{L}_t\pn{j-1}+xe_{b_j}}_i)
        &=\frac{\alpha}{(z+\eta_t\underline{\hat{L}_t\pn{j-1}+xe_{b_j}}_i)^{\alpha+1}}F\qty(z+\eta_t\underline{\hat{L}_t\pn{j-1}+xe_{b_j}}_i)\\
        &\leq \frac{\alpha}{(z+\eta_t\underline{\hat{L}_t\pn{j-1}+xe_{b_j}}_i)^{\alpha+1}}.\numberthis\label{eq:transform_pdf}
    \end{align*}
    Now we bound $\frac{\mathrm{d}}{\mathrm{d} x}\phi_{i}(\eta_t(\hat{L}_t\pn{j-1}+xe_{b_j}))$ for $x\leq \zeta/\eta_t$. 
    By Lemma~\ref{lem:tool_relation} we have
    \begin{align*}
        \lefteqn{\frac{\mathrm{d}}{\mathrm{d} x}\phi_{i}(\eta_t(\hat{L}_t\pn{j-1}+xe_{b_j}))}\\
        &=\frac{\mathrm{d}}{\mathrm{d} x}\phi_{i}\qty(\eta_t(\hat{L}_t\pn{j-1}+xe_{b_j});\dis,m-1,[d]/\{b_j\})
        +
        \frac{\mathrm{d}}{\mathrm{d} x}\phi_{i,m,b_j}\qty(\eta_t(\hat{L}_t\pn{j-1}+xe_{b_j});\dis,[d])\\
        &=
        \frac{\mathrm{d}}{\mathrm{d} x}\phi_{i,m,b_j}\qty(\eta_t(\hat{L}_t\pn{j-1}+xe_{b_j});\dis,[d]),
    \end{align*}
    where the last equality holds since $x$ only affects the $b_j$-th component of $\hat{L}_t\pn{j-1}+xe_{b_j}$. 
    Recall the definition of $\phi_{i,m,b_j}(\cdot)$ given in \eqref{eq:prob_theta_j}, and 
    note that $(\hat{L}_t\pn{j-1}+xe_{b_j})_{i'}=\hat{L}_{t,i'}\pn{j-1}+\ind[i'=b_j]x$ holds for any $i'\in[d]$. 
    By letting $b_{j'}$ with $a_{b_{j'}}^*=1$ such that $\sigma_{b_{j'}}(-(\hat{L}_t\pn{j-1}+xe_{b_j}))=m$, for any $i$ with $a_i^*=0$ we have
    \begin{align*}
        \lefteqn{\frac{\mathrm{d}}{\mathrm{d} x}\phi_{i,m,b_j}\qty(\eta_t(\hat{L}_t\pn{j-1}+xe_{b_j});\dis,[d])}\\
        &=
        \frac{\mathrm{d}}{\mathrm{d} x}\int_{\nu-\eta_t(\hat{L}_{t,b_{j'}}\pn{j-1}+\ind[b_{j'}=b_j]x)}^\infty F\qty(z+\eta_t(\hat{L}_{t,b_j}\pn{j-1}+x))f\qty(z+\eta_t\hat{L}_{t,i}\pn{j-1})
        \\
        &\hspace{3cm}
        \sP_{\{r_k\}_{k\in[d]\setminus\{i,b_j\}} \sim \mathcal{D},\,r_i=z+\eta_t\hat{L}_{t,i}\pn{j-1}}\qty[\sigma_i\qty(r-\eta_t\hat{L}_t\pn{j-1},[d]\setminus\{b_j\})=m]
        \dd z\\
        &=
        \eta_t\int_{\nu}^\infty f\qty(z+\eta_t\underline{\hat{L}_t\pn{j-1}+xe_{b_j}}_{b_j})f\qty(z+\eta_t\underline{\hat{L}_t\pn{j-1}+xe_{b_j}}_i)\\
        &\hspace{3cm}
        \sP_{\{r_k\}_{k\in[d]\setminus\{i,b_j\}} \sim \mathcal{D},\,r_i=z+\eta_t\underline{\hat{L}_t\pn{j-1}+xe_{b_j}}_i}
        \qty[\sigma_i\qty(r-\eta_t\hat{L}_t\pn{j-1},[d]\setminus\{b_j\})=m]
        \dd z\\
        &\hspace{0.5cm}+
        \ind\qty[b_{j'}=b_j]\eta_t F(\nu)f\qty(\eta_t\underline{\hat{L}_t\pn{j-1}+xe_{b_j}}_i)\\
        &\hspace{3cm}
        \sP_{\{r_k\}_{k\in[d]\setminus\{i,b_j\}} \sim \mathcal{D},\,r_i=\eta_t\underline{\hat{L}_t\pn{j-1}+xe_{b_j}}_i}
        \qty[\sigma_i\qty(r-\eta_t\hat{L}_t\pn{j-1},[d]\setminus\{b_j\})=m]
        \tag*{$\qty(\text{by }\underline{\hat{L}_t\pn{j-1}+xe_{b_j}}_i=(\hat{L}_t\pn{j-1}+xe_{b_j})_i-(\hat{L}_t\pn{j-1}+xe_{b_j})_{b_{j'}})$}
        \displaybreak[0]\\
        &\leq
        \eta_t\int_{\nu-\eta_t(\hat{L}_{t,b_{j'}}\pn{j-1}+\ind[b_{j'}=b_j]x)}^\infty f\qty(z+\eta_t\underline{\hat{L}_t\pn{j-1}+xe_{b_j}}_{b_j})f\qty(z+\eta_t\underline{\hat{L}_t\pn{j-1}+xe_{b_j}}_i)\dd z
        \tag{by $\sP[\cdot]\leq 1$ and $F(\nu)=0$}
        \displaybreak[0]\\
        &\leq
        \eta_t\int_{\nu}^\infty f\qty(z+\eta_t\underline{\hat{L}_t\pn{j-1}+xe_{b_j}}_{b_j})\frac{\alpha}{(z+\eta_t\underline{\hat{L}_t\pn{j-1}+xe_{b_j}}_i)^{\alpha+1}}\dd z
        \tag{by \eqref{eq:transform_pdf}}\\
        &\leq
        \frac{\alpha\eta_t}{(\eta_t\underline{\hat{L}_t\pn{j-1}+xe_{b_j}}_i)^{\alpha+1}}\int_{\nu}^\infty f\qty(z)\dd z\\
        &\leq
        \frac{\alpha\eta_t}{(1-\zeta)^{\alpha+1}(\eta_t\underline{\hat{L}}_{t,i})^{\alpha+1}}\tag*{(by \eqref{eq:zeta_bound_frechet} and $\int_{\nu}^\infty f\qty(z)\dd z=1$)}\\
        &=\frac{\alpha}{m(1-\zeta)^{\alpha+1}\underline{\hat{L}}_{t,i}}.
    \end{align*}
    Here, the last equality follows from the property \eqref{eq:key_property} of $F_t$ or $D_t$, which implies that
    \begin{equation*}
        (\eta_t\underline{\hat{L}}_{t,i})^\alpha
        \geq
        m^{\alpha/(\alpha-1)}\geq m,\,\forall i:a^*_i=0.
    \end{equation*}
    Then, we have
    \begin{align*}
        \phi_{i}(\eta_t\hat{L}\pn{j}_t)-\phi_{i}(\eta_t\hat{L}\pn{j-1}_t)
        &=
        \int_0^{\max_{i:a^*_i=1}\hat{\ell}_{t,i}}\frac{\mathrm{d}}{\mathrm{d} x}\phi_{i}(\eta_t(\hat{L}_t\pn{j-1}+xe_{b_j}))\dd x\\
        &\leq 
        \max_{i:a^*_i=1}\hat{\ell}_{t,i}\frac{\alpha}{m(1-\zeta)^{\alpha+1}\underline{\hat{L}}_{t,i}}.\numberthis\label{eq:optimal_action_frechet_mL}
    \end{align*}
    By combining \eqref{eq:optimal_action_frechet_transfer} and \eqref{eq:optimal_action_frechet_mL}, we obtain
    \begin{align*}
        \lefteqn{\E\qty[\ind\qty[\max_{i:a^*_i=1}\hat{\ell}_{t,i}\leq\zeta/\eta_t]\sum_{i:a_i^*=1}\hat{\ell}_{t,i}\qty(\phi_i(\eta_t\hat{L}_t)-\phi_i(\eta_t(\hat{L}_t+\hat{\ell}_t)))\m\hat{L}_t]}\displaybreak[0]\\
        &\leq
        \E\qty[\ind\qty[\max_{i:a^*_i=1}\hat{\ell}_{t,i}\leq\zeta/\eta_t]\max_{i:a_i^*=1}\hat{\ell}_{t,i}\sum_{i:a_i^*=0}\sum_{j=1}^m\qty(\phi_{i}(\eta_t\hat{L}\pn{j}_t)-\phi_{i}(\eta_t\hat{L}\pn{j-1}_t))\m\hat{L}_t]\\
        &\hspace{10.6cm}+\sum_{i:a_i^*=0}\frac{c_{s,4}(\dis)}{\underline{\hat{L}}_{t,i}}
        \displaybreak[0]\\
        &\leq
        \E\qty[\ind\qty[\max_{i:a^*_i=1}\hat{\ell}_{t,i}\leq\zeta/\eta_t]\qty(\max_{i:a_i^*=1}\hat{\ell}_{t,i})^2\sum_{i:a_i^*=0}\frac{\alpha}{(1-\zeta)^{\alpha+1}\underline{\hat{L}}_{t,i}}\m\hat{L}_t]
        +\sum_{i:a_i^*=0}\frac{c_{s,4}(\dis)}{\underline{\hat{L}}_{t,i}}
        \displaybreak[0]\\
        &\leq
        \E\qty[\max_{i:a_i^*=1}\hat{\ell}_{t,i}^2\sum_{i:a_i^*=0}\frac{\alpha}{(1-\zeta)^{\alpha+1}\underline{\hat{L}}_{t,i}}\m\hat{L}_t]
        +\sum_{i:a_i^*=0}\frac{c_{s,4}(\dis)}{\underline{\hat{L}}_{t,i}}
        \displaybreak[0]\\
        &\leq
        \E\qty[\max_{i:a_i^*=1}\frac{2\ell_{t,i}^2}{w_{t,i}}\sum_{i:a_i^*=0}\frac{\alpha}{(1-\zeta)^{\alpha+1}\underline{\hat{L}}_{t,i}}\m\hat{L}_t]
        +\sum_{i:a_i^*=0}\frac{c_{s,4}(\dis)}{\underline{\hat{L}}_{t,i}}\tag*{(by \eqref{eq:square_expectation})}
        \displaybreak[0]\\
        &\leq
        \qty(\frac{2\alpha}{\wlow(1-\zeta)^{\alpha+1}}+c_{s,4}(\dis))\sum_{i:a_i^*=0}\frac{1}{\underline{\hat{L}}_{t,i}}.
        \numberthis\label{eq:optimal_action_frechet_bound_case1}
    \end{align*}

    (b) Next, we consider the second case where $\max_{i:a^*_i=1}\widehat{w_{t,i}^{-1}}>\zeta/\eta_t$. 
    By Lemma~\ref{lem:optimal_action_bound_case2} we have
    \begin{align*}
        \lefteqn{\E\qty[\ind\qty[\max_{i:a^*_i=1}\hat{\ell}_{t,i}\geq\zeta/\eta_t]\sum_{i:a_i^*=1}\hat{\ell}_{t,i}\qty(\phi_i(\eta_t\hat{L}_t)-\phi_i(\eta_t(\hat{L}_t+\hat{\ell}_t)))\m\hat{L}_t]}\\
        &\leq
        \E\qty[\ind\qty[\max_{i:a^*_i=1}\hat{\ell}_{t,i}\geq\zeta/\eta_t]\sum_{i:a_i^*=1}\hat{\ell}_{t,i}\m\hat{L}_t]\\
        &\leq
        \sum_{i:a_i^*=1}\sum_{j:a_j^*=1}\E\qty[\ind\qty[\hat{\ell}_{t,i}\geq\zeta/\eta_t]\hat{\ell}_{t,j}\m\hat{L}_t]\\
        &=
        \sum_{i:a_i^*=1}\E\qty[\ind\qty[\hat{\ell}_{t,i}\geq\zeta/\eta_t]\hat{\ell}_{t,i}\m\hat{L}_t]+\sum_{i:a_i^*=1}\sum_{j:a_j^*=1,j\neq i}\E\qty[\ind\qty[\hat{\ell}_{t,i}\geq\zeta/\eta_t]\hat{\ell}_{t,i}\m\hat{L}_t]\\
        &\leq
        \frac{m}{1 - \wlow} (1 - \wlow)^{\tfrac{\zeta}{\eta_t}} \qty( \frac{\zeta}{\eta_t} + \frac{1}{\wlow} )+
        m^2(1 - \wlow)^{\frac{\zeta}{\eta_t} -1}.\numberthis\label{eq:optimal_action_frechet_bound_case2}
    \end{align*}
    We obtain the lemma by combining 
    \eqref{eq:optimal_action_frechet_bound_case1} 
    and 
    \eqref{eq:optimal_action_frechet_bound_case2}.
\end{proof}

\subsection{Proofs of Theorems~\ref{thm:sto_bound} and \ref{thm:sto_bound_alpha_not2}}\label{subsec:proofs_sto_bounds}

In this section we provide proofs of the regret bounds for stochastic setting with general shape parameters $\alpha>1$ in Theorem~\ref{thm:sto_bound} ($\alpha=2$) and Theorem~\ref{thm:sto_bound_alpha_not2} ($\alpha<2$ and $\alpha>2$). 
We divide the proofs into two cases $\alpha\geq 2$ and $\alpha<2$, both of which follow essentially the same argument. The result of Theorem~\ref{thm:sto_bound} ($\alpha=2$) is recovered as a special case of the first case. 

In the following, for simplicity, let 
$\ma=m^{\frac{1}{2}-\frac{1}{\alpha}}$ 
and
$\da=d^{\frac{1}{\alpha}-\frac{1}{2}}$
so that
$\eta_t=c\ma\da/\sqrt{t}$.
\subsubsection{Analysis for Shape \texorpdfstring{$\alpha\geq 2$}{α ≥ 2}}

By aggregating the results obtained so far, it follows that the regret satisfies
\begin{align*}
    &\R(T) \leq \sum_{t=1}^{T} \E \qty[ \left\langle \hat{\ell}_t, \phi\qty(\eta_t \hat{L}_t ; \dis) - \phi\qty(\eta_t (\hat{L}_t + \hat{\ell}_t) ; \dis) \right\rangle ]  \\
    &\hspace{1.7cm}+ \sum_{t=1}^T \qty(\frac{1}{\eta_{t+1}}-\frac{1}{\eta_t}) \E_{r_{t+1}\sim\dis}\qty[\left\langle r_{t+1}, a_{t+1} - a^*\right\rangle]
    +\frac{\E_{r_1 \sim \dis} \qty[ a_1^\top r_1 ]}{\eta_1} \tag{by Lemma~\ref{lem:regret_decomposition}}
    \displaybreak[0]\\
    &= \sum_{t=1}^{T} \E\qty[\E \qty[ \left\langle \hat{\ell}_t, \phi\qty(\eta_t \hat{L}_t ; \dis) - \phi\qty(\eta_t (\hat{L}_t + \hat{\ell}_t) ; \dis) \right\rangle
    +
    \qty(\frac{1}{\eta_{t+1}}-\frac{1}{\eta_t})\left\langle r_{t+1}, a_{t+1} - a^*\right\rangle\m\hat{L}_t]] \\
    &\hspace{3cm}+
    \begin{cases}
        \frac{\qty(\frac{\alpha}{\alpha-1}(m-1)^{1-\frac{1}{\alpha}}+\Gamma\qty(1-\frac{1}{\alpha}))\qty(d+1)^{\frac{1}{\alpha}}+m}{c\ma\da}, & \dis=\Fdis,\\
        \frac{\qty(\frac{\alpha}{\alpha-1}(m-1)^{1-\frac{1}{\alpha}}+\Gamma\qty(1-\frac{1}{\alpha}))\qty(d+1)^{\frac{1}{\alpha}}}{c\ma\da}, & \dis=\Pdis,
    \end{cases}
    \tag{by Lemma~\ref{lem:regret_decomposition}}
    \displaybreak[0]\\
    &\leq \sum_{t=1}^{T} \E\qty[\E \qty[ \left\langle \hat{\ell}_t, \phi\qty(\eta_t \hat{L}_t ; \dis) - \phi\qty(\eta_t (\hat{L}_t + \hat{\ell}_t) ; \dis) \right\rangle
    +
    \frac{\left\langle r_{t+1}, a_{t+1} - a^*\right\rangle}{2c\ma\da\sqrt{t}}\m\hat{L}_t]]\\
    &\hspace{11cm}+C_2(\dis)\sqrt{md},
    \numberthis\label{eq:stochastic_regret_decomposition_pareto}
\end{align*}
where the last inequality follows from the fact that
\begin{equation*}
    \frac{1}{\eta_{t+1}}-\frac{1}{\eta_t}
    =\frac{1}{c\ma\da}\qty(\sqrt{t+1}-\sqrt{t})
    =\frac{1}{c\ma\da}\qty(\sqrt{1+1/t}-1)
    \leq \frac{1}{2c\ma\da\sqrt{t}},
\end{equation*}
and we defined
\begin{equation*}
    C_2(\dis)\coloneqq
    \begin{cases}
        \frac{1}{c}\qty(\frac{\alpha}{\alpha-1}(1-\frac{1}{m})^{1-\frac{1}{\alpha}}+\Gamma\qty(1-\frac{1}{\alpha})m^{\frac{1}{\alpha}-1})\qty(1+\frac{1}{d})^{\frac{1}{\alpha}}+\frac{1}{c}\qty(\frac{m}{d})^{\frac{1}{\alpha}}, & \dis=\Fdis,\\
        \frac{1}{c}\qty(\frac{\alpha}{\alpha-1}(1-\frac{1}{m})^{1-\frac{1}{\alpha}}+\Gamma\qty(1-\frac{1}{\alpha})m^{\frac{1}{\alpha}-1})\qty(1+\frac{1}{d})^{\frac{1}{\alpha}}, & \dis=\Pdis.
    \end{cases}
\end{equation*}

Now we consider the first term of \eqref{eq:stochastic_regret_decomposition_pareto} on $F_t$ (resp. $D_t$) for $\dis=\Fdis$ (resp. $\dis=\Pdis$), where $\eta_t\underline{\hat{L}}_{t,i}\geq m^{1/(\alpha-1)}$ for any $i$ with $a^*_i=0$. 
For $\alpha\geq 2$ we have
\begin{align*}
    \lefteqn{\E \qty[ \left\langle \hat{\ell}_t, \phi\qty(\eta_t \hat{L}_t ; \dis) - \phi\qty(\eta_t (\hat{L}_t + \hat{\ell}_t) ; \dis) \right\rangle\m\hat{L}_t]}\\
    &=
    \sum_{i:a^*_i=0}\E \qty[ \hat{\ell}_{t,i}\qty(\phi_i\qty(\eta_t \hat{L}_t ; \dis) - \phi_i\qty(\eta_t (\hat{L}_t + \hat{\ell}_t) ; \dis)) \m\hat{L}_t]\\
    &\hspace{3.5cm}+
    \sum_{i:a^*_i=1}\E \qty[ \hat{\ell}_{t,i}\qty(\phi_i\qty(\eta_t \hat{L}_t ; \dis) - \phi_i\qty(\eta_t (\hat{L}_t + \hat{\ell}_t) ; \dis)) \m\hat{L}_t]\numberthis\label{eq:to_apply_L_bound}
    \displaybreak[0]\\
    &\leq
    2(\alpha+1)\sum_{i:a^*_i=0} \frac{1}{\underline{\hat{L}}_{t,i}}
    +
    \qty(\frac{2\alpha }{\wlow(1-\zeta)^{\alpha+1}}+c_{s,4}(\dis))\sum_{i:a_i^*=0}\frac{1}{\underline{\hat{L}}_{t,i}}\\
    &\hspace{1cm}+\frac{m}{1 - \wlow} (1 - \wlow)^{\tfrac{\zeta}{\eta_t}} \qty( \frac{\zeta}{\eta_t} + \frac{1}{\wlow} )+
        m^2(1 - \wlow)^{\frac{\zeta}{\eta_t} -1},\numberthis\label{eq:stability_bound_alpha_ge_2}
\end{align*}
where the inequality follows from application of Lemmas~\ref{lem:adv_bound_term} and \ref{lem:optimal_action_bound} to the first and second term of \eqref{eq:to_apply_L_bound}, respectively. 
Here, $\zeta\in(0,1)$ is an arbitrary parameter introduced in Lemma~\ref{lem:optimal_action_bound} and we defined
\begin{equation*}
        c_{s,4}(\dis)=
        \begin{cases}
            4/e, & \dis=\Fdis,\\
            7.2\alpha, & \dis=\Pdis,
        \end{cases}
        \quad\text{and}\quad
        \wlow=
        \begin{cases}
            1/e, & \dis=\Fdis,\\
            0.14, & \dis=\Pdis.
        \end{cases}
\end{equation*}
Then, since $\underline{\hat{L}}_{t,i}>0$ for any $i$ with $a^*_i=0$ on $F_t$ or $D_t$, by applying the second part of Lemma~\ref{lem:penalty_adv_bound} we have
\begin{equation}\label{eq:penalty_bound_alpha_ge2}
    \E\qty[\frac{\left\langle r_{t+1}, a_{t+1} - a^*\right\rangle}{2c\ma\da\sqrt{t}}\m\hat{L}_t]
    \leq
    \frac{1}{2c\ma\da\sqrt{t}}\frac{\alpha}{\alpha-1}\frac{1}{(\eta_t\underline{\hat{L}}_{t,i})^{\alpha-1}}.
\end{equation} 
Combining \eqref{eq:stability_bound_alpha_ge_2} and \eqref{eq:penalty_bound_alpha_ge2}, we have

\begin{align*}
    \lefteqn{\E \qty[ \left\langle \hat{\ell}_t, \phi\qty(\eta_t \hat{L}_t ; \dis) - \phi\qty(\eta_t (\hat{L}_t + \hat{\ell}_t) ; \dis) \right\rangle
    +
    \frac{\left\langle r_{t+1}, a_{t+1} - a^*\right\rangle}{2c\ma\da\sqrt{t}}\m\hat{L}_t]}\\
    &\leq
    \sum_{i:a^*_i=0} \qty(\qty(2(\alpha+1)+\frac{2\alpha }{\wlow(1-\zeta)^{\alpha+1}}+c_{s,4}(\dis))\frac{1}{\underline{\hat{L}}_{t,i}}
    +\frac{1}{2c\ma\da\sqrt{t}}\frac{\alpha}{\alpha-1}\frac{1}{(\eta_t\underline{\hat{L}}_{t,i})^{\alpha-1}})\\
    &\hspace{2cm}+\frac{m}{1 - \wlow} (1 - \wlow)^{\tfrac{\zeta}{\eta_t}} \qty( \frac{\zeta}{\eta_t} + \frac{1}{\wlow} )+
        m^2(1 - \wlow)^{\frac{\zeta}{\eta_t} -1}
    \numberthis\label{eq:stochastic_bound_on_Dt_to_copy}
    \displaybreak[0]\\
    &\leq
    \sum_{i:a^*_i=0} \qty(\qty(2(\alpha+1)+\frac{2\alpha }{\wlow(1-\zeta)^{\alpha+1}}+c_{s,4}(\dis))\frac{1}{\underline{\hat{L}}_{t,i}}
    +\frac{1}{2(c\ma\da)^2}\frac{\alpha}{\alpha-1}\frac{1}{\underline{\hat{L}}_{t,i}})\\
    &\hspace{2cm}+\frac{m}{1 - \wlow} (1 - \wlow)^{\tfrac{\zeta}{\eta_t}} \qty( \frac{\zeta}{\eta_t} + \frac{1}{\wlow} )+
        m^2(1 - \wlow)^{\frac{\zeta}{\eta_t} -1}
    \numberthis\label{eq:Dt_power_reduction}
    \displaybreak[0]\\
    &=
    C_3(\dis)\sum_{i:a^*_i=0} 
    \frac{1}{\underline{\hat{L}}_{t,i}}
    +C_{4,t}(\dis),\numberthis\label{eq:stochastic_bound_on_Dt}
\end{align*}
where we defined
\begin{align*}
    C_3(\dis)&\coloneqq 2(\alpha+1)+\frac{2\alpha }{\wlow(1-\zeta)^{\alpha+1}}+c_{s,4}(\dis)+\frac{\alpha}{2(\alpha-1)(c\ma\da)^2},\\
    C_{4,t}(\dis)&\coloneqq 
    \frac{m}{1 - \wlow} (1 - \wlow)^{\tfrac{\zeta}{\eta_t}} \qty( \frac{\zeta}{\eta_t} + \frac{1}{\wlow} )+
        m^2(1 - \wlow)^{\frac{\zeta}{\eta_t} -1}.
\end{align*}
Here, \eqref{eq:Dt_power_reduction} holds since for $\alpha\geq 2$ we have
\begin{equation*}
    \frac{1}{2c\ma\da\sqrt{t}}\frac{\alpha}{\alpha-1}\frac{1}{(\eta_t\underline{\hat{L}}_{t,i})^{\alpha-1}}
    =
    \frac{1}{2(c\ma\da)^2}\frac{1}{\underline{\hat{L}}_{t,i}}
    \frac{1}{(\eta_t\underline{\hat{L}}_{t,i})^{\alpha-2}}
    \leq
    \frac{1}{2(c\ma\da)^2}\frac{1}{\underline{\hat{L}}_{t,i}},
\end{equation*}
where the inequality follows from $\eta_t\underline{\hat{L}}_{t,i}\geq m^{1/(\alpha-1)}\geq 1$ for any $i$ with $a^*_i=0$ on $F_t$ or $D_t$.

Then, by the second part of Lemma~\ref{lem:optimal_action_bound_case2}, $C_{4,t}(\dis)$ satisfies
\begin{align*}
    \sum_{t=1}^T C_{4,t}(\dis)&< \sum_{t=1}^\infty C_{4,t}(\dis)\\
    &=\sum_{t=1}^\infty \frac{m}{1 - \wlow} (1 - \wlow)^{\tfrac{\zeta}{\eta_t}} \qty( \frac{\zeta}{\eta_t} + \frac{1}{\wlow} )+
        m^2\sum_{t=1}^\infty (1 - \wlow)^{\frac{\zeta}{\eta_t} -1}\\
    &\leq O\qty(c^2m^{1-\frac{2}{\alpha}}d^{\frac{2}{\alpha}-1})+m^2\cdot O\qty(c^2m^{1-\frac{2}{\alpha}}d^{\frac{2}{\alpha}-1})\\
    &\leq
    O\qty(c^2m^{3-\frac{2}{\alpha}}d^{\frac{2}{\alpha}-1}).\numberthis\label{eq:C3t_sum_bound}
\end{align*}

Now we consider the first term of \eqref{eq:stochastic_regret_decomposition_pareto} on $F_t^c$ or $D_t^c$. 
By applying Lemmas~\ref{lem:stability} and \ref{lem:penalty_adv_bound}, we obtain
\begin{align*}
    \lefteqn{\E \qty[ \left\langle \hat{\ell}_t, \phi\qty(\eta_t \hat{L}_t ; \dis) - \phi\qty(\eta_t (\hat{L}_t + \hat{\ell}_t) ; \dis) \right\rangle
    +
    \frac{\left\langle r_{t+1}, a_{t+1} - a^*\right\rangle}{2c\ma\da\sqrt{t}}\m\hat{L}_t]}
    \displaybreak[0]\\
    &\leq
    \frac{\acon c\ma\da}{\sqrt{t}}\qty(m+\frac{\alpha}{\alpha-1} m^{\frac{1}{\alpha}}(d-m)^{1-\frac{1}{\alpha}})\\
    &\hspace{3cm}+
    \begin{cases}
        \frac{\qty(\frac{\alpha}{\alpha-1}(m-1)^{1-\frac{1}{\alpha}}+\Gamma\qty(1-\frac{1}{\alpha}))\qty(d+1)^{\frac{1}{\alpha}}+m}{2c\ma\da}, & \dis=\Fdis,\\
        \frac{\qty(\frac{\alpha}{\alpha-1}(m-1)^{1-\frac{1}{\alpha}}+\Gamma\qty(1-\frac{1}{\alpha}))\qty(d+1)^{\frac{1}{\alpha}}}{2c\ma\da}, & \dis=\Pdis,\\
    \end{cases}
    \displaybreak[0]\\
    &=
    \frac{\acon c\qty(\frac{m}{d})^{\frac{1}{2}-\frac{1}{\alpha}}}{\sqrt{t}}\qty(m+\frac{\alpha}{\alpha-1} m^{\frac{1}{\alpha}}(d-m)^{1-\frac{1}{\alpha}})\\
    &\hspace{3cm}+
    \begin{cases}
        \frac{\qty(\frac{\alpha}{\alpha-1}(m-1)^{1-\frac{1}{\alpha}}+\Gamma\qty(1-\frac{1}{\alpha}))\qty(d+1)^{\frac{1}{\alpha}}+m}{2c\qty(\frac{m}{d})^{\frac{1}{2}-\frac{1}{\alpha}}}, & \dis=\Fdis,\\
        \frac{\qty(\frac{\alpha}{\alpha-1}(m-1)^{1-\frac{1}{\alpha}}+\Gamma\qty(1-\frac{1}{\alpha}))\qty(d+1)^{\frac{1}{\alpha}}}{2c\qty(\frac{m}{d})^{\frac{1}{2}-\frac{1}{\alpha}}}, & \dis=\Pdis,
    \end{cases}\\
    &\leq 
    C_5(\dis)
    \sqrt{\frac{md}{t}},\numberthis\label{eq:stochastic_bound_on_Dtc}
\end{align*}
where we recall that
\begin{equation*}
    \acon=
    \begin{cases}
        2(\alpha+1)\qty(49+\frac{7\cdot 4^{1/\alpha}}{1-1/\alpha}), & \text{if }\dis=\Fdis,\\
        9(\alpha+1), &\text{if } \dis=\Pdis,
    \end{cases} 
\end{equation*}
is given in Lemma~\ref{lem:adv_bound_term} and we defined
\begin{equation*}
    C_5(\dis)\coloneqq
    \begin{cases}
        2c(\alpha+1)\qty(49+\frac{7\cdot 4^{1/\alpha}}{1-1/\alpha})\qty(\qty(\frac{m}{d})^{1-\frac{1}{\alpha}}+\frac{\alpha}{\alpha-1}\qty(1-\frac{m}{d})^{1-\frac{1}{\alpha}})\\
        \hspace{3cm}+\frac{\qty(\frac{\alpha}{\alpha-1}(1-\frac{1}{m})^{1-\frac{1}{\alpha}}+\Gamma\qty(1-\frac{1}{\alpha})m^{\frac{1}{\alpha}-1})(1+\frac{1}{d})^{\frac{1}{\alpha}}+\frac{m}{d}^{\frac{1}{\alpha}}}{2c}, & \dis=\Fdis,\\
        9c(\alpha+1)\qty(\qty(\frac{m}{d})^{1-\frac{1}{\alpha}}+\frac{\alpha}{\alpha-1}\qty(1-\frac{m}{d})^{1-\frac{1}{\alpha}})\\
        \hspace{3cm}+\frac{\qty(\frac{\alpha}{\alpha-1}(1-\frac{1}{m})^{1-\frac{1}{\alpha}}+\Gamma\qty(1-\frac{1}{\alpha})m^{\frac{1}{\alpha}-1})(1+\frac{1}{d})^{\frac{1}{\alpha}}}{2c}, & \dis=\Pdis.
    \end{cases}
\end{equation*}

In the following analysis, for notational convenience, we write both events $F_t$ for $\dis=\Fdis$ and $D_t$ for $\dis=\Pdis$ as $F_t$, since the corresponding results take the same form up to constant factors.

Combining \eqref{eq:stochastic_regret_decomposition_pareto} with the above results, we have 
\begin{align*}
    \R(T) &\leq
    \sum_{t=1}^{T} \E\Bigg[
        \E\Bigg[
        \ind\qty[F_t]
        \sum_{i:a^*_i=0} \frac{C_3(\dis)}{\underline{\hat{L}}_{t,i}}
        +
        \ind\qty[F_t^c]
        C_5(\dis)
        \sqrt{\frac{md}{t}}\Bigg|
        \hat{L}_t\Bigg]
        \Bigg]
    \\
    &\hspace{5.5cm}+
    C_2(\dis)\sqrt{md}
    +
    \sum_{t=1}^T C_{4,t}(\dis)\tag{by \eqref{eq:stochastic_bound_on_Dt} and \eqref{eq:stochastic_bound_on_Dtc}}\\
    &\leq
    \sum_{t=1}^{T} \E\Bigg[
        \E\Bigg[
        \ind\qty[F_t]
        \sum_{i:a^*_i=0} \frac{C_3(\dis)}{\underline{\hat{L}}_{t,i}}
        +
        \ind\qty[F_t^c]
        C_5(\dis)
        \sqrt{\frac{md}{t}}\Bigg|
        \hat{L}_t\Bigg]
        \Bigg]
    \\
    &\hspace{5.5cm}+
    C_2(\dis)\sqrt{md}
    +
    O\qty(c^2m^{3-\frac{2}{\alpha}}d^{\frac{2}{\alpha}-1}),
    \numberthis\label{eq:self-bounding_upper_bound_pareto}
\end{align*}
where the last inequality follows from \eqref{eq:C3t_sum_bound}. 

On the other hand, by Lemmas~\ref{lem:lower_bound_Ft} and \ref{lem:lower_bound_Dt}, the regret is bounded from below as
\begin{equation}\label{eq:self-bounding_lower_bound_pareto}
    \R(T)\geq\sum_{t=1}^T\E\qty[\ind\qty[F_t]c_{s,1}(\dis)\sum_{i:a^*_i=0}\frac{\Delta_i t^\frac{\alpha}{2}}{(c\ma\da\underline{\hat{L}}_{t,i})^\alpha}
    +\ind\qty[F_t^c]\frac{c_{s,2}(\dis)\Delta}{m^{\alpha/(\alpha-1)}}].
\end{equation}
We apply the self-bounding technique and considering
$\eqref{eq:self-bounding_upper_bound_pareto}-\eqref{eq:self-bounding_lower_bound_pareto}/2$ we have
\begin{align*}
    \frac{\R(T)}{2}&\leq
    \sum_{t=1}^{T}\E\qty[
        \ind\qty[F_t]
        \sum_{i:a^*_i=0} 
        \qty(\frac{C_3(\dis)}{\underline{\hat{L}}_{t,i}}
        -
        \frac{c_{s,1}(\dis)\Delta_i t^\frac{\alpha}{2}}{2(c\ma\da\underline{\hat{L}}_{t,i})^\alpha})
    ]\\
    &\hspace{3cm}+\sum_{t=1}^{T}\E\qty[
        \ind\qty[F_t^c]
        \qty(C_5(\dis)
        \sqrt{\frac{md}{t}}
        -
        \frac{c_{s,2}(\dis)\Delta}{2m^{\alpha/(\alpha-1)}})
    ]\\
    &\hspace{3cm}+
    C_2(\dis)\sqrt{md}
    +
    O\qty( c^2m^{3-\frac{2}{\alpha}}d^{\frac{2}{\alpha}-1}).\numberthis\label{eq:self-bounding_general_pareto}
\end{align*}

For the first term of \eqref{eq:self-bounding_general_pareto}, we have
\begin{align*}
    \lefteqn{\frac{C_3(\dis)}{\underline{\hat{L}}_{t,i}}
        -
        \frac{c_{s,1}(\dis)\Delta_i t^\frac{\alpha}{2}}{2(c\ma\da\underline{\hat{L}}_{t,i})^\alpha}
    \leq
    C_3(\dis)\frac{\alpha}{\alpha-1}\qty(\frac{2C_3(\dis)}{\alpha c_{s,1}(\dis)\Delta_i})^{\frac{1}{\alpha-1}}
    \frac{\qty(c\ma\da)^{\frac{\alpha}{\alpha-1}}}{t^{\frac{\alpha}{2(\alpha-1)}}}}\\
    &=
    \begin{cases}
        \qty(2(\alpha+1)+\frac{\alpha}{2(\alpha-1)(c\ma\da)^2}+\frac{2\alpha e}{(1-\zeta)^{\alpha+1}}+\frac{4}{e})
        \frac{\alpha}{\alpha-1}\\
        \hspace{1.5cm}\times\qty(\frac{4(\alpha+1)+\frac{\alpha}{(\alpha-1)(c\ma\da)^2}+\frac{4\alpha e}{(1-\zeta)^{\alpha+1}}+\frac{8}{e}}{\alpha c_{s,1}(\Fdis)\Delta_i})^{\frac{1}{\alpha-1}}
        \frac{\qty(c\ma\da)^{\frac{\alpha}{\alpha-1}}}{t^{\frac{\alpha}{2(\alpha-1)}}}, & \dis=\Fdis,\\
        \qty(2(\alpha+1)+\frac{\alpha}{2(\alpha-1)(c\ma\da)^2}+\frac{14.4\alpha}{(1-\zeta)^{\alpha+1}}+7.2\alpha)
        \frac{\alpha}{\alpha-1}\\
        \hspace{1.5cm}\times\qty(\frac{4(\alpha+1)+\frac{\alpha}{(\alpha-1)(c\ma\da)^2}+\frac{28.8\alpha}{(1-\zeta)^{\alpha+1}}+14.4\alpha}{\alpha c_{s,1}(\Pdis)\Delta_i})^{\frac{1}{\alpha-1}}
        \frac{\qty(c\ma\da)^{\frac{\alpha}{\alpha-1}}}{t^{\frac{\alpha}{2(\alpha-1)}}}, & \dis=\Pdis,
    \end{cases}\\
    &=
    O\qty(\frac{1}{
        \Delta_i^{\frac{1}{\alpha-1}}
        m^{\frac{\alpha-2}{2(\alpha-1)}}
        d^{\frac{2-\alpha}{2(\alpha-1)}}
        t^{\frac{\alpha}{2(\alpha-1)}}
        }),\numberthis\label{eq:self-bounding_first_term_pareto}
\end{align*}
where the inequality follows from $Ax-Bx^\alpha\leq A\frac{\alpha}{\alpha-1}(\frac{A}{\alpha B})^{\frac{1}{\alpha-1}}$ for any $A,B>0$ and $\alpha>1$.

For the second term of \eqref{eq:self-bounding_general_pareto}, we have
\begin{align*}
    \lefteqn{\sum_{t=1}^T 
    C_5(\dis)\sqrt{\frac{md}{t}}
    -
    \frac{\frac{c_{s,2}(\dis)\Delta}{m^{\alpha/(\alpha-1)}}}{2}
    }
    \displaybreak[0]\\
    &\leq
    \sum_{t=1}^T\max\Bigg\{
    C_5(\dis)\sqrt{\frac{md}{t}}
    -
    \frac{c_{s,2}(\dis)\Delta}{2m^{\alpha/(\alpha-1)}},0
    \Bigg\}
    \displaybreak[0]\\
    &\leq
    C_5^2(\dis)md\frac{2m^{\alpha/(\alpha-1)}}{c_{s,2}(\dis)\Delta}
    \displaybreak[0]\\
    &=
    \begin{cases}
        md\Bigg(2c(\alpha+1)\qty(49+\frac{7\cdot 4^{1/\alpha}}{1-1/\alpha})\qty(\qty(\frac{m}{d})^{1-\frac{1}{\alpha}}+\frac{\alpha}{\alpha-1}\qty(1-\frac{m}{d})^{1-\frac{1}{\alpha}})\\
        \hspace{1.4cm}+
        \frac{\qty(\frac{\alpha}{\alpha-1}(1-\frac{1}{m})^{1-\frac{1}{\alpha}}+\Gamma\qty(1-\frac{1}{\alpha})m^{\frac{1}{\alpha}-1})(1+\frac{1}{d})^{\frac{1}{\alpha}}+\frac{m}{d}^{\frac{1}{\alpha}}}{2c}
        \Bigg)^2 
        \frac{2^{\alpha+1} \qty(m^\frac{\alpha}{\alpha-1}+1)+2}{\Delta}, & \dis=\Fdis,\\
        md\Bigg(
        9c(\alpha+1)\qty(\qty(\frac{m}{d})^{1-\frac{1}{\alpha}}+\frac{\alpha}{\alpha-1}\qty(1-\frac{m}{d})^{1-\frac{1}{\alpha}})
        \\
        \hspace{1.4cm}+
        \frac{\qty(\frac{\alpha}{\alpha-1}(1-\frac{1}{m})^{1-\frac{1}{\alpha}}+\Gamma\qty(1-\frac{1}{\alpha})m^{\frac{1}{\alpha}-1})(1+\frac{1}{d})^{\frac{1}{\alpha}}}{2c}
        \Bigg)^2 
         \frac{2^{\alpha-1}(2+m^{\alpha/(\alpha-1)})+2}{\Delta}, & \dis=\Pdis,
    \end{cases}
    \displaybreak[0]\\
    &=
    O\qty(\frac{m^{\frac{2\alpha-1}{\alpha-1}}d}{\Delta}).\numberthis\label{eq:self-bounding_second_term_pareto}
\end{align*}
By substituting the results of 
\eqref{eq:self-bounding_first_term_pareto} and
\eqref{eq:self-bounding_second_term_pareto} into
\eqref{eq:self-bounding_general_pareto}, we obtain the regret bound for Pareto distributions or Fr\'{e}chet distributions with shape $\alpha\geq 2$ as
\begin{align*}
    \R(T) &\leq
    O\qty(\sum_{i:a^*_i=0}\sum_{t=1}^T\frac{1}{
        \Delta_i^{\frac{1}{\alpha-1}}
        m^{\frac{\alpha-2}{2(\alpha-1)}}
        d^{\frac{2-\alpha}{2(\alpha-1)}}
        t^{\frac{\alpha}{2(\alpha-1)}}
        }
    )\\
    &\hspace{2.8cm}+
    O\qty(\frac{m^{\frac{2\alpha-1}{\alpha-1}}d}{\Delta})
    +
    O\qty(\frac{\sqrt{md}}{c})
    +
    O\qty( c^2m^{3-\frac{2}{\alpha}}d^{\frac{2}{\alpha}-1})
    \displaybreak[0]\\
    &\leq
    O\qty(\frac{m^{\frac{2\alpha-1}{\alpha-1}}d}{\Delta})+
    O\qty(\frac{\sqrt{md}}{c})+
    O\qty( c^2m^{3-\frac{2}{\alpha}}d^{\frac{2}{\alpha}-1})\\
    &\hspace{2.8cm}+
    \begin{cases*}
        O\qty(\sum_{i:a^*_i=0}\frac{\log T}{\Delta_i}), & if $\alpha=2$,\\
        O\qty(\sum_{i:a^*_i=0}\frac{1}{(\alpha-2)}\frac{T^{\frac{\alpha-2}{2(\alpha-1)}}}{\Delta_i^{\frac{1}{\alpha-1}}m^{\frac{\alpha-2}{2(\alpha-1)}}d^{\frac{2-\alpha}{2(\alpha-1)}}}), & if $\alpha>2$.
    \end{cases*}
\end{align*}

\subsubsection{Analysis for Shape \texorpdfstring{$\alpha\in\qty(1,2)$}{α ∈ (1,2)}}

Similarly to the previous case, we still write both events $F_t$ for $\dis=\Fdis$ and $D_t$ for $\dis=\Pdis$ as $F_t$ here. 
The proof for the case of $\alpha\in(1,2)$ coincides with the previous case up to \eqref{eq:stochastic_bound_on_Dt_to_copy} and differs from \eqref{eq:Dt_power_reduction} onward, where on $F_t$ we obtain
\begin{align*}
    \lefteqn{\E \qty[ \left\langle \hat{\ell}_t, \phi\qty(\eta_t \hat{L}_t ; \dis) - \phi\qty(\eta_t (\hat{L}_t + \hat{\ell}_t) ; \dis) \right\rangle
    +
    \frac{\left\langle r_{t+1}, a_{t+1} - a^*\right\rangle}{2c\ma\da\sqrt{t}}\m\hat{L}_t]}\\
    &\leq
    \sum_{i:a^*_i=0} \qty(\qty(2(\alpha+1)+\frac{2\alpha }{\wlow(1-\zeta)^{\alpha+1}}+c_{s,4}(\dis))\frac{1}{\underline{\hat{L}}_{t,i}}
    +\frac{1}{2c\ma\da\sqrt{t}}\frac{\alpha}{\alpha-1}\frac{1}{(\eta_t\underline{\hat{L}}_{t,i})^{\alpha-1}})\\
    &\hspace{0.8cm}+\frac{m}{1 - \wlow} (1 - \wlow)^{\tfrac{\zeta}{\eta_t}} \qty( \frac{\zeta}{\eta_t} + \frac{1}{\wlow} )+
        m^2(1 - \wlow)^{\frac{\zeta}{\eta_t} -1}
    \tag{\eqref{eq:stochastic_bound_on_Dt_to_copy} itself}
    \displaybreak[0]\\
    &\leq
    \sum_{i:a^*_i=0} \frac{2(\alpha+1)+\frac{2\alpha }{\wlow(1-\zeta)^{\alpha+1}}+c_{s,4}(\dis)}{\underline{\hat{L}}_{t,i}}
    +
    \frac{1}{2c\ma\da\sqrt{t}}\frac{\alpha}{\alpha-1}\frac{1}{(\eta_t\underline{\hat{L}}_{t,i})^{\alpha-1}}
    +
    C_{4,t}(\dis)
    \displaybreak[0]\\
    &=
    \sum_{i:a^*_i=0} \frac{C_3'(\dis)}{\underline{\hat{L}}_{t,i}}
    +
    \frac{1}{2(c\ma\da)^\alpha}\frac{\alpha}{\alpha-1}\frac{1}{t^{1-\frac{\alpha}{2}}\underline{\hat{L}}_{t,i}^{\alpha-1}}
    +
    C_{4,t}(\dis).
\end{align*}
Here, we defined
\begin{equation*}
    C_3'(\dis)\coloneqq
    \begin{cases}
        2(\alpha+1)+\frac{2\alpha e}{(1-\zeta)^{\alpha+1}}+\frac{4}{e}, & \dis=\Fdis,\\
        2(\alpha+1)+\frac{14.4\alpha}{(1-\zeta)^{\alpha+1}}+7.2\alpha, & \dis=\Pdis.
    \end{cases}
\end{equation*}

Note that we obtained \eqref{eq:self-bounding_upper_bound_pareto} by combining \eqref{eq:stochastic_regret_decomposition_pareto} and \eqref{eq:stochastic_bound_on_Dt}--\eqref{eq:stochastic_bound_on_Dtc}. Here \eqref{eq:stochastic_regret_decomposition_pareto}, \eqref{eq:C3t_sum_bound} and \eqref{eq:stochastic_bound_on_Dtc} also hold for this setting. By replacing \eqref{eq:stochastic_bound_on_Dt} with the above result, instead of \eqref{eq:self-bounding_upper_bound_pareto} we obtain
\begin{align*}
    \R(T)&\leq
    \sum_{t=1}^{T}\E\qty[
        \ind\qty[F_t]
        \sum_{i:a^*_i=0} 
        \qty(\frac{C_3'(\dis)}{\underline{\hat{L}}_{t,i}}
        +
        \frac{1}{2(c\ma\da)^\alpha}\frac{\alpha}{\alpha-1}\frac{1}{t^{1-\frac{\alpha}{2}}\underline{\hat{L}}_{t,i}^{\alpha-1}})
    ]\\
    &\hspace{1.7cm}
    +\sum_{t=1}^{T}\E\qty[
        \ind\qty[F_t^c]
        C_5(\dis)\sqrt{\frac{md}{t}}
    ]
    +
    C_2(\dis)\sqrt{md}
    +
    O\qty( c^2m^{3-\frac{2}{\alpha}}d^{\frac{2}{\alpha}-1}).\numberthis\label{eq:self-bounding_upper_bound_alpha_less_2}
\end{align*}
The regret lower bound in \eqref{eq:self-bounding_lower_bound_pareto} still holds for this setting. 
Then, by following the same step as \eqref{eq:self-bounding_general_pareto} to apply the self-bounding technique 
and considering $\eqref{eq:self-bounding_upper_bound_alpha_less_2}-\eqref{eq:self-bounding_lower_bound_pareto}/2$, we have

\begin{align*}
    \frac{\R(T)}{2}&\leq
    \sum_{t=1}^{T}\E\qty[
        \ind\qty[F_t]
        \sum_{i:a^*_i=0} 
        \qty(\frac{C_3'(\dis)}{\underline{\hat{L}}_{t,i}}
        +
        \frac{1}{2(c\ma\da)^\alpha}\frac{\alpha}{\alpha-1}\frac{1}{t^{1-\frac{\alpha}{2}}\underline{\hat{L}}_{t,i}^{\alpha-1}}-
        \frac{c_{s,1}(\dis)\Delta_i t^\frac{\alpha}{2}}{2(c\ma\da\underline{\hat{L}}_{t,i})^\alpha})
    ]
    \numberthis\label{eq:self-bounding_general_pareto_alpha_less_2_part1}\\
    &+\sum_{t=1}^{T}\E\qty[
        \ind\qty[F_t^c]
        \qty(
        C_5(\dis)\sqrt{\frac{md}{t}}
        -
        \frac{c_{s,2}(\dis)\Delta}{2m^{\alpha/(\alpha-1)}})
    ]
    +
    C_2(\dis)\sqrt{md}
    +
    O\qty( c^2m^{3-\frac{2}{\alpha}}d^{\frac{2}{\alpha}-1}),
\end{align*}
where the first term \eqref{eq:self-bounding_general_pareto_alpha_less_2_part1} can be written as
\begin{multline}\label{eq:to_bound_two_terms_pareto_alpha_less_2}
    \eqref{eq:self-bounding_general_pareto_alpha_less_2_part1}=
    \sum_{t=1}^{T}\E\Bigg[
        \ind\qty[F_t]
        \sum_{i:a^*_i=0} 
        \qty(\frac{C_3'(\dis)}{\underline{\hat{L}}_{t,i}}
    -
    \frac{c_{s,1}(\dis)\Delta_i t^\frac{\alpha}{2}}{4(c\ma\da\underline{\hat{L}}_{t,i})^\alpha})
    \\+
    \sum_{i:a^*_i=0} \frac{1}{2(c\ma\da)^\alpha}
    \qty(\frac{\alpha}{\alpha-1}\frac{1}{t^{1-\frac{\alpha}{2}}\underline{\hat{L}}_{t,i}^{\alpha-1}}
    -
    \frac{c_{s,1}(\dis)\Delta_i t^\frac{\alpha}{2}}{2\underline{\hat{L}}_{t,i}^\alpha}
    )\Bigg].
\end{multline}
The first term of \eqref{eq:to_bound_two_terms_pareto_alpha_less_2} can be bounded in the same way as \eqref{eq:self-bounding_first_term_pareto}, which results in
\begin{align*}
    \lefteqn{\frac{C_3'(\dis)}{\underline{\hat{L}}_{t,i}}
    -
    \frac{c_{s,1}(\dis)\Delta_i t^\frac{\alpha}{2}}{4(c\ma\da\underline{\hat{L}}_{t,i})^\alpha}
    \leq
    C_3'(\dis)\frac{\alpha}{\alpha-1}\qty(\frac{4C_3'(\dis)}{\alpha c_{s,1}(\dis)\Delta_i})^{\frac{1}{\alpha-1}}
    \frac{\qty(c\ma\da)^{\frac{\alpha}{\alpha-1}}}{t^{\frac{\alpha}{2(\alpha-1)}}}}\\
    &=
    \begin{cases}
        \qty(2(\alpha+1)+\frac{2\alpha e}{(1-\zeta)^{\alpha+1}}+\frac{4}{e})
        \frac{\alpha}{\alpha-1}
        \qty(\frac{8(\alpha+1)+\frac{8\alpha e}{(1-\zeta)^{\alpha+1}}+\frac{16}{e}}{\alpha c_{s,1}(\Fdis)\Delta_i})^{\frac{1}{\alpha-1}}
        \frac{\qty(c\ma\da)^{\frac{\alpha}{\alpha-1}}}{t^{\frac{\alpha}{2(\alpha-1)}}}, & \dis=\Fdis,\\
        \qty(2(\alpha+1)+\frac{14.4\alpha}{(1-\zeta)^{\alpha+1}}+7.2\alpha)
        \frac{\alpha}{\alpha-1}
        \qty(\frac{8(\alpha+1)+\frac{57.6\alpha}{(1-\zeta)^{\alpha+1}}+28.8\alpha}{\alpha c_{s,1}(\Pdis)\Delta_i})^{\frac{1}{\alpha-1}}
        \frac{\qty(c\ma\da)^{\frac{\alpha}{\alpha-1}}}{t^{\frac{\alpha}{2(\alpha-1)}}}, & \dis=\Pdis.
    \end{cases}
\end{align*}

For the second term of \eqref{eq:to_bound_two_terms_pareto_alpha_less_2}, since $Ax^{\alpha-1}-Bx^\alpha\leq\frac{A}{\alpha}\qty(\frac{\alpha-1}{\alpha}\frac{A}{B})^{\alpha-1}$ for any $A,B>0$ and $\alpha>1$, we have
\begin{align*}
    \frac{\alpha}{\alpha-1}\frac{1}{t^{1-\frac{\alpha}{2}}\underline{\hat{L}}_{t,i}^{\alpha-1}}
    -
    \frac{c_{s,1}(\dis)\Delta_i t^\frac{\alpha}{2}}{2\underline{\hat{L}}_{t,i}^\alpha}
    &\leq
    \frac{1}{\alpha-1}\frac{1}{t^{1-\frac{\alpha}{2}}}\qty(\frac{2}{\Delta_i t c_{s,1}(\dis)})^{\alpha-1}\\
    &=
    \frac{1}{\alpha-1}\qty(\frac{2}{\Delta_i c_{s,1}(\dis)})^{\alpha-1}\frac{1}{t^{\frac{\alpha}{2}}}.
\end{align*}
Therefore, \eqref{eq:self-bounding_general_pareto_alpha_less_2_part1} can be bounded as
\begin{equation*}
    \eqref{eq:self-bounding_general_pareto_alpha_less_2_part1}\leq
    \begin{cases}
        \qty(2(\alpha+1)+\frac{2\alpha e}{(1-\zeta)^{\alpha+1}}+\frac{4}{e})
        \frac{\alpha}{\alpha-1}
        \qty(\frac{8(\alpha+1)+\frac{8\alpha e}{(1-\zeta)^{\alpha+1}}+\frac{16}{e}}{\alpha c_{s,1}(\Fdis)\Delta_i})^{\frac{1}{\alpha-1}}\\
        \hspace{1cm}\times\sum_{t=1}^T \frac{\qty(c\ma\da)^{\frac{\alpha}{\alpha-1}}}{t^{\frac{\alpha}{2(\alpha-1)}}}
        +\frac{1}{2(c\ma\da)^\alpha}
        \frac{1}{\alpha-1}\qty(\frac{2}{\Delta_i c_{s,1}(\dis)})^{\alpha-1}\sum_{t=1}^T\frac{1}{t^{\frac{\alpha}{2}}}, & \dis=\Fdis,\\
        \qty(2(\alpha+1)+\frac{14.4\alpha}{(1-\zeta)^{\alpha+1}}+7.2\alpha)
        \frac{\alpha}{\alpha-1}
        \qty(\frac{8(\alpha+1)+\frac{57.6\alpha}{(1-\zeta)^{\alpha+1}}+28.8\alpha}{\alpha c_{s,1}(\Pdis)\Delta_i})^{\frac{1}{\alpha-1}}\\
        \hspace{1cm}\times\sum_{t=1}^T \frac{\qty(c\ma\da)^{\frac{\alpha}{\alpha-1}}}{t^{\frac{\alpha}{2(\alpha-1)}}}
        +\frac{1}{2(c\ma\da)^\alpha}
        \frac{1}{\alpha-1}\qty(\frac{2}{\Delta_i c_{s,1}(\dis)})^{\alpha-1}\sum_{t=1}^T\frac{1}{t^{\frac{\alpha}{2}}}, & \dis=\Pdis.
    \end{cases}
\end{equation*}
Here, note that the exponent of $t$ in the first term is $\frac{\alpha}{2(\alpha-1)}>1$.
Therefore, the summation over $t$ of the first term converges to a constant. Then, by following the same steps from \eqref{eq:self-bounding_second_term_pareto}, we obtain the regret bound for Pareto distributions or Fr\'{e}chet distributions with shape $\alpha\in(1,2)$ as
\begin{multline*}
    \R(T) \leq
    O\qty(\sum_{i:a^*_i=0}\frac{1}{2-\alpha}\frac{1}{c^\alpha m^{\frac{\alpha}{2}-1} d^{1-\frac{\alpha}{2}}}\frac{T^{1-\frac{\alpha}{2}}}{\Delta_i^{\alpha-1}})
    +
    O\qty(\sum_{i:a^*_i=0}\frac{m^{\frac{\alpha-2}{2(\alpha-1)}}d^{\frac{2-\alpha}{2(\alpha-1)}}}{\Delta_i^{\frac{1}{\alpha-1}}})\\
    +O\qty(\frac{m^{\frac{2\alpha-1}{\alpha-1}}d}{\Delta})+
    O\qty(\frac{\sqrt{md}}{c})+
    O\qty( c^2m^{3-\frac{2}{\alpha}}d^{\frac{2}{\alpha}-1}).\dqed
\end{multline*}

\section{Conclusion}

In this paper, we studied the optimality and complexity of FTPL with Fr\'{e}chet or Pareto perturbation with shape $\alpha>1$ in $m$-set combinatorial semi-bandit problems. 
We showed that FTPL achieves the best possible regret bound of $O(\sqrt{mdT})$ in adversarial setting and improves upon the worst-case regret bound in stochastic setting. 
Especially, FTPL achieves a stochastic logarithmic regret when $\alpha=2$, meaning that FTPL can achieve BOBW guarantee in $m$-set semi-bandits. 
Furthermore, we extended CGR to $m$-set semi-bandits, enabling FTPL to achieve the average complexity of $O\qty(md(\log(d/m)+1))$ without sacrificing the regret guarantee. 
Our experiments demonstrated that CGR achieves superiority on computational efficiency without degrading regret performance compared to the original GR.


\acks{BC was supported partially by JST/CREST Innovative Measurement and Analysis (Grant Number JPMJCR2333). 
JL was supported by the National Research Foundation of Korea (NRF) grant funded by the Korea government (MSIT) (No. RS-2024-00395303) and by the grant No. 2025-02304717 (IITP) funded by Korea Government (MSIT). 
CK was supported by the grant Nos. 2024-00460980; and 2025-02304717 (IITP) funded by the Korea government (the Ministry of Science and ICT). 
JH was supported partially by JSPS KAKENHI (Grant Number JP25K03184). 
Chansoo Kim and Junya Honda are the co-corresponding authors.}


\newpage

\appendix
\section{Technical Lemmas}\label{app:lemmas}

\begin{lemma}(Gautschi's inequality, \citeauthor{gautschi1959some}, \citeyear{gautschi1959some})\label{lem:Gautschi}
    For $x > 0$ and $s \in (0,1)$,
    $$x^{1 - s} < \frac{\Gamma(x + 1)}{\Gamma(x + s)} < (x + 1)^{1 - s}.$$
\end{lemma}

\begin{lemma}\citep[Eq.~(3.7)]{malik1966exact}
    \label{lem:pareto_order_statistics}
    Let $X_{k,n}$ be the $k$-th order statistics of i.i.d.~RVs from $\mathcal{P}_{\alpha}$ for $k\in[n]$, where $\alpha>1$. 
    Then, we have
    \begin{equation*}
        \E[X_{k,n}] =
        \frac{\Gamma\qty(n+1)\Gamma\qty(n-k-\frac{1}{\alpha}+1)}{\Gamma\qty(n-k+1)\Gamma\qty(n-\frac{1}{\alpha}+1)}.
    \end{equation*}
\end{lemma}

\begin{lemma}\label{lem:order_statistics}
Let $F(x)$ and $G(x)$ be CDFs of some random variables
such that $G(x)\ge F(x)$ for all $x\in\mathbb{R}$.
Let $(X_1,X_2,\dots,X_n)$ (resp.~$(Y_1,Y_2,\dots,Y_n)$) be RVs i.i.d.~from $F$ (resp.~$G$), and
$X_{k,n}$ (resp.~$Y_{k,n}$) be its $k$-th order statistics for any $k \in [n]$.
Then, $\E[Y_{k,n}] \le \E[X_{k,n}]$ holds.
\end{lemma}

\begin{proof}
Let $U\in[0,1]$ be uniform random variable over $[0,1]$
and let $X=F^{-1}(U)$ and $Y=G^{-1}(U)$, where
$F^{-1}$ and $G^{-1}$ are the left-continuous inverses of $F$ and $G$, respectively.
Then, 
$Y\le X$ holds almost surely and the marginal distributions satisfy $X \sim F$ and $Y\sim G$.
Therefore, if when take $(X_1,Y_1),\dots, (X_n,Y_n)$ as i.i.d.~copies of this $(X,Y)$, we see that
$Y_{k,n}\le X_{k,n}$ holds almost surely, which proves the lemma.
\end{proof}

\begin{lemma}\label{lem:penalty_pareto_frechet}
Let
$X_{k,n}$ (resp.~$Y_{k,n}$) be the $k$-th order statistics
of i.i.d.~RVs from $\mathcal{P}_{\alpha}$ (resp.~$\mathcal{F}_{\alpha}$) for $k\in[n]$.
Then, $\E[Y_{k,n}] \le \E[X_{k,n}]+1$ holds.
\end{lemma}

\begin{proof}
Letting $F(x)$ and $G(x)$ be the CDFs of $\mathcal{P}_{\alpha}$ and $\mathcal{F}_{\alpha}$,
we have
\begin{align*}
G(x)&=\ind[x\geq 0]e^{-1/x^{\alpha}}\\
&\geq \ind[x\geq 0]\left(1-\frac{1}{x^{\alpha}}\right)\\
&\geq \ind[x-1\geq 0]\left(1-\frac{1}{((x-1)+1)^{\alpha}}\right)\\
& = F(x-1),
\end{align*}
where $F(x-1)$ is the CDF of $X+1$ for $X\sim \mathcal{P}_{\alpha}$.
Then, it holds from Lemma~\ref{lem:order_statistics} that
$\E[Y_{k,n}] \le \E[X_{k,n}+1]= \E[X_{k,n}]+1$.
\end{proof}

\begin{lemma}\label{lem:mono_B}
    For any decresaing function $f(x)$ and positive function $g(x)$ with $x\in[a,b]$ where $b\leq\infty$, $\int_x^b f(z)g(z)\dd z/\int_x^b g(z)\dd z$ is monotonically decreasing in $x\in[a,b]$.
\end{lemma}

\begin{proof2}
    We have
    \begin{align*}
        \frac{\mathrm{d}}{\mathrm{d} x}\frac{\int_x^b f(z)g(z)\dd z}{\int_x^b g(z)\dd z} 
        &= \frac{1}{\qty(\int_x^b g(x)\dd x)^2} 
        \left(-f(x)g(x) \int_x^b g(z)\dd z 
        + g(x)\int_x^b f(z)g(z)\dd z  \right) \\
        &= \frac{-g(x)}{\qty(\int_x^b g(x)\dd x)^2} 
        \left( \int_x^b f(x)g(z)\dd z 
        - \int_x^b f(z)g(z)\dd z  \right) \leq 0.\dqed
    \end{align*}
\end{proof2}

\begin{lemma}\label{lem:phi_monotone}
    If the components of the $d$-dimensional perturbation $r$ are i.i.d., then for $\lambda\in\mathbb{R}^d$ and any base-arm set $\B\subset[d]$, $\sum_{i\in\B}\phi_i(\lambda)$ is monotonically decreasing in $\lambda_j$ with $j\in\B$ and monotonically increasing in $\lambda_j$ with $j\in[d]\setminus\B$.
\end{lemma}

\begin{proof}
Let $\max_{k\in\mathcal{S}}^{(i)}a_k$ be the $i$-th largest value of $\{a_k\}_{k\in \mathcal{S}}$.
Then
\begin{align*}
\sum_{i\in\B}\phi_i(\lambda)
&=
\E\left[
|\{k\in \B: \sigma_k(r-\lambda)\le m\}|
\;\m \lambda
\right]
\nn
&=
\sum_{i=1}^m
\E\left[
\ind\left[|\{k\in \B: \sigma_k(r-\lambda)\le m\}|\ge i\right]
\;\m \lambda
\right]
\nn
&=
\sum_{i=1}^m
\E\left[
\ind\left[\max_{k\in\B}^{(i)}\{r_k-\lambda_k\}\ge
\max_{k\in[d]\setminus \B}^{(m-i+1)}\{r_k-\lambda_k\}\right]
\right].
\end{align*}
We obtain the lemma since
$\max_{k\in \mathcal{S}}^{(i)}\{r_k-\lambda_k\}$ is decreasing in $\lambda_j$ if $j \in \mathcal{S}$ and
does not depend on $\lambda_j$ otherwise.
\end{proof}

\begin{lemma}\label{lem:max_sum_perturbation}
    Let $r_k^*$ be the $k$-th largest perturbation among $r_{1,1},r_{1,2},\dots,r_{1,d}$ i.i.d. from $\mathcal{D}_{\alpha}$ for $k\in[d]$ and $\alpha>1$.
    Then, we have
    \begin{equation*}
        \E_{r \sim \dis}\qty[\sum_{k=1}^m r_k^*]
        \leq
        \begin{cases}
            \qty(\frac{\alpha}{\alpha-1}(m-1)^{1-\frac{1}{\alpha}}+\Gamma\qty(1-\frac{1}{\alpha}))\qty(d+1)^{\frac{1}{\alpha}} & \text{if } \dis = \Pdis \\
            \qty(\frac{\alpha}{\alpha-1}(m-1)^{1-\frac{1}{\alpha}}+\Gamma\qty(1-\frac{1}{\alpha}))\qty(d+1)^{\frac{1}{\alpha}}+m & \text{if } \dis = \Fdis.
        \end{cases}
    \end{equation*}
\end{lemma}

\begin{proof2}
    Firstly, we have
    \begin{equation}\label{eq:penalty_order_decom}
        \E_{r \sim \dis}\qty[\sum_{k=1}^m r_k^*]
        = \sum_{k=1}^m \E_{r \sim \dis}\qty[r_k^*].
    \end{equation}
    \paragraph{Pareto Distribution}
    If $\dis=\Pdis$, by Lemma~\ref{lem:pareto_order_statistics}, we obtain
    \begin{equation}\label{eq:penalty_order_decom2}
        \sum_{k=d-m+1}^d \E_{r \sim \Pdis}\qty[r_k^*]
        \leq
        \sum_{k=1}^m \frac{\Gamma\qty(d+1)\Gamma\qty(d-k-\frac{1}{\alpha}+1)}{\Gamma\qty(d-k+1)\Gamma\qty(d-\frac{1}{\alpha}+1)}.
    \end{equation}
    For $k=m=d$, we have
    \begin{align*}
        \frac{\Gamma\qty(d+1)\Gamma\qty(d-k-\frac{1}{\alpha}+1)}{\Gamma\qty(d-k+1)\Gamma\qty(d-\frac{1}{\alpha}+1)}&=
        \frac{\Gamma\qty(d+1)\Gamma\qty(1-\frac{1}{\alpha})}{\Gamma\qty(d-\frac{1}{\alpha}+1)}\\
        &\leq
        \Gamma\qty(1-\frac{1}{\alpha})\qty(d+1)^{\frac{1}{\alpha}},\numberthis\label{eq:penalty_order_kmd}
    \end{align*}
    where the last inequality follows from Gautschi's inequality in Lemma~\ref{lem:Gautschi}.
    Similarly, for $k\in[m]$ and $k<d$, 
    by Gautschi's inequality, we have
    \begin{equation}\label{eq:penalty_Gautschi}
        \frac{\Gamma\qty(d+1)\Gamma\qty(d-k-\frac{1}{\alpha}+1)}{\Gamma\qty(d-k+1)\Gamma\qty(d-\frac{1}{\alpha}+1)} \leq \qty(\frac{d+1}{d-k})^{\frac{1}{\alpha}}.
    \end{equation}
    By combining \eqref{eq:penalty_order_decom}--\eqref{eq:penalty_Gautschi}, we have
    \begin{align*}
        \E_{r\sim\Pdis}\qty[\max_{a\in\nec}a^\top r]
        &=
        \sum_{k=d-m+1}^d \E_{r \sim \Pdis}\qty[r_k^*]\\
        &\leq \Gamma\qty(1-\frac{1}{\alpha})\qty(d+1)^{\frac{1}{\alpha}}+
        \sum_{k=d-m+1}^{d-1} \qty(\frac{d+1}{d-k})^{\frac{1}{\alpha}}\\
        &=
        \Gamma\qty(1-\frac{1}{\alpha})\qty(d+1)^{\frac{1}{\alpha}}+\sum_{k=1}^{m-1} \qty(\frac{d+1}{k})^{\frac{1}{\alpha}}\\
        &\leq \qty(\Gamma\qty(1-\frac{1}{\alpha})+1+\int_{1}^{m-1} x^{-\frac{1}{\alpha}} \dd x)\qty(d+1)^{\frac{1}{\alpha}}\\
        &= \qty(\Gamma\qty(1-\frac{1}{\alpha})+1+\frac{\alpha}{\alpha-1}x^{1-\frac{1}{\alpha}}\bigg|_1^{m-1})\qty(d+1)^{\frac{1}{\alpha}}\\
        &= \qty(\frac{\alpha}{\alpha-1}(m-1)^{1-\frac{1}{\alpha}}+\Gamma\qty(1-\frac{1}{\alpha})-\frac{1}{\alpha-1})\qty(d+1)^{\frac{1}{\alpha}}\\
        &< \qty(\frac{\alpha}{\alpha-1}(m-1)^{1-\frac{1}{\alpha}}+\Gamma\qty(1-\frac{1}{\alpha}))\qty(d+1)^{\frac{1}{\alpha}}.\numberthis\label{eq:penalty_pareto}
    \end{align*}
    \paragraph{Fr\'{e}chet Distribution}
    If $\dis=\Fdis$, by combining \eqref{eq:penalty_order_decom}, Lemma~\ref{lem:penalty_pareto_frechet} and \eqref{eq:penalty_pareto} we have
    \begin{align*}
        \E_{r\sim\Fdis}\qty[\max_{a\in\nec}a^\top r]
        &= \sum_{k=1}^m \E_{r \sim \Fdis}\qty[r_k^*]\\
        &\leq \sum_{k=1}^m \E_{r \sim \Pdis}\qty[r_k^*]+1\\
        &< \qty(\frac{\alpha}{\alpha-1}(m-1)^{1-\frac{1}{\alpha}}+\Gamma\qty(1-\frac{1}{\alpha}))\qty(d+1)^{\frac{1}{\alpha}} + m.\dqed
    \end{align*}
\end{proof2}

\section{Properties of Conditional Geometric Resampling}\label{app:proof_GR}
In this appendix, we provide the proofs of Lemmas~\ref{lem:cgr_general_idea} and \ref{lem:CGR},
and then give a detailed analysis on the complexity of CGR.

\subsection{Proof of Lemma~\ref{lem:cgr_general_idea}}\label{append_proof_cgr}

    Define
    \begin{equation*}
        \chi_{t,i}(r_t'') =
        \begin{cases} 
            1, & \text{if } \qty[\argmin_{a \in \nec} \left\{a^\top (\eta_t \hat{L}_t - r''_t)\right\}]_i=1, \\
            0, & \text{otherwise}.
        \end{cases}
    \end{equation*}
    Consider $w_{t,i}$, the probability that base-arm $i$ is selected, with the condition $\mathcal{E}_{t,i}$.
Then $w_{t,i}$ can be expressed as 
    \begin{align*}
        w_{t,i} &= \sP[\chi_{t,i}(r_t'')=1|\hat{L}_t] \\
        &= \sP[\chi_{t,i}(r_t'')=1|\mathcal{E}_{t,i},\hat{L}_t]\sP[\mathcal{E}_{t,i}|\hat{L}_t]+\sP[\chi_{t,i}(r_t'')=1|\mathcal{E}_{t,i}^c,\hat{L}_t]\sP[\mathcal{E}_{t,i}^c|\hat{L}_t]\\
&
=\sP[\chi_{t,i}(r_t'')=1|\mathcal{E}_{t,i},\hat{L}_t]\sP[\mathcal{E}_{t,i}|\hat{L}_t],
\numberthis{\label{eq: conditioned on nec_t or not}}
    \end{align*}
where the last equality follows since $\mathcal{E}_{t,i}$ is a necessary condition for $\chi_{t,i}(r_t'')=1$.

Note that the number of resampling $M_{t,i}$ follows the geometric distribution with success probability
$\sP[\chi_{t,i}(r_t'')=1|\mathcal{E}_{t,i},\hat{L}_t]$ given $\hat{L}_t$,
which satisfies $\sP[\chi_{t,i}(r_t'')=1|\mathcal{E}_{t,i},\hat{L}_t]=w_{t,i}/\sP[\mathcal{E}_{t,i}|\hat{L}_t]$
by \eqref{eq: conditioned on nec_t or not}.
Since geometric distribution with success probability $p$ has expectation $1/p$ and variance $1/p^2-1/p$,
we have
\begin{align}
\E\qty[M_{t,i}\m \hat{L}_t]=\frac{\sP[\mathcal{E}_{t,i}|\hat{L}_t]}{w_{t,i}},\quad
\mathrm{Var}\qty[M_{t,i}\m \hat{L}_t]
    =\frac{\sP[\mathcal{E}_{t,i}]^2}{w_{t,i}^2}-\frac{\sP[\mathcal{E}_{t,i}]}{w_{t,i}},\n
\end{align}
from which we immediately have
\begin{align}
\E\qty[\frac{M_{t,i}}{\sP[\mathcal{E}_{t,i}]}\m \hat{L}_t]=\frac{1}{w_{t,i}},\quad
\mathrm{Var}\qty[\frac{M_{t,i}}{\sP[\mathcal{E}_{t,i}]}\m \hat{L}_t]
    =\frac{1}{w_{t,i}^2}-\frac{1}{\sP[\mathcal{E}_{t,i}]w_{t,i}}
\le
   \frac{1}{w_{t,i}^2}-\frac{1}{w_{t,i}}.\dqed\n
\end{align}

\subsection{Proof of Lemma~\ref{lem:CGR}}\label{append_lem_cgr}

    Let $\sP^*[\cdot]$ denote the probability distribution of $r''_t$ after the value-swapping operation, and $\sigma_{i,j}$ denote the rank of $r''_{t,j}$ among $\qty{r''_{t,k}:\sigma_k\leq\sigma_i}$ for $j$ such that $\sigma_j\le \sigma_i$.
Then, we have $\mathcal{E}_{t,i}=\qty{\sigma_{i,i}\in[m]}$.
Given $\hat{L}_t,a_{t,i}$ and $\theta$, for any realization $\theta_0$ in $[m]$ of $\theta$ we have
\begin{multline*}
\sP^*\qty[\bigcap\nolimits_{j:\sigma_j\leq\sigma_i} \left\{r''_{t,j} \leq x_j\right\}\m\hat{L}_t,\theta=\theta_0]
=\\
        \sum_{j:\sigma_j\leq\sigma_i}
\sP\qty[\bigcap\nolimits_{k:\sigma_k\leq\sigma_i,i\notin\qty{j,i}} \left\{r''_{t,k} \leq x_k\right\},\, r''_{t,j} \leq x_i,\, r''_{t,i} \leq x_j,\, \sigma_{i,j}=\theta_0 \m\hat{L}_t].
\end{multline*}
    By symmetry of $r''_t\in [\nu,\infty)^d$, we have
\begin{align}
\lefteqn{
\sP\qty[\bigcap\nolimits_{k:\sigma_k\leq\sigma_i,i\notin\qty{j,i}} \left\{r''_{t,k} \leq x_k\right\},\, r''_{t,j} \leq x_i,\, r''_{t,i} \leq x_j,\, \sigma_{i,j}=\theta_0 \m\hat{L}_t]
}\nn
&=
\sP\qty[\bigcap\nolimits_{k:\sigma_k\leq\sigma_i,i\notin\qty{j,i}} \left\{r''_{t,k} \leq x_k\right\},\, r''_{t,i} \leq x_i,\, r''_{t,j} \leq x_j,\, \sigma_{i,i}=\theta_0 \m\hat{L}_t]
\nn
&=
\sP\qty[\bigcap\nolimits_{k:\sigma_k\leq\sigma_i} \left\{r''_{t,k} \leq x_k\right\},\, \sigma_{i,i}=\theta_0 \m\hat{L}_t],
\n
\end{align}
by which we obtain
\begin{align}
\sP^*\qty[\bigcap\nolimits_{j:\sigma_j\leq\sigma_i} \left\{r''_{t,j} \leq x_j\right\}\m\hat{L}_t,\theta=\theta_0]
&= 
\sum_{j:\sigma_j\leq\sigma_i}
\sP\qty[\bigcap\nolimits_{k:\sigma_k\leq\sigma_i} \left\{r''_{t,k} \leq x_k\right\},\, \sigma_{i,i}=\theta_0 \m\hat{L}_t]
\nn
&= 
\sigma_i
\sP\qty[\bigcap\nolimits_{k:\sigma_k\leq\sigma_i} \left\{r''_{t,k} \leq x_k\right\},\, \sigma_{i,i}=\theta_0 \m\hat{L}_t].
\label{eq: prob_swap}
\end{align}

Recall that $\theta$ is randomly chosen from $m$ independent of $\hat{L}_t$.
Then 
\begin{align}
\sP\qty[\theta=\theta_0\m\,\hat{L}_t]=1/m.
\label{theta_uniform}
\end{align}
In addition,
by symmetry of $r''_t\in [\nu,\infty)^d$, we see that
the rank $\sigma_{i,i}$ of $r''_{t,i}$ among $\{r''_{t,j}\}_{j: \sigma_j\le \sigma_i}$
is uniformly distributed over $[\sigma_i]$ given $\hat{L}_t$.
Then we have
\begin{align}
\sP\qty[\sigma_{i,i}\in [m]\m\,\hat{L}_t]
=
\sum_{\theta_0\in[m]}
\sP\qty[\sigma_{i,i}=\theta_0\m\,\hat{L}_t]=
 m/\sigma_i.
\label{eq:prob_event}
\end{align}

By combining the above results we obtain
\begin{align}
\sP^*\qty[\bigcap\nolimits_{j:\sigma_j\leq\sigma_i} \left\{r''_{t,j} \leq x_j\right\}\m\hat{L}_t]
&=\sum_{\theta_0\in[m]}\sP^*\qty[\bigcap\nolimits_{j:\sigma_j\leq\sigma_i} \left\{r''_{t,j} \leq x_j\right\},\theta=\theta_0\m\hat{L}_t]\nn
&=\sum_{\theta_0\in[m]}\sP^*\qty[\bigcap\nolimits_{j:\sigma_j\leq\sigma_i} \left\{r''_{t,j} \leq x_j\right\}\m\hat{L}_t,\theta=\theta_0]\sP\qty[\theta=\theta_0\m\hat{L}_t]\displaybreak[0]\nn
&=\frac{1}{m}\sum_{\theta_0\in[m]}\sP^*\qty[\bigcap\nolimits_{j:\sigma_j\leq\sigma_i} \left\{r''_{t,j} \leq x_j\right\}\m\hat{L}_t,\theta=\theta_0]
\tag*{(by \eqref{theta_uniform})}\displaybreak[0]
\nn
&=\frac{\sigma_i}{m}\sum_{\theta_0\in[m]}\sP\qty[\bigcap\nolimits_{j:\sigma_j\leq\sigma_i} \left\{r''_{t,j} \leq x_j\right\},\,
\sigma_{i,i}=\theta_0
\m\hat{L}_t]
\tag*{(by \eqref{eq: prob_swap})}
\displaybreak[0]\nn
&=\frac{\sigma_i}{m}\sP\qty[\bigcap\nolimits_{j:\sigma_j\leq\sigma_i} \left\{r''_{t,j} \leq x_j\right\},\,
\sigma_{i,i}\in[m]
\m\hat{L}_t]
\displaybreak[0]\nn
&=\frac{\sigma_i \sP\qty[\sigma_{i,i}\in[m] \m\hat{L}_t ]}{m}\sP\qty[\bigcap\nolimits_{j:\sigma_j\leq\sigma_i} \left\{r''_{t,j} \leq x_j\right\}
\m\hat{L}_t,\,\sigma_{i,i}\in[m]]
\displaybreak[0]\nn
&=\sP\qty[\bigcap\nolimits_{j:\sigma_j\leq\sigma_i} \left\{r''_{t,j} \leq x_j\right\}
\m\hat{L}_t,\,\sigma_{i,i}\in[m]],
\tag*{(by \eqref{eq:prob_event})}
\end{align}
which means that CGR samples $r''_t$ from the conditional distribution of $\mathcal{D}$ conditioned on $\qty{\sigma_{i,i}\in[m]}$.

Next we prove the second part of the lemma in \eqref{eq:resampling_bound}.
For $i\in[d]$ such that $\sigma_i\le m$, 
the resampling method is the same as the original GR, which satisfies
    \begin{equation*}
        \E_{r'_t\sim\mathcal{D}}[M_{t,i}|\hat{L}_t]=\frac{1}{w_{t,i}}.
    \end{equation*}
For $i\in[d]$ such that $\sigma_i> m$, 
recall that $\mathcal{E}_{t,i}=\qty{\sigma_{i,i}\in[m]}$ and
$\sP[\sigma_{i,i}\in[m]|\hat{L}_t]=\sigma_i/m$ by \eqref{eq:prob_event}.
Then, by Lemma~\ref{lem:cgr_general_idea} we have
    \begin{equation*}
        \E_{r''_t \sim \mathcal{D}|\mathcal{E}_{t,i}}[M_{t,i}|\hat{L}_t]=\frac{\sP\qty[\sigma_{i,i}\in[m]\m\,\hat{L}_t]}{w_{t,i}}=\frac{m}{\sigma_i w_{t,i}}.
    \end{equation*}
We obtain \eqref{eq:resampling_bound} by the equalities for these two cases.
\qed

\subsection{Complexity of CGR}\label{append_complexity}
In Algorithm~\ref{alg:CGR}, the complexity for computing $U$ and $\{C_i\}_{i\in U}$ in Line~\ref{line:scan_end} is
$O\qty(md)$ since $\sigma_i$ can be computed in $O(d)$ time for each selected base-arm $i$.

Sampling from $\mathcal{D}$ (Line~\ref{line:sample_d}) and finding $\argmin_{a \in \nec} \left\{a^\top (\eta_t \hat{L}_t - r'_t)\right\}$
(Lines~\ref{line:global_end})
can be done in $O(d)$ time as in GR.
Since these procedures are repeated
for $\max_{i: a_{t,i}=1}M_{t,i}$ times,
the total complexity for this part
is
$O\qty(d\cdot \max_{i: a_{t,i}=1}M_{t,i})$.

Finding the $\theta$-th largest element in $r'_{t}$ (Line~\ref{line:swap_begin}) and
$\argmin_{a \in \nec} \left\{a^\top (\eta_t \hat{L}_t - r''_t)\right\}$ (Line~\ref{line:find_a}) can also be done in
$O(d)$ time in the same way as above.
Since 
Lines~\ref{line:swap_begin}--\ref{line:inner_end} are repeated
for $\sum_{i: a_{t,i}=1, \sigma_i>m}M_{t,i}$ times in total,
the total complexity for this part
is
$O\qty(d\cdot \sum_{i: a_{t,i}=1, \sigma_i>m}M_{t,i})$.

Here note that the expectation of $\sum_{i: a_{t,i}=1}M_{t,i}$ is bounded by
\begin{align}
\E\left[
\sum_{i: a_{t,i}=1}
M_{t,i}\m \hat{L}_t
\right]
&=
\E\left[
\sum_{i: a_{t,i}=1}
\left(1\lor \frac{m}{\sigma_i}\right)\frac{1}{w_{t,i}}
\right]
\tag*{(by Lemma~\ref{lem:CGR})}
\nn
&=
\sum_{i\in[d]}
\sP[a_{t,i}=1|\hat{L}_t]
\left(1\lor \frac{m}{\sigma_i}\right)\frac{1}{w_{t,i}}
\nn
&=
\sum_{i\in[d]: \sigma_i\le m}^m w_{t,i}\cdot \frac{1}{w_{t,i}}
+
\sum_{i\in [d]: \sigma_i> m}^d w_{t,i}\cdot \frac{m}{\sigma_i}\frac{1}{w_{t,i}}
\nn
&\le
d+d\int_{m}^{d} \frac{1}{x} \dd x
=d\qty(1+\log(d/m)).\n
\end{align}

From these results, the average total complexity of CGR is bounded by
\begin{align}
C_{\text{CGR}}
&=
O(md)+d\cdot O\qty(
\E\left[
\max_{i: a_{t,i}=1}M_{t,i}+
\sum_{i: a_{t,i}=1, \sigma_i>m}M_{t,i}
\m \hat{L}_t
\right]
)
\nn
&\le
O(md)+d\cdot O\qty(
\E\left[
\sum_{i: a_{t,i}=1}
M_{t,i}\m \hat{L}_t
\right])\nn
&\le O\qty(md\qty(1+\log(d/m))).\n
\end{align}

\section{Proof of Lemma~\ref{lem:both_increasing}}\label{app:proof_both_increasing}
    Let $$g(\lambda_j)=\int_{\nu}^\infty \frac{1}{z+\lambda_i}\psi(z)F(z+\lambda_j) \dd z,\quad h(\lambda_j)=\int_{\nu}^\infty \psi(z)F(z+\lambda_j) \dd z.$$
    The derivative of $g(\lambda_j)/h(\lambda_j)$ with respect to $\lambda_j$ is expressed as
    \begin{equation*}
        \frac{\mathrm{d}}{\mathrm{d} \lambda_j}g(\lambda_j)/h(\lambda_j)=\frac{g'(\lambda_j)h(\lambda_j)-g(\lambda_j)h'(\lambda_j)}{(h(\lambda_j))^2},
    \end{equation*}
    where $g'(\lambda_j)$ and $h'(\lambda_j)$ are respectively expressed as
    \begin{equation*}
        g'(\lambda_j)=\frac{\partial}{\partial \lambda_j}\int_{\nu}^\infty \frac{1}{z+\lambda_i}\psi(z)F(z+\lambda_j)\dd z
        =\frac{\partial}{\partial \lambda_j}\int_{\nu}^\infty \frac{1}{z+\lambda_i}\psi(z)f(z+\lambda_j)\dd z,
    \end{equation*}
    and
    \begin{equation*}
        h'(\lambda_j)=\frac{\partial}{\partial \lambda_j}\int_{\nu}^\infty \psi(z)F(z+\lambda_j)\dd z
        =\int_{\nu}^\infty \psi(z)f(z+\lambda_j)\dd z.
    \end{equation*}

    Next, we divide the proof into two cases.
    \paragraph{Fr\'{e}chet Distribution}
    When $F(x)$ is the cumulative distribution function of Fr\'{e}chet distribution, we define $\widetilde{\psi}(x)=\psi(x)e^{-1/(x+\lambda_j)^\alpha}$. 
    Under this definition, we have
    \begin{align*}
        &g'(\lambda_j)h(\lambda_j)
        =\iint_{z,w \geq 0} \frac{\psi(z)\psi(w)f(z+\lambda_j)F(w+\lambda_j)}{(z+\lambda_i)} \dd z \dd w\\
        &=\alpha\iint_{z,w \geq -(0\land\lambda_i\land\lambda_j)} \frac{\widetilde{\psi}(z)\widetilde{\psi}(w)}{(z+\lambda_i)(z+\lambda_j)^{\alpha+1}}\dd z \dd w\\
        &=\frac{\alpha}{2}\iint_{z,w \geq -(0\land\lambda_i\land\lambda_j)} \widetilde{\psi}(z)\widetilde{\psi}(w)\qty(\frac{1}{(z+\lambda_i)(z+\lambda_j)^{\alpha+1}}+\frac{1}{(w+\lambda_i)(w+\lambda_j)^{\alpha+1}})\dd z \dd w,
    \end{align*}
    where the second equality follows from the fact that $\psi(z)f(z+\lambda_j)$ (resp. $\psi(w)F(w+\lambda_j)$) equals zero for $z<-(0\land\lambda_i\land\lambda_j)$ (resp. $w<-(0\land\lambda_i\land\lambda_j)$).
    Similarly, we have
    \begin{align*}
        g(\lambda_j)&h'(\lambda_j)
        =\iint_{z,w \geq -(\lambda_i\land\lambda_j)} \frac{\psi(z)\psi(w)F(z+\lambda_j)f(w+\lambda_j)}{(z+\lambda_i)} \dd z \dd w\\
        &=\alpha\iint_{z,w \geq -(\lambda_i\land\lambda_j)} \frac{\widetilde{\psi}(z)\widetilde{\psi}(w)}{(z+\lambda_i)(w+\lambda_j)^{\alpha+1}}\dd z \dd w\\
        &=\frac{\alpha}{2}\iint_{z,w \geq -(\lambda_i\land\lambda_j)} \widetilde{\psi}(z)\widetilde{\psi}(w)\qty(\frac{1}{(z+\lambda_i)(w+\lambda_j)^{\alpha+1}}+\frac{1}{(w+\lambda_i)(z+\lambda_j)^{\alpha+1}})\dd z \dd w.
    \end{align*}
    Here, by an elementary calculation one can see that for $z,w \in [-(0\land\lambda_i\land\lambda_j),\infty)$, we have
    \begin{align*}
        &\frac{1}{(z+\lambda_i)(z+\lambda_j)^{\alpha+1}}+\frac{1}{(w+\lambda_i)(w+\lambda_j)^{\alpha+1}}-\frac{1}{(z+\lambda_i)(w+\lambda_j)^{\alpha+1}}-\frac{1}{(w+\lambda_i)(z+\lambda_j)^{\alpha+1}}\\
        &=\frac{w-z}{(z+\lambda_i)(w+\lambda_i)}\qty(\frac{1}{(z+\lambda_j)^{\alpha+1}}-\frac{1}{(w+\lambda_j)^{\alpha+1}})\\
        &=\frac{(w-z)\qty(\qty(w+\lambda_j)^{\alpha+1}-\qty(z+\lambda_j)^{\alpha+1})}{(z+\lambda_i)(w+\lambda_i)(z+\lambda_j)^{\alpha+1}(w+\lambda_j)^{\alpha+1}}\geq 0,
    \end{align*}
    where the last inequality holds since $\varphi(x)=x^{\alpha+1}$ is monotonically increasing in $[0,+\infty)$ for $\alpha>0$. Therefore, when $F(x)$ is the cumulative distribution function of Fr\'{e}chet distribution, we have $\frac{\mathrm{d}}{\mathrm{d} \lambda_j}g(\lambda_j)/h(\lambda_j)\geq 0$, which implies that $g(\lambda_j)/h(\lambda_j)$ is monotonically increasing in $\lambda_j$.

    \paragraph{Pareto Distribution}
    When $F(x)$ is the cumulative distribution function of Pareto distribution, we have
    \begin{align*}
        &g'(\lambda_j)h(\lambda_j)=\iint_{z,w \geq 1} \frac{\psi(z)\psi(w)f(z+\lambda_j)F(w+\lambda_j)}{(z+\lambda_i)} \dd z \dd w\\
        &=\alpha\iint_{z,w \geq 1-(0\land\lambda_i\land\lambda_j)} \frac{\psi(z)\psi(w)(1-(w+\lambda_j)^{-\alpha})}{(z+\lambda_i)(z+\lambda_j)^{\alpha+1}}\dd z \dd w\\
        &=\frac{\alpha}{2}\iint_{z,w \geq 1-(0\land\lambda_i\land\lambda_j)} \psi(z)\psi(w)\qty(\frac{(1-(w+\lambda_j)^{-\alpha})}{(z+\lambda_i)(z+\lambda_j)^{\alpha+1}}+\frac{(1-(z+\lambda_j)^{-\alpha})}{(w+\lambda_i)(w+\lambda_j)^{\alpha+1}})\dd z \dd w,
    \end{align*}
    where the second equality follows from the fact that $\psi(z)f(z+\lambda_j)$ (resp. $\psi(w)F(w+\lambda_j)$) equals zero for $z<1-(0\land\lambda_i\land\lambda_j)$ (resp. $w<1-(0\land\lambda_i\land\lambda_j)$).
    Similarly, we have
    \begin{align*}
        &g(\lambda_j)h'(\lambda_j)=\iint_{z,w \geq 1} \frac{\psi(z)\psi(w)f(z+\lambda_j)F(w+\lambda_j)}{(z+\lambda_i)} \dd z \dd w\\
        &=\alpha\iint_{z,w \geq 1-(0\land\lambda_i\land\lambda_j)} \frac{\psi(z)\psi(w)(1-(z+\lambda_j)^{-\alpha})}{(z+\lambda_i)(w+\lambda_j)^{\alpha+1}}\dd z \dd w\\
        &=\frac{\alpha}{2}\iint_{z,w \geq 1-(0\land\lambda_i\land\lambda_j)} \psi(z)\psi(w)\qty(\frac{(1-(z+\lambda_j)^{-\alpha})}{(z+\lambda_i)(w+\lambda_j)^{\alpha+1}}+\frac{(1-(w+\lambda_j)^{-\alpha})}{(w+\lambda_i)(z+\lambda_j)^{\alpha+1}})\dd z \dd w.
    \end{align*}
    Here, by an elementary calculation one can see that for $z,w \in [1-(0\land\lambda_i\land\lambda_j),\infty)$, we have
    \begin{align*}
        &\frac{(1-(w+\lambda_j)^{-\alpha})}{(z+\lambda_i)(z+\lambda_j)^{\alpha+1}}+\frac{(1-(z+\lambda_j)^{-\alpha})}{(w+\lambda_i)(w+\lambda_j)^{\alpha+1}}-\frac{(1-(z+\lambda_j)^{-\alpha})}{(z+\lambda_i)(w+\lambda_j)^{\alpha+1}}-\frac{(1-(w+\lambda_j)^{-\alpha})}{(w+\lambda_i)(z+\lambda_j)^{\alpha+1}}\\
        &=\frac{w-z}{(z+\lambda_i)(w+\lambda_i)}\qty(\frac{1-(w+\lambda_j)^{-\alpha}}{(z+\lambda_j)^{\alpha+1}}-\frac{1-(z+\lambda_j)^{-\alpha}}{(w+\lambda_j)^{\alpha+1}})\\
        &=\frac{w-z}{(z+\lambda_i)(w+\lambda_i)(z+\lambda_j)^{\alpha+1}(w+\lambda_j)^{\alpha+1}}\\
        &\hspace{5cm}\qty(\qty((w+\lambda_j)^{\alpha+1}-(w+\lambda_j))-\qty((z+\lambda_j)^{\alpha+1}-(z+\lambda_j)))\geq 0,
    \end{align*}
    where the last inequality holds because $\varphi(x)=x^{\alpha+1}-x$ is monotonically increasing in $[1,+\infty)$ for $\alpha>0$. Therefore, we have $\frac{\mathrm{d}}{\mathrm{d} \lambda_j}g(\lambda_j)/h(\lambda_j)\geq 0$, which concludes the proof. \qed









\vskip 0.2in
\bibliography{references}

\end{document}